\documentclass[journal, two column, 10pt]{IEEEtran}
\pdfoutput=1
\usepackage[pdftex]{graphicx}
\usepackage{epstopdf}
\usepackage{amsmath,amssymb}
\usepackage{color}
\usepackage{flushend}

\usepackage[T1]{fontenc}
\usepackage{mathptmx}
\usepackage[mathscr]{euscript}
\usepackage{algorithmic}
\usepackage[english]{babel}
\usepackage{amsthm}
\usepackage[font=footnotesize, labelfont=bf, justification=centering]{caption}
\usepackage{graphicx}
\usepackage{cases}
\usepackage{amsmath}

 \usepackage{multirow}
\usepackage{booktabs}

\usepackage{subfig}

\usepackage[ruled,linesnumbered]{algorithm2e}

\SetCommentSty{mycommfont}
\SetAlgoSkip{}
 \newtheorem{thm}{Theorem}[section]

 \newtheorem{prop}{Proposition}[section]
 \newtheorem{defn}{Definition}[section]

 \newtheorem{rem}{Remark}[section]

\graphicspath{{./Figures/}}

\begin{document}
\bstctlcite{IEEEexample:BSTcontrol}

\title{{\huge CARE: \underline{C}ooperative \underline{A}utonomy for \underline{R}esilience and \underline{E}fficiency of Robot Teams for Complete Coverage of Unknown Environments under Robot Failures \vspace{10pt}
}
\thanks{$^{\dag}$ Department of Electrical and Computer Engineering, University of Connecticut, Storrs, CT, USA.}\\
\thanks{$^{\star}$ Corresponding Author (email id: shalabh.gupta@uconn.edu)}
} 
\author{ \begin{tabular}{cccccccccc}
 \textbf{Junnan Song$^\dag$} & \ \ \  \textbf{Shalabh Gupta${^\dag}{^\star}$} \\
\end{tabular}
\\ \vspace{-10pt}
}

\maketitle
\begin{abstract}
This paper addresses the problem of \textit{Multi-robot Coverage Path Planning} (MCPP) for unknown environments in the presence of robot failures. Unexpected robot failures can seriously degrade the performance of a robot team and in extreme cases jeopardize the overall operation. Therefore, this paper presents a distributed algorithm, called \textit{Cooperative Autonomy for Resilience and Efficiency} (CARE), which not only provides resilience to the robot team against failures of individual robots, but also improves the overall efficiency of operation via event-driven replanning. The algorithm uses distributed \textit{Discrete Event Supervisors} (DESs), which trigger games between a set of feasible players in the event of a robot failure or idling, to make collaborative decisions for task reallocations. The game-theoretic structure is built using \textit{Potential Games}, where the utility of each player is aligned with a shared objective function for all players. The algorithm has been validated in various complex scenarios on a high-fidelity robotic simulator, and the results demonstrate that the team achieves complete coverage under failures, reduced coverage time, and faster target discovery as compared to three alternative methods.
\end{abstract}

\vspace{-3pt}
\vspace{-0pt}

\thispagestyle{empty}

\section{Introduction}\label{sec:intro}
The search and coverage operations of autonomous robots have widespread applications such as floor cleaning, lawn mowing, oil spill cleaning, crop cutting and seeding, mine countermeasures, ocean floor inspection. These operations require \textit{Coverage Path Planning} (CPP)~\cite{SG17}\cite{XVR14}\cite{BKABOOMS16}\cite{SGHZ13}\cite{BTA14}, where a coverage path is needed for the robot to completely cover the search area while avoiding obstacles and having minimum overlapping trajectory to minimize the total coverage time.

Thus far, several CPP algorithms have been reported in literature~\cite{C01}\cite{GC13} and a brief review is provided in Section~\ref{sec:review}.

\vspace{-0pt}
\subsection{\textbf{Motivation}}\label{sec:motivation}
Although many CPP methods are available when using a single robot, only a limited body of work has focused on \textit{Multi-robot Coverage Path Planning} (MCPP). A popular control architecture in the existing MCPP methods is to split the overall workload into multiple tasks, and then use some single-robot CPP method for coverage in each task~\cite{HK08}\cite{SGH14}.

However, since the robots typically operate in uncertain environments, they are prone to different failures such as sensor or actuator malfunctions, mechanical defects, loss of power~\cite{CM05}. The consequences of these failures include coverage gaps, loss of critical data, performance degradation (e.g., missed detections of targets), prolonged operation time, and in extreme cases overall mission failure. For example, coverage gaps in mine countermeasure operations can leave undetected underwater mines which are serious threats to traversing vessels. It is therefore critical that the robot team is resilient to failures, in the sense that it can sustain the overall team operation and protect the mission goals (e.g., complete coverage) even in presence of a few robot failures~\cite{RGM09}. The role of resilience is to \textit{assure system-level survivability and fast recovery to normalcy from unanticipated emergency situations (e.g., robot failures)}. In the context of the MCPP problem, a resilient robot team is expected to autonomously re-organize the active robots in an optimal manner to complete the unfinished tasks of failed robots.

Secondly, it is also important that the robot team operates efficiently. Typically, due to incorrect, incomplete or lack of \textit{a priori} knowledge of the environment, the initial task allocation may be sub-optimal. As a result, some robots may finish their tasks earlier and become idle, which is a waste of their resources. Thus, it is critical that the robot team autonomously reallocates these idling robots in an optimal manner to assist other robots to reduce the total coverage time.

\begin{figure}[t]
  \centering
  \includegraphics[width=1\columnwidth]{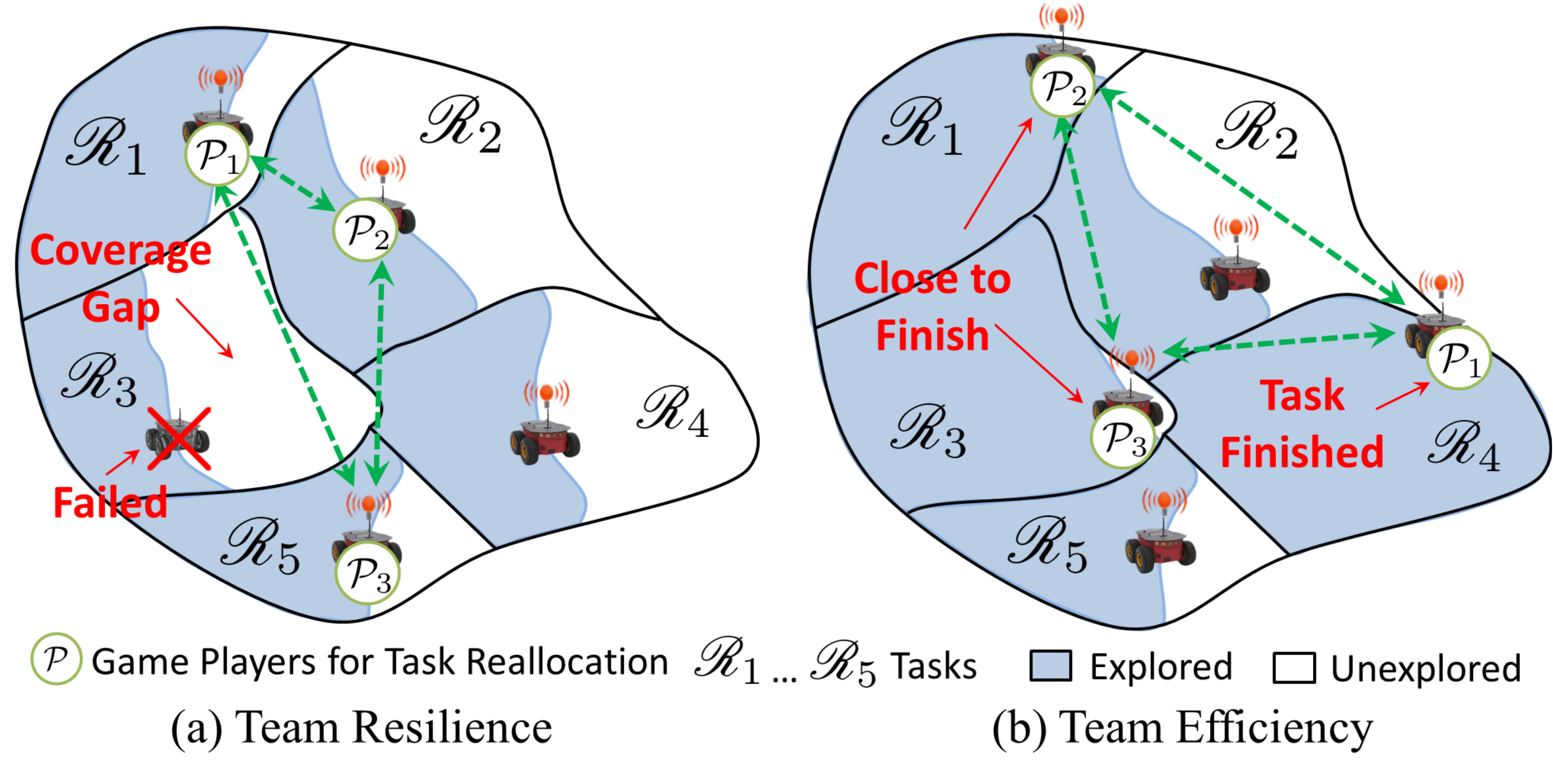}
  \caption{Concepts of resilience and efficiency of a robot team} \vspace{-8pt}
  \label{fig:requirements}
\end{figure}

Fig.~\ref{fig:requirements} illustrates the above concepts of resilience and efficiency. Fig.~\ref{fig:requirements}a shows an example of resilience where the neighbors of a failed robot proactively negotiate to decide whether any of them should leave its current task to fill the coverage gap. Fig.~\ref{fig:requirements}b shows an example of efficiency where a group of robots that have finished (or are close to finish) their current tasks negotiate to optimally reallocate to new tasks or to help other robots in their existing tasks.

\vspace{-0pt}
\subsection{\textbf{Challenges}}\label{sec:challenges}
The challenges associated with the problem of resilient and efficient MCPP are presented below.

\begin{itemize}
\item \textit{Scalability}: The MCPP algorithm should be scalable to accommodate a growing number of tasks and/or robots, thus making a distributed control structure appropriate.
\vspace{2pt}
\item \textit{Optimization factors}: The optimization for task reallocation must consider the following factors:
\begin{enumerate}
\item Task worths, which can be quantified by the expected number of undiscovered targets (e.g., crops to cut or mines to discover) in the tasks.
\item Probabilities of success of the available robots in finishing the contested tasks, which depend on various factors including their current energy levels, the costs of traveling to the contested tasks, and the costs of finishing those tasks.
\end{enumerate}

\item \textit{Dynamically changing conditions}: The conditions of robots as well as tasks change dynamically during coverage. The task worths decrease as targets are discovered. On the other hand, the robots drain their batteries during exploration, hence decreasing their probabilities of success. Therefore, the optimization process must accommodate these dynamic factors.

\vspace{2pt}
\item \textit{Computation time}: First of all, the optimization must be event-driven, i.e., triggered only in case of failures and/or idling. Secondly, once the optimization is triggered, the task reallocation decision must be made in a timely manner to avoid prolonged coverage time, thus motivating a local distributed event-focused optimization over only a subset of available robots and tasks.

\vspace{2pt}
\item \textit{Connection between local and global objectives}: Although the local optimization decision can be sub-optimal for the whole team, it is important that it is still aligned with the global objectives of the team. In other words, the local optimization must not only benefit the involved robots but also the whole team. The objectives include early detection of remaining targets, reduction in the total coverage time, and complete coverage.

\vspace{2pt}
\item \textit{Complete coverage}: The MCPP algorithm must guarantee complete coverage of the \textit{a priori} unknown environment.
\end{itemize}

\vspace{-9pt}
\subsection{\textbf{Our Contributions}}\label{sec:contribution}
To the best of our knowledge, the concept of resilient coverage has not been adequately addressed in the existing MCPP methods. Thus, we present a novel online MCPP algorithm for resilient and efficient coverage in \textit{a priori} unknown environments, which addresses the challenges discussed above. The algorithm is called \textit{Cooperative Autonomy for Resilience and Efficiency} (CARE). For coverage control in each task, CARE utilizes our previously published $\epsilon^\star$ algorithm~\cite{SG17}, which is a single-robot online CPP algorithm. More details of the $\epsilon^\star$ algorithm are provided in Section~\ref{sec:review}. 

The CARE algorithm operates in a distributed yet cooperative fashion. Each robot is controlled using a \textit{Discrete Event Supervisor} (DES), which triggers games between a set of feasible players in the event of a robot failure or idling, to make collaborative decisions for task reallocations. The game-theoretic structure is modeled using \textit{Potential Games}~\cite{MS96}, where the utility of each player is connected with a shared objective function for all players. In case of no failures, CARE reallocates idling robots to support other robots in their tasks; hence reduces coverage time and improves team efficiency. In case of robot failures, CARE guarantees complete coverage via filling coverage gaps by optimal reallocation of other robots; hence providing resilience, albeit with a possibly small degradation in coverage time.

The CARE algorithm has been validated on a high-fidelity robot simulator (i.e., Player/Stage) in various complex obstacle-rich scenarios. The results demonstrate that the team achieves complete coverage under failures and enables faster target discovery as compared to three alternative methods.

The rest of the paper is organized as follows. Section~\ref{sec:review} presents a brief review of the existing MCPP algorithms. Section~\ref{sec:problemformulation} formulates the MCPP problem and Section~\ref{sec:tepsilon} presents the details of the CARE algorithm. The results are discussed in Section~\ref{sec:results} and the paper is concluded in Section~\ref{sec:conclusion} with recommendations for future work.

\vspace{-9pt}
\section{Related Work}\label{sec:review}
Existing CPP methods can be categorized as offline or online (i.e., sensor-based). Offline approaches assume the environment to be \textit{a priori} known, while online approaches generate coverage paths \textit{in situ} based on sensor information. Independently, CPP approaches are also described as randomized or systematic. Random strategies follow simple behavior-based rules, requiring in some cases no localization system or costly computational resources. However, they tend to generate strongly overlapped trajectories, thus unsuitable for time-critical operations. In contrast, systematic approaches are typically based on cellular decomposition~\cite{AC02}\cite{GS04} of the search area into sub-regions, and then adopt certain pre-defined path pattern (e.g., back and forth) for coverage in each sub-region; or by partitioning the search area into grid cells and then constructing potential field~\cite{ZJBY93} or spanning trees~\cite{GR01}\cite{SWV15} to generate coverage paths. In our previous work~\cite{SG17}, we presented the $\epsilon^\star$ algorithm for single-robot CPP in unknown environments. The $\epsilon^\star$ algorithm uses an Exploratory Turing Machine (ETM) that consists of a 2D multilevel tape formed by Multiscale Adaptive Potential Surfaces (MAPS). The ETM stores and updates the explored, unexplored, and obstacle regions as time-varying potentials on MAPS. By default, the ETM adopts the lowest level of MAPS to generate the next waypoint; while it switches to higher levels as needed to evacuate from a local extremum. It is shown that $\epsilon^\star$ is computationally efficient for real-time applications, guarantees complete coverage in unknown environments, and does not require cellular decomposition of search area using critical points on obstacles. In CARE, $\epsilon^\star$ is used as the baseline coverage method by each robot to search within its task.

{\setlength{\belowcaptionskip}{-8pt}
\begin{table*}[t]
\centering
\caption{A comparison of key features with other online MCPP algorithms}
\label{table:comparison}
\begin{tabular}{ccccccc}
\toprule
\multirow{2}{*}{}        &  \multirow{2}{*}{CARE}  & \multirow{2}{*}{First-responder~\cite{SGH14}} &   \multirow{2}{*}{Brick \& Mortar~\cite{FTL07}}    & \multirow{2}{*}{ORMSTC~\cite{AHK08}}\\
                                                                                        &       &   \\ \hline \\ [-0.16cm]
\textit{Path Pattern}       & Back and forth  & Back and forth & No obvious pattern observed  & Spiral \\ [0.1cm]

\multirow{6}{*}{\begin{tabular}[c]{@{}c@{}c@{}} \textit{Resilience} \\\textit{Strategy} \end{tabular}} & \multirow{6}{*}{\begin{tabular}[c]{@{}c@{}c@{}} Neighbors jointly optimize to \\reorganize themselves to \\immediately fill the coverage \\gap caused by the failed robot if\\the optimization criteria are satisfied \end{tabular}}& \multirow{6}{*}{\begin{tabular}[c]{@{}c@{}c@{}} Wait until some robot \\finishes its task and is \\reassigned to fill\\ the coverage gap \end{tabular}}  & \multirow{6}{*}{\begin{tabular}[c]{@{}c@{}c@{}} Remaining robots continue re-\\gularly. The coverage gaps be-\\come extra workloads. May pro-\\duce strongly overlapped paths \\due to the looping problem \end{tabular}}   & \multirow{6}{*}{\begin{tabular}[c]{@{}c@{}c@{}} Neighbors extend their trees\\ to fill the coverage gap of\\ the failed robot, but the approach \\ is not proactive and the already\\ explored area by the failed\\ robot is scanned again \end{tabular}} \\
                                                                                        &                                                                                             &                                                                                           \\ [1.3cm]
\multirow{4}{*}{\begin{tabular}[c]{@{}c@{}} \textit{No-idling} \\\textit{Strategy}\end{tabular}} & \multirow{4}{*}{\begin{tabular}[c]{@{}c@{}} The idling robot and its near-finishing\\ neighbors jointly optimize to help\\ other robots to reduce coverage\\ time and collect more worth early\end{tabular}}                                                         & \multirow{4}{*}{\begin{tabular}[c]{@{}c@{}} Idling robots are\\ reallocated to new\\ tasks that maximize\\ their own utility\end{tabular}}   & \multirow{4}{*}{\begin{tabular}[c]{@{}c@{}} None\end{tabular}} & \multirow{4}{*}{\begin{tabular}[c]{@{}c@{}} None \end{tabular}} \\
 &                                                                                             \\ [0.8cm]
\begin{tabular}[c]{@{}c@{}} \textit{Optimization} \\\textit{Factors} \end{tabular}  & \begin{tabular}[c]{@{}c@{}}Estimated worths of contested\\ tasks, remaining reliability and\\ traveling time of live robots\end{tabular} & \begin{tabular}[c]{@{}c@{}}Unexplored portion of \\ tasks and traveling\\ time of robots \end{tabular} & \begin{tabular}[c]{@{}c@{}}None \end{tabular} & \begin{tabular}[c]{@{}c@{}}None\end{tabular}  \\ \bottomrule \vspace{-18pt}
\end{tabular}
\end{table*}
}

In terms of MCPP, Batalin and Sukhatme~\cite{BS02} proposed two local approaches for MCPP in unknown environments, based on mutually dispersive interaction between robots. Latimer et al.~\cite{LSLSCH02} presented a boustrophedon cellular decomposition-based approach using a team of circular robots. The robots operate together, but can split up into smaller teams when cells are created or completed. Rekleitis et al.~\cite{RNRC08} presented a distributed auction-based cooperative coverage algorithm, where the whole space is partitioned into tasks of fixed height and width, and robots utilized the Morse decomposition based single-robot CPP algorithm to search within each task. Sheng et al.~\cite{SYTX06} proposed a multi-robot area exploration method with limited communication range, where the waypoint of each robot is computed using a distributed bidding mechanism based on frontier cells. The bids rely on the information gain, communication limitation, and traveling costs to frontier cells. Rutishauser et al.~\cite{RCM09} presented a distributed coverage method using miniature robots that are subject to sensor and actuator noise. Xu and Stentz~\cite{XS11} presented the $k$-Rural Postman Problem ($k$-RPP) algorithm to achieve environmental coverage with incomplete prior information using $k$ robots, that seeks to equalize the lengths of $k$ paths. Bhattacharya et al.~\cite{BGK14} generalized the control law towards minimizing the coverage functional to non-Euclidean spaces, and presented a discrete implementation using graph search-based algorithms for MCPP. Karapetyan et al.~\cite{KBMTR17} presented two approximation heuristics for MCPP in known environments, where the search area is divided into equal regions and exact cellular decomposition based coverage was used to search each region. Later, these methods were improved to consider vehicle kinematic constraints~\cite{KMLLOR18}. Yang et al.~\cite{YL04} proposed an online neural network based MCPP approach. In their method, the discovered environment was represented according to the dynamic activity landscape of the neural network, which is used to compute robot waypoints; and robots treat each other as moving obstacles during operation.

However, the above-mentioned algorithms have not addressed the problem of resilience in MCPP. In this regard, Agmon et al.~\cite{AHK08} presented a family of Multi-robot Spanning Tree Coverage (MSTC) algorithms, where the Online Robust MSTC (ORMSTC) algorithm enables each robot to incrementally construct a local spanning tree to cover a portion of the whole space. If some robot fails, its local tree is released and taken over by its neighbors, but the already explored region of the failed robot must be scanned again. Also, the tree grows on the scale of $2\times2$ cells, while if any cell within such larger cell is occupied by obstacles, the whole larger cell would not be covered, thus leading to incomplete coverage. Zheng et al.~\cite{ZKKJ10} presented a polynomial-time Multi-Robot Forest Coverage (MFC) algorithm that computes tree covers for each robot with trees of balanced weights, and they showed the superiority of MFC in time to MSTC via simulations; however, their algorithm does not consider failures. Song et. al.~\cite{SGH14} presented the First-Responder (FR) cooperative coverage strategy, where early completed robots are reassigned to available new tasks that can maximize their own utility. However, this algorithm is not proactive, i.e., the coverage gaps caused by robot failures will not be filled until some other robots complete their tasks. Ferranti et. al.~\cite{FTL07} presented the Brick and Mortar (B\&M) algorithm, where the waypoint of each robot is computed locally based on the states of cells in the neighborhood. The idea behind B\&M is to gradually thicken the blocks of inaccessible cells (i.e., visited or wall cells), while maintaining the connectivity of accessible cells (i.e., explored or unexplored cells). An unexplored cell can be marked as explored or visited, where the latter is allowed if it does not block the path between any two accessible cells in the neighborhood. The waypoint gives priority to the unexplored cell in the neighborhood, which has the most inaccessible cells around it. When some robot fails, the remaining robots continue regularly and the coverage gap becomes an extra workload; however, their method may produce redundant coverage due to the looping problem.

\vspace{0pt}
\textbf{Research Gap}: Although resilience concepts have been discussed in robot design~\cite{TSMGWKWW14}, robot damage detection and recovery~\cite{KCM13}, flocking of robot teams~\cite{SSPPK17} and networked control systems security under attacks~\cite{SPYZH17}, there is a scarcity of efforts that deal with the resilient coverage using multiple robots. Some of above-mentioned papers considered robot failures during coverage, however, their remedy was to simply release the coverage gaps to the remaining team, without optimization over the criticality (i.e., available worth) of such coverage gaps and the reliability of remaining robots. Thus, they are not proactive in filling the coverage gaps immediately if they satisfy optimization criteria, they wait until some other robots finish their tasks. In this regard, this paper presents a game-theoretic method for resilient and efficient coverage that incorporates these optimization factors while making event-driven proactive task reallocations. Table~\ref{table:comparison} presents a comparison of the key features of the CARE algorithm with the other relevant online MCPP algorithms.

\section{Problem Description}\label{sec:problemformulation}
This section presents the description of the robots, the MCPP problem and the performance metrics.

\vspace{-5pt}
\subsection{\textbf{Description of the Robots}} \label{Section:AV}
Let $V=\{v_1, v_2,\ldots v_N\}$ be the team of $N \in \mathbb{N}^+$ robots, which are unmanned autonomous vehicles, as shown in Fig.~\ref{fig:tiling_vehicle}. It is assumed that each robot is equipped with:
\begin{itemize}
\item a localization device (e.g., GPS) or a SLAM~\cite{KE15} system for operations in GPS-denied environments;
\item a range detector (e.g., a laser) to detect obstacles within a radius $r_s \in \mathbb{R}^+$; 
\item a task specific sensor for performing the desired task; and
\item a wireless communication device for (periodic or event-driven) information exchange between all pairs of robots. The communication is assumed to be perfect.
\end{itemize}

The robots continuously deplete the energies from their batteries; thus, their reliability is assessed based on the remaining energy as presented below.

\vspace{6pt}
\textbf{Battery Reliability}: Each robot $v_\ell \in V$, is assumed to carry a battery whose reliability~\cite{IAHZ15}, denoted as $R_{v_\ell}(t)$, can be computed as $R_{v_\ell}(t) = 1 - F(t)$, where $F(t)$ is the probability of battery being drained up to time $t$. Typically, the state-of-charge of a battery can be model using the realistic \textit{Kinetic Battery Model} (KiBaM), which takes into account many important non-linear properties of batteries such as the rate-capacity effect and the recovery effect~\cite{JH09}. It is shown in~\cite{CJH07} that with KiBaM, $F(t)$ follows a S-shaped curve when operating under different stochastic workload models (e.g., the on/off model and the burst model). The S-shaped curve can be approximated using a sigmoid function~\cite{IAHZ15}. As such, the reliability of a robot $v_{\ell}$ is given as:
\vspace{-3pt}
\begin{equation}\label{eq:battery}
R_{v_\ell}(t) = \frac{1}{1 + e^{\rho_0(t-\rho_1)}},
\end{equation}
where $\rho_0$ and $\rho_1$ indicate the curvature of the growth part and the inflection point, respectively. Their exact values depend on the choice of batteries. More details on the selection of these parameters are presented in Section~\ref{sec:results}.

\vspace{-5pt}
\subsection{\textbf{The MCPP Problem}}\label{Sect:subseccoverage}
The search area $\mathcal{R} \subset \mathbb{R}^2$ is assumed to be a planar field whose borderline is defined either by a hard barrier (e.g., walls or obstacles) or by a soft boundary (e.g., sub-area of a large field). A finite but unknown number of obstacles with arbitrary shapes are assumed to populate this area, but their exact locations and shapes are \textit{a priori} unknown.

For the purpose of coverage path planning, a tiling $\mathcal{T}=\{\tau_{\alpha} \subset \mathbb{R}^2, \alpha = 1, \ldots |\mathcal{T}|\}$ is constructed to cover $\mathcal{R}$, i.e. $\mathcal{R} \subseteq \bigcup_{\alpha=1}^{|\mathcal{T}|} \tau_{\alpha}$, as shown in Fig.~\ref{fig:tiling_vehicle}. Each $\tau_{\alpha} \in \mathcal{T}$ is called an $\epsilon$-cell, which is a square-shaped cell of side length $\epsilon \in \mathbb{R}^+$. The tiling is formed as minimal such that all $\epsilon$-cells are disjoint from each other, i.e., $\tau^{\circ}_{\alpha} \bigcap \tau^{\circ}_{\beta} = \emptyset$, $\forall \alpha, \beta \in \{1, \ldots |\mathcal{T}|\}, \alpha \neq \beta$, where $^{\circ}$ denotes the interior; and the removal of any single $\epsilon$-cell from $\mathcal{T}$ will destroy the covering.

The tiling $\mathcal{T}$ is partitioned into three sets: i) \textit{obstacle} ($\mathcal{T}^o$), ii) \textit{forbidden} ($\mathcal{T}^f$), and iii) \textit{allowed} ($\mathcal{T}^a$). While the cells in $\mathcal{T}^o$ are occupied by obstacles, the cells in $\mathcal{T}^f$ create a buffer around the obstacles to prevent collisions due to inertia or large turning radius of the robots. Due to lack of \textit{a priori} knowledge of the environment, the obstacle cells and forbidden cells are discovered online using sensor measurements. The remaining cells are allowed, which form the free space $\mathcal{R}^a = \bigcup_{\tau_\alpha \in \mathcal{T}^a} \tau_{\alpha}$ that is desired to be covered.

\begin{figure}[t]
    \centering
    \includegraphics[width = 0.72\columnwidth]{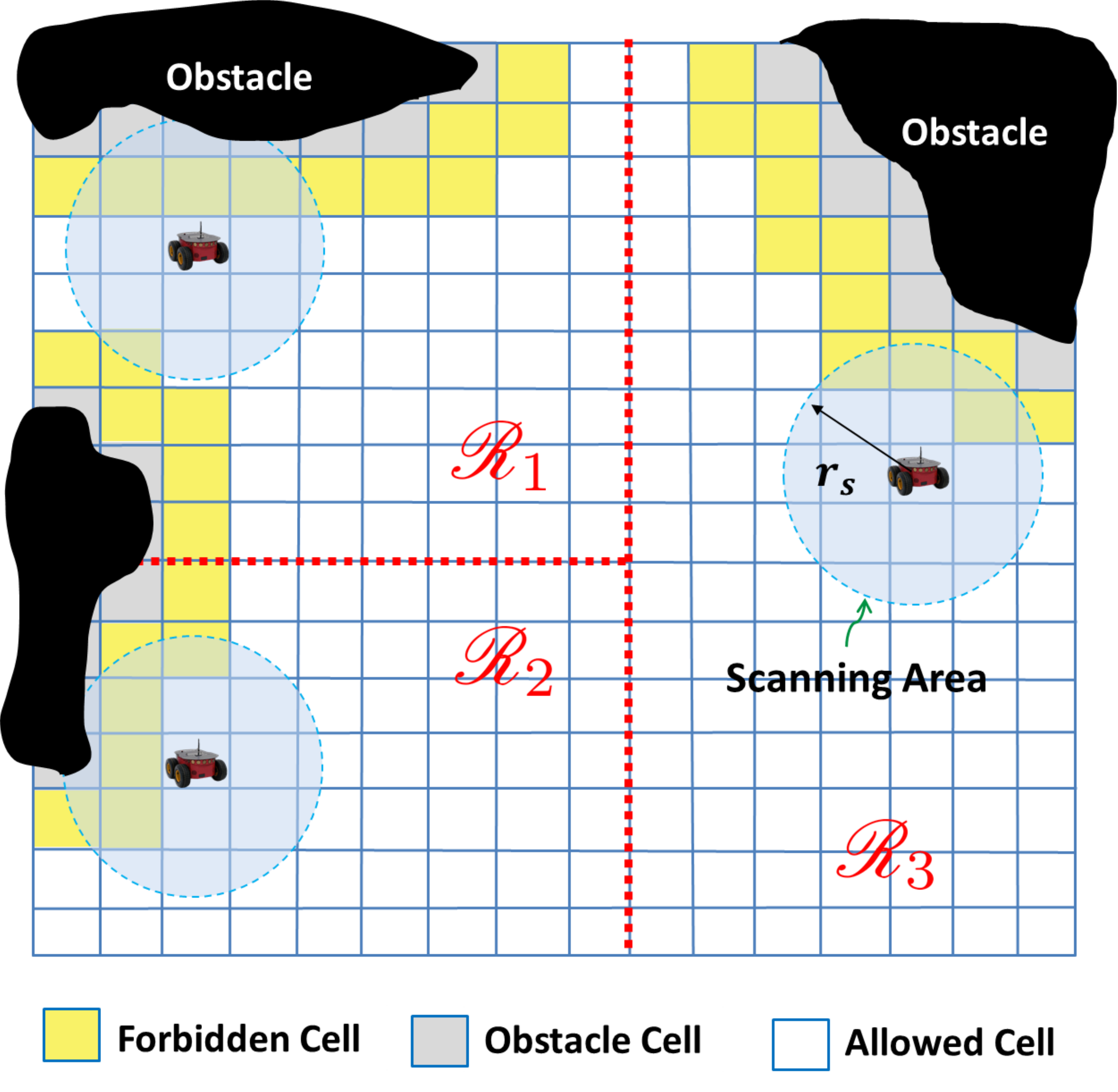}
    \caption{Example of a search area and its tiling. A team of 3 robots are scanning in three different tasks $\mathcal{R}_1$, $\mathcal{R}_2$ and $\mathcal{R}_3$. Robots are equipped with lasers for obstacle mapping.}
    \label{fig:tiling_vehicle}\vspace{-12pt}
\end{figure}

For distribution of multiple robots, an initial task allocation is required. Thus, the tiling $\mathcal{T}$ is grouped into $M$ disjoint regions $\{\mathcal{R}_r \subset \mathcal{T}: r = 1,\ldots M\}$, s.t. $\mathcal{R} = \bigcup_{r = 1}^M \mathcal{R}_r$. Each region $\mathcal{R}_r$ is regarded as one task and is referred as task $r$. Fig.~\ref{fig:tiling_vehicle} shows an example of the area with $M=3$ tasks. Each robot can work on one task at a time, but one task can be assigned to multiple robots. Note that $M$ may not be equal to $N$. 

\begin{rem} The problem of optimal space partitioning into disjoint tasks and optimal initial robot allocations may require consideration of several factors (e.g., obstacle distribution, robot capabilities, and terrain types (or bathymetry)) and is beyond the scope of this paper. Here, we assume that no \textit{a priori} knowledge of the environment is available, thus the tasks are made of equal sizes. However, as more information is obtained during exploration, event-driven task re-allocations are performed for performance improvement. 
\end{rem}

\begin{defn}[\textbf{Complete Coverage}] Let $\epsilon_\ell(k) \in \mathcal{T}$ be the $\epsilon$-cell that is visited and explored by the robot $v_\ell$ at time $k$. Then the robot team $V$ is said to achieve complete coverage, if $\exists K \in \mathbb{N}$, s.t. the sequences $\{\epsilon_\ell(k), k=0, \ldots K\}, \forall \ell=1,\ldots N$, jointly cover the free space $\mathcal{R}^a$, i.e.,

\vspace{-3pt}
\begin{equation}\label{defn:epsiloncov}
\mathcal{R}^a \subseteq \bigcup_{\ell=1}^N \bigcup_{k=1}^{K}\epsilon_\ell(k).
\end{equation}
\end{defn}

In other words, the coverage is said to be complete if every cell in $\mathcal{R}^a$ is explored by at least one robot.

\vspace{3pt}

Next, it is assumed that each task contains randomly distributed targets, and their exact numbers and locations are unknown (details in Section \ref{sec:taskreallocation}). However, it is assumed that the expected number of targets in each task is known, which in practice could be obtained by various means such as field surveys, aerial views or prior knowledge from other sources.

\begin{rem} If the total number and spatial distribution of targets is known a priori, then complete coverage may not be necessary and an optimal traversing strategy could be constructed to find all the targets. However, in this paper, we assume that the planner neither knows the exact number of these targets, nor their exact locations, thus complete coverage becomes mandatory to guarantee finding all the targets. 
\end{rem}

Due to non-uniform spatial distribution of targets and obstacles within each task, the targets are discovered at unequal rates by all robots. Thus, at any point of time all tasks could contain significantly different numbers of undiscovered targets. It is therefore critical that the regions with the maximum number of targets are scanned earlier and are given priority. Early detection of targets helps when the mission is terminated prematurely due to emergencies, failures or other reasons and also provides mental comfort to the operator. For example, once the highly utilized areas of a building floor are cleaned then other areas could be cleaned gradually at ease.

Furthermore, the robots may suffer from unexpected failures during the coverage operation due to several reasons (e.g., sensor (or actuator) malfunctions, and mechanical defects) which lead to coverage gaps. Thus, it is important to fill these coverage gaps by task reallocations of healthy robots. It is also important that the criticality of the task of the failed robot, as measured by the expected number of remaining targets, is evaluated for task-reallocations in comparison with the existing tasks of healthy robots.

\vspace{-6pt}
\subsection{\textbf{Performance Metrics}}\label{Sec:performancemetrics}
The quality of multi-robot coverage can be evaluated based on the following performance metrics:
\begin{itemize}
\item \textit{Coverage ratio} ($CR$): The ratio of the explored free space to the total free space, i.e.,

\begin{equation}\label{eq:c_r}
CR = \frac{\big( \cup_{\ell=1}^N \cup_{k=1}^{K}\epsilon_{\ell}(k) \big) \cap \mathcal{R}^a}{\mathcal{R}^a} \in [0, 1].
\end{equation}

Note that $CR < 1$ if the coverage gaps caused by robot failures are left unattended.

\vspace{3pt}
\item \textit{Coverage time} ($CT$): The total operation time of the team. This is measured by the last robot that finishes its task.

\vspace{3pt}
\item \textit{Remaining reliability} ($RR$): The remaining reliability of all live robots at the end of the operation.

\vspace{3pt}
\item \textit{Number of Targets Found} ($NoTF$): The total number of targets discovered by the whole team.

\vspace{3pt}
\item \textit{Time of Target Discovery} ($ToTD$): The time for the whole team to discover a certain percentage of all targets. Note that the time of discovering all targets is less than or equal to the coverage time. Only in the limiting case, when the last target is discovered in the last visited cell by the robot that stops last, the coverage time will be equal to the $ToTD$ for all targets.
\end{itemize}

\begin{figure*}[t]
  \centering
  \includegraphics[width = .98\textwidth]{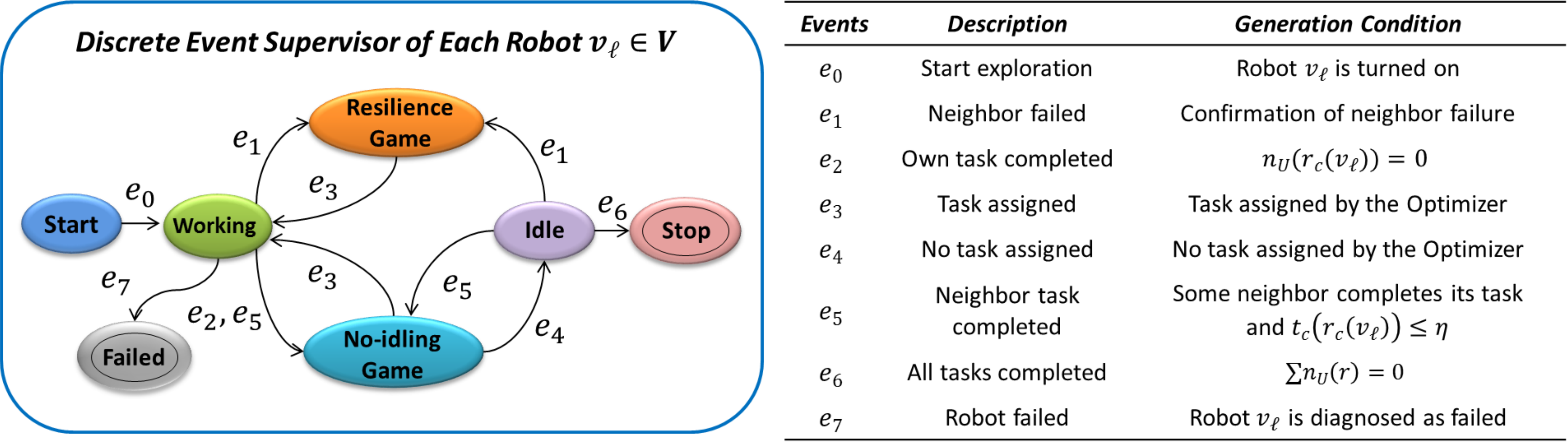}
  \caption{The discrete event supervisor in the CARE algorithm} \vspace{-9pt}
  \label{fig:TE}
\end{figure*}

The objective of MCPP is to achieve $CR = 1$ (even under a few robot failures), while minimizing $CT$, minimizing $ToTD$, and maximizing $RR$. 

\section{CARE Algorithm}\label{sec:tepsilon}

The CARE algorithm addresses the above-mentioned MCPP problem via facilitating distributed event-driven task reallocations. In CARE, a set of local robots jointly re-plan their task assignments in two situations: (1) when a robot has finished its current task, or (2) when a robot has failed and is detected as non-responsive. The replanning algorithm employs a game-theoretic formulation, which computes the task worths and the success probabilities for each participating robot-task pair as optimization factors for optimal task reallocations. The task worths are measured by their expected number of undiscovered targets, while the probabilities of success of robot-task pairs are computed based on the robots' battery reliabilities, travel times, and predicted times to finish the contested tasks.

CARE utilizes a distributed yet cooperative control architecture, where each robot $v_\ell \in V$ is controlled by a Discrete Event Supervisor (DES) that is modeled as a finite state automaton~\cite{CL09}.

\vspace{0pt}
\subsection{\textbf{Discrete Event Supervisor}}\label{sec:supervisor}
The DES as shown in Fig.~\ref{fig:TE} is defined below.

\vspace{-3pt}
\begin{defn}[\textbf{DES}]\label{def:DFSA}
The \textit{DES}, denoted as $H$, is a deterministic finite state automaton represented by a $5$-tuple as follows
\begin{equation}
H = (X, \mathscr{E}, \delta, x_0, X_m), \nonumber
\end{equation}\vspace{-3pt}
where:
\begin{itemize}
\item $X = \{ST, WK, NG, RG, ID, FL, SP\}$ is the set of states, where $ST\equiv$ `\textit{Start}', $WK\equiv$ `\textit{Working}', $NG\equiv$ `\textit{No-idling Game}', $RG\equiv$ `\textit{Resilience Game}', $ID\equiv$ `\textit{Idle}', $FL\equiv$ `\textit{Failed}' and $SP\equiv$ `\textit{Stop}'.
\item $\mathscr{E} = \{e_0, e_1, \ldots e_7\}$ is the finite set of events.
\item $\delta: X \times \mathscr{E} \rightarrow X$ is the partial state transition function. It is defined from one state to another if and only if there exists an arrow connecting them carrying an event.
\item $x_0 = ST$ is the initial state.
\item $X_m = \{SP, FL\}$ is the set of marked states, which means a robot can either stop after finishing all the tasks or it may fail unexpectedly.
\end{itemize}
\end{defn}

While the states $ST$, $ID$, $FL$ and $SP$ are self-explanatory, the operations in states $WK$, $NG$ and $RG$ are described as follows. In state $WK$, the supervisor $H$ of robot $v_\ell$ adopts the $\epsilon^\star$ algorithm~\cite{SG17} for online coverage within its own task. Since no \textit{a priori} information is available, all cells are initialized as unexplored. As the robot explores its task, it updates these cells as explored, obstacles and forbidden as suitable to track the progress of exploration~\cite{SRP09}. This information is then periodically shared and synchronized with other robots such that each robot maintains a symbolic map of the entire region.

In states $RG$ and $NG$, $H$ triggers the \textit{Optimizer} to play resilience games and no-idling games, respectively. The objective of resilience games is to optimally re-organize the neighbors of the failed robot to immediately fill the coverage gap, if it contains higher worth; while for no-idling games, the objective is to optimally reallocate the idling robot and its near-finishing neighbors to help other robots to reduce coverage time and collect more worth early. Details of \textit{Optimizer} functionality are explained later in Section~\ref{sec:taskreallocation}.

\vspace{6pt}
\textbf{Events and State Transitions}: The events in $\mathscr{E}$ enable state transitions in $H$, which are explained below. First, we define:
\begin{enumerate}
\item $r_c: V \rightarrow \{1,\ldots M\}$ to be the allocation function that indicates the current task allocations of robots;
\item $t_c: \{1,\ldots M\} \rightarrow [0,\infty)$ to be the remaining time required to complete a given task by its assigned robots;
\item $n_U: \{1,\ldots M\} \rightarrow \mathbb{N}$ to be the number of unexplored cells in a given task.
\end{enumerate}

Now, consider a robot $v_\ell \in V$ that is currently working in task $r_c(v_\ell)$. Event $e_0$ is generated when $v_\ell$ is turned on, and $H$ moves to the state $WK$ to start searching in task $r_c(v_\ell)$ using the $\epsilon^\star$ algorithm.

Event $e_1$ is produced if any of its neighbor robot fails. This transitions $H$ to the state $RG$, that in turn invokes the \textit{Optimizer} to play the resilience game to generate a task reallocation decision for $v_\ell$. Failure of a robot is detected using a standard mechanism based on heartbeat signals~\cite{CTA02}. Each robot periodically broadcasts heartbeat signals, and also listening from others. Then a neighbor robot is detected as failed if its message is not received by $v_\ell$ constantly for a certain period of time $T_0 \in \mathbb{R}^+$. To ensure robustness to false alarms, its failure is further confirmed if the majority of $\kappa_2 \in \mathbb{N}^+$ neighbors detect its failure. Event $e_2$ occurs as soon as task $r_c(v_\ell)$ is completed, i.e., the number of unexplored cells in task $r_c(v_\ell)$, denoted as $n_U(r_c(v_\ell))$, becomes $0$. Event $e_2$ moves $H$ to the state $NG$, where the \textit{Optimizer} is called to play the no-idling game for finding a new task for $v_{\ell}$.

Event $e_3$ appears if the \textit{Optimizer} assigns a new (or current) task to robot $v_\ell$, which drives $H$ back to the state $WK$ to search in the assigned task; otherwise if no task is assigned, event $e_4$ is generated that moves $H$ to the state $ID$ and the robot becomes idle. Event $e_5$ is produced if some neighbor robot just completed its task and triggered the no-idling game, while $v_\ell$ is close to finish task $r_c(v_\ell)$, i.e., $t_c(r_c(v_\ell)) \leq \eta \in \mathbb{R}^+$, hence ready to reallocate after finishing the current task. Specifically, $t_c(r_c(v_\ell))=\frac{n_U(r_c(v_{\ell}))}{\omega}$, where $\omega \in \mathbb{R}^+$ is the speed of tasking a cell by the assigned robots. Then again, $H$ comes to the state $NG$ and the \textit{Optimizer} is invoked to compute for a new task.

Event $e_6$ occurs when the entire area $\mathcal{R}$ is covered, by satisfying Eq. (\ref{defn:epsiloncov}). This happens when no unexplored cells are left in the whole region, i.e., $\sum_{r=1}^M n_U(r) = 0$. This moves $H$ to the terminal state $SP$ and the coverage is complete. At last, event $e_7$ is generated if $v_\ell$ itself is diagnosed as failed by its own diagnosis, and $H$ moves to the state $FL$. An advanced failure diagnostic tool is beyond the scope of this paper.

\vspace{-12pt}
\subsection{\textbf{Distributed Optimizer}}\label{sec:taskreallocation}
The \textit{Optimizer} is invoked by the supervisor $H$ to compute reallocation decisions under two conditions: (i) $H$ reaches $RG$ state upon detection of a neighbor failure (i.e., event $e_1$); or (ii) $H$ reaches $NG$ state upon completion of its own task (i.e., event $e_2$), or completion of a neighbor's task (i.e., event $e_5$).

Specifically, the \textit{Optimizer} is built based on the concept of \textit{Potential Games}~\cite{MS96}, which have the following advantages: (i) at least one Nash Equilibrium is guaranteed to exist, (ii) several learning algorithms are available (e.g., the Max-Logit algorithm~\cite{SWL11}\cite{DHY15}) that can converge fast to the optimal equilibrium, and (iii) the utility of each player is perfectly aligned with a globally shared potential function, thus as each player seeks to increase its own utility, the potential function is simultaneously improved and maximized upon reaching the optimal equilibrium.

Before presenting the details of the \textit{Optimizer}, we list the various useful parameters in Table~\ref{table:variables}. Some mathematical preliminaries are presented below.

\vspace{6pt}
\textbf{Preliminaries}: A game $G$ in strategic form~\cite{M13} consists of:
\begin{itemize}
\item A finite set of players $\mathscr{P} = \{\mathscr{P}_i \in V: i = 1, \ldots |\mathscr{P}|\}$, which includes all available robots that could be reallocated.
\item A non-empty set of actions $\mathscr{A}_i$ associated to each player $\mathscr{P}_i$. In this paper, each action $a_i \in \mathscr{A}_i$ corresponds to the index of an available task, and the action set is assumed identical for all players, i.e., $\mathscr{A}_i = \mathscr{A}_j =\tilde{\mathscr{A}}$, $\forall i, j \in \{1,\ldots |\mathscr{P}|\}$.
\item The \textit{utility function} associated with each player $\mathscr{P}_i$, defined as $\mathscr{U}_i: \mathscr{A}_{\mathscr{P}} \rightarrow \mathbb{R}$, where $\mathscr{A}_{\mathscr{P}} = \mathscr{A}_1 \times \ldots \times \mathscr{A}_{|\mathscr{P}|}$ denotes the set of joint actions for all players.
\end{itemize}

The utility function computes the payoff that $\mathscr{P}_i$ can receive by taking an action $a_i \in \mathscr{A}_i$, given that the rest of the players jointly select $a_{-i} \in \mathscr{A}_{-i}$, where $\mathscr{A}_{-i} := \mathscr{A}_1 \times \ldots \times \mathscr{A}_{i-1} \times \mathscr{A}_{i+1} \times \ldots \times \mathscr{A}_{|\mathscr{P}|}$. A joint action of all players $a_{\mathscr{P}} \in \mathscr{A}_{\mathscr{P}}$ is often written as $a_{\mathscr{P}} = (a_i, a_{-i})$.

{\setlength{\belowcaptionskip}{0pt}
\begin{table}[t]
\centering
\caption{List of key parameters in CARE}
\label{table:variables}
\begin{tabular}{cl}
\toprule
Parameter & Description  \\ [0.1cm] \hline \\ [-0.16cm]
$N$           & Total number of robots   \\[0.1cm]
$M$           & Total number of tasks  \\[0.1cm]
$\rho_0$ & Curvature of the growth part in battery model \\[0.1cm]
$\rho_1$      & Inflection point in battery model \\[0.1cm]
$u$           & Robot traveling speed  \\[0.1cm]
$\omega$      & Robot tasking speed    \\[0.1cm]
$\lambda_r$ & Expected number of targets in task $r$ \\[0.1cm]
\multirow{2}{*}{$\eta$}        & \multirow{2}{*}{\begin{tabular}[l]{@{}l@{}}Threshold to identify robots that are close to \\finishing their tasks\end{tabular}} \\[0.4cm]
\multirow{2}{*}{$\gamma$} & \multirow{2}{*}{\begin{tabular}[l]{@{}l@{}} Threshold to identify incomplete tasks with\\ sufficient work left \end{tabular}} \\[0.4cm]
$\kappa_1$    & Neighborhood size in no-idling games  \\[0.1cm]
$\kappa_2$    & Neighborhood size in resilient games \\ \bottomrule
\end{tabular}
\end{table}
}

\begin{defn}[\textbf{Nash Equlibrium}]\label{defn:ne}
A joint action $a_{\mathscr{P}}^\star = (a_i^\star, a_{-i}^\star) \in \mathscr{A}_{\mathscr{P}}$ is called a pure Nash Equilibrium if

\begin{equation}\label{eq:ne}
\mathscr{U}_i(a_i^\star, a_{-i}^\star) = \underset{a_i \in \mathscr{A}_i}{\max} \ \mathscr{U}_i(a_i, a_{-i}^\star), \  \forall \mathscr{P}_i \in \mathscr{P}.
\end{equation}
\end{defn}

\vspace{-3pt}
\begin{defn}[\textbf{Potential Games}]\label{defn:potentialgames}
A game $G$ in strategic form with action sets $\{\mathscr{A}_i\}_{i=1}^{|\mathscr{P}|}$ together with utility functions $\{\mathscr{U}_i\}_{i=1}^{|\mathscr{P}|}$ is a potential game if and only if, a potential function $\phi: \mathscr{A}_{\mathscr{P}} \rightarrow \mathbb{R}$ exists, s.t. $\forall$ $\mathscr{P}_i \in \mathscr{P}$
\begin{equation}\label{eq:potentialgames}
\mathscr{U}_i(a_i', a_{-i}) - \mathscr{U}_i(a_i'', a_{-i}) = \phi(a_i', a_{-i}) - \phi(a_i'', a_{-i}),
\end{equation}
$\forall$ $a_i', a_i'' \in \mathscr{A}_i$ and $\forall$ $a_{-i} \in \mathscr{A}_{-i}$.
\end{defn}
A potential game requires perfect alignment between the utility of an individual player and the globally shared potential function $\phi$ for all players, in the sense that the utility change by unilaterally deviating a player's action is equal to the amount of change in the potential function. In other words, the potential function $\phi$ can track the changes in payoffs as some player unilaterally deviates from its current action. Therefore, if $\phi$ is designed as the global objective, then as players negotiate towards maximizing their individual utilities, the global objective is simultaneously optimized.

Now, let us present the resilience games and no-idling games modeled as potential games.

\vspace{6pt}
\textbf{Specifics of Resilience Games and No-idling Games}: Due to different objectives and triggering conditions, the player set and action set are fundamentally different for resilience games and no-idling games. Let $\mathscr{N}_{\kappa}^{v_\ell}$ denote the set of $\kappa \in \mathbb{N}^+$ nearest neighbors of robot $v_\ell$.

$\bullet$ \textbf{No-idling Game}: A no-idling game is triggered when some robot $v_{id} \in V$ completes its current task and becomes idle. Then, it calls its $\kappa_1$ nearest neighbors $v_\ell \in \mathscr{N}_{\kappa_1}^{v_{id}}$ that are close to finish their tasks to participate in the game. Thus, a no-idling game comprises of:
\begin{itemize}
\item $\mathscr{P} = \{v_{id}\} \cup \{v_\ell \in \mathscr{N}_{\kappa_1}^{v_{id}}: t_c(r_c(v_\ell)) \leq \eta\}$.
\item $\tilde{\mathscr{A}} = \{r \in \{1,\ldots M\}: t_c(r) \geq \gamma \in \mathbb{R}^+\}$, which contains incomplete tasks that have sufficient work left to be finished by their currently assigned robots. If some players still have some work left in their current tasks, they are assigned such that they finish their current tasks before being reallocated to new tasks.
\end{itemize}

$\bullet$ \textbf{Resilience Game:} A resilience game is triggered when some robot $v_f$ fails. Then, the $\kappa_2$ nearest neighbors of $v_f$ are involved in the game to re-optimize their current task allocations. Thus, a resilience game comprises of:
\begin{itemize}
\item $\mathscr{P} = \mathscr{N}_{\kappa_2}^{v_f}$.
\item $\tilde{\mathscr{A}} = \{r_c(v_f)\} \cup \{r_c(v_\ell), v_\ell \in \mathscr{N}_{\kappa_2}^{v_f}: t_c(r_c(v_\ell)) > \eta\}$, which contains the current tasks of all players and the failed robot. The condition $t_c(r_c(v_\ell)) > \eta$ ensures that those tasks close to be finished will be completed by their currently assigned robot and hence not needed to be part of the game.
\end{itemize}

\begin{rem}
If there exist other active robots working in the same task of the failed robot, then they will take over this task and no resilience game is triggered.
\end{rem}

\begin{rem}\label{rem:data_synchronization}
When a game is initiated, the information is exchanged and synchronized between all players, including their locations, discovered environment maps, success probabilities and estimated task worths.
\end{rem}

Although the game specifics are different for the resilience and no-idling games, they follow the same design of the potential function and utility function as follows.

\vspace{6pt}
\textbf{Design of Potential Function for Task Reallocations}: As explained earlier in Section~\ref{sec:challenges}, the players must analyze the following optimization factors during task reallocation:
\begin{enumerate}
\item Task worths, which can be quantified by the expected number of undiscovered targets in the tasks.

\item Probability of success of each player to finish a certain task, which depends on its current battery reliability, the cost of traveling to the new task, and the cost of finishing the new task.
\end{enumerate}
Thus, the potential function $\phi$ for all players in the game is defined to be the total expected worth~\cite{AMS07} obtained by choosing a joint action $a_{\mathscr{P}} \in \mathscr{A}_{\mathscr{P}}$, as follows.

\vspace{-3pt}
\begin{equation}\label{eq:phi}
\phi(a_{\mathscr{P}}) = \sum\limits_{r\in \tilde{\mathscr{A}}} w_r \bigg( 1 - \underset{\mathscr{P}_i \in \{\mathscr{P}\}_r}{\prod} \big[ 1 - p_r(\mathscr{P}_i) \big] \bigg),
\end{equation}
where $\{\mathscr{P}\}_r \triangleq \{\mathscr{P}_i \in \mathscr{P}: a_i = r\}$ denotes the subset of players that choose the same task $r \in \{1,\ldots M\}$ in the joint action $a_{\mathscr{P}}$; $w_r$ is the current available worth of task $r$; and $p_r(\mathscr{P}_i)$ is the success probability of player $\mathscr{P}_i$ to finish task $r$. The term $p(r) := 1 - \prod_{\mathscr{P}_i \in \{\mathscr{P}\}_r} \big[ 1 - p_r(\mathscr{P}_i) \big]$ is the joint success probability for all players to finish task $r$ together.

As exploration continues, the conditions of robots and tasks change dynamically. Thus, the success probability $p_r(\mathscr{P}_i)$ and the task worth $w_r$ in Eq.~(\ref{eq:phi}) must be updated before a game is played.

\vspace{6pt}
\textbf{Computation of Success Probability}: The success probability $p_r(\mathscr{P}_i)$ is evaluated online using Eq. (\ref{eq:battery}) as follows.
\vspace{-3pt}
\begin{equation}\label{eq:successprobability}
p_r(\mathscr{P}_i) = R_{\mathscr{P}_i}(\tilde{t}),
\end{equation}
where $R_{\mathscr{P}_i}(\tilde{t})$ is the reliability of player $\mathscr{P}_i$ at time $\tilde{t}$, which is estimated as
\vspace{-3pt}
\begin{equation}\label{eq:totaltime}
\tilde{t} = t_k + t_{tr} + t_r,
\end{equation}
where $t_k$ is the total tasking time of $\mathscr{P}_i$ since the beginning until the game was initiated, $t_{tr}$ is the traveling time to task $r$, and $t_r$ is the estimated time to complete task $r$. Specifically, $t_{tr} = \frac{Dist(\mathscr{P}_i, r)}{u}$, where $Dist(\mathscr{P}_i, r)$ measures the distance between player $\mathscr{P}_i$'s current location and the centroid of task $r$, and $u \in \mathbb{R}^+$ is the robot's traveling speed; and the time $t_r = \frac{n_U(r)}{\omega}$, where $\omega$ is the speed of tasking a cell by the assigned robots.

In addition, if a robot is selected as a player to find a new task but it still has a small portion left in its current task, then it would like to first finish this task before being reallocated to a new one. Hence, an extra term $t_c$ is included in Eq. (\ref{eq:totaltime}) if the estimated time to complete the unfinished part of its current task $r_c(\mathscr{P}_i)$ satisfies $t_c(r_c(\mathscr{P}_i)) \leq \eta$.

\vspace{6pt}
\textbf{Computation of Task Worths}: The worth $w_r$ in Eq.~(\ref{eq:phi}) indicates the expected number of undiscovered targets in task $r$ that are available to the players. Let $\mathbf{x_r}$ be a random variable that denotes the total number of targets in task $r$, which is assumed to follow the Poisson distribution with parameter $\lambda_r$. Its probability mass function is given as:
\vspace{-3pt}
\begin{equation}\label{eq:poisson_pdf}
Pr\big(\mathbf{x_r} = x\big) = e^{-\lambda_r} \cdot \frac{{\lambda_r}^x}{x!}, \ x = 0, 1, 2\ldots
\end{equation}

If $\xi$ targets have been already discovered in task $r$, then the estimated remaining number of targets, $\tilde{w_r}$, is computed as:
\vspace{-3pt}
\begin{align}
  \tilde{w_r} &= \begin{aligned}[t]
      &\sum_{x=\xi+1}^{\infty} (x-\xi) \cdot e^{-\lambda_r} \cdot \frac{{\lambda_r}^x}{x!}
      \end{aligned} \nonumber\\
      &= \begin{aligned}[t]
      &\sum_{x = 0}^{\infty} x \cdot e^{-\lambda_r} \cdot \frac{{\lambda_r}^x}{x!} - \xi \cdot \sum_{x = 0}^{\infty} e^{-\lambda_r} \cdot \frac{{\lambda_r}^x}{x!} - \sum_{x = 0}^{\xi} (x-\xi) \cdot e^{-\lambda_r} \cdot \frac{{\lambda_r}^x}{x!}
      \end{aligned} \nonumber
\end{align}

By definition, Poisson distribution has mean $\lambda_r$, i.e., $\sum_{x = 0}^{\infty} x \cdot e^{-\lambda_r} \cdot \frac{{\lambda_r}^x}{x!} = \lambda_r$. Also, one has $\sum_{x = 0}^{\infty} e^{-\lambda_r} \cdot \frac{{\lambda_r}^x}{x!} = 1$. Thus, $\tilde{w_r}$ is computed as:

\vspace{-3pt}
\begin{equation}\label{eq:taskworth}
  \tilde{w_r} = (\lambda_r -\xi) + e^{-\lambda_r} \cdot \sum_{x = 0}^{\xi} (\xi-x) \cdot \frac{{\lambda_r}^x}{x!}
\end{equation}

Next, we decide the portion of $\tilde{w_r}$ available to the players, i.e., $w_r$. Since task $r$ may contain some robots currently working there but are not participating in the task reallocation, i.e., they are not players, then if a player selects task $r$, it must work together with these existing robots. In turn, the maximum payoff a player could expect from task $r$ becomes less due to sharing with the existing robots. Let $\bar{\mathscr{P}} \triangleq V \setminus \mathscr{P}$ denote the subset of robots that are not players. Similarly, let $\{\bar{\mathscr{P}}\}_r$ be the set of non-player robots that are currently working in task $r$, which have a joint success probability $q(r)$ for task $r$, i.e., $q(r)= 1 - \prod_{v_\ell \in \{\bar{\mathscr{P}}\}_r} \big(1 - p_r(v_\ell) \big)$. Then $w_r$ is computed as:
\vspace{-3pt}
\begin{equation}\label{eq:worth}
w_r = \tilde{w_r} \cdot \big( 1 - q(r) \big).
\end{equation}

\vspace{6pt}
\textbf{Utility Function of Each Player}: In order to form a potential game, the utility function, together with the potential function defined in Eq. (\ref{eq:phi}), must satisfy Eq. (\ref{eq:potentialgames}). Since the utility of a player also depends on the actions taken by the rest of the players, thus a rule is needed to distribute the total produced payoff among contributing players. In this regard, this paper adopts the concept of \textit{Marginal Contribution} due to its low computation burden thus feasible for online decision-making~\cite{AMS07}.

\vspace{-2pt}
\begin{defn}[\textbf{Marginal Contribution}]\label{def:MC}
The marginal contribution of player $\mathscr{P}_i$ in a joint action $a_{\mathscr{P}} = (a_i, a_{-i})$ is
\begin{equation}\label{eq:MC}
\mathscr{MC}_i = \phi(a_i, a_{-i}) - \phi(\emptyset, a_{-i}),
\end{equation}
where $\emptyset$ represents player $\mathscr{P}_i$'s null action, indicating no task is assigned to it.
\end{defn}

{The utility function is derived as follows. First, substitute Eq.~(\ref{eq:phi}) into Eq.~(\ref{eq:MC}), one has:
\vspace{-3pt}
\begin{align}
  \mathscr{U}_i(a_i, a_{-i}) &= \begin{aligned}[t]
      &\mathscr{MC}_i
      \end{aligned} \nonumber\\
      &= \begin{aligned}[t]
      &\sum\limits_{r\in \tilde{\mathscr{A}}} w_r \bigg( 1 - \underset{\mathscr{P}_j \in \{\mathscr{P}\}_r}{\prod} \big[ 1 - p_r(\mathscr{P}_j) \big] \bigg)\\
      &- \sum\limits_{r\in \tilde{\mathscr{A}}} w_r \bigg( 1 - \underset{\mathscr{P}_j \in \{\mathscr{P}\}_r \setminus \mathscr{P}_i}{\prod} \big[ 1 - p_r(\mathscr{P}_j) \big] \bigg)
      \end{aligned} \nonumber\\
       &= \begin{aligned}[t]
      &\sum\limits_{r\in \tilde{\mathscr{A}}} w_r \cdot \underset{\mathscr{P}_j \in \{\mathscr{P}\}_r\setminus \mathscr{P}_i}{\prod} [1 - p_r(\mathscr{P}_j)] \\
      &- \sum\limits_{r\in \tilde{\mathscr{A}}} w_r \cdot \underset{\mathscr{P}_j \in \{\mathscr{P}\}_r}{\prod} [1 - p_r(\mathscr{P}_j)] \nonumber
      \end{aligned}
\end{align}

Note that for any task $r$ not selected by player $\mathscr{P}_i$, i.e., $r \neq a_i$, one has $\{\mathscr{P}\}_{r} = \{\mathscr{P}\}_r \setminus \mathscr{P}_i$. Thus, the produced potentials in these tasks are canceled in the above equation. It can then be further simplified as below, where $w_{a_i}$ is the worth of task $a_i$.

\vspace{-6pt}
\begin{align}\label{eq:utility}
  \mathscr{U}_i(a_i, a_{-i}) &= \begin{aligned}[t]
      &w_{a_i} \cdot \underset{\mathscr{P}_j \in \{{\mathscr{P}}\}_{a_i}\setminus \mathscr{P}_i }{\prod} [1 - p_{a_i}(\mathscr{P}_j)] \\
      &- w_{a_i} \cdot \underset{\mathscr{P}_j \in \{{\mathscr{P}}\}_{a_i}}{\prod} [1 - p_{a_i}(\mathscr{P}_j)] \nonumber
      \end{aligned} \nonumber\\
       &= \begin{aligned}[t]
      &w_{a_i} \cdot \underset{\mathscr{P}_j \in \{\mathscr{P}\}_{a_i}\setminus \mathscr{P}_i}{\prod} [1 - p_{a_i}(\mathscr{P}_j)]\nonumber\\
      &- w_{a_i} \cdot [1-p_{a_i}(\mathscr{P}_i)] \underset{\mathscr{P}_j \in \{\mathscr{P}\}_{a_i}\setminus \mathscr{P}_i}{\prod} [1 - p_{a_i}(\mathscr{P}_j)]
      \end{aligned} \\
      &= \begin{aligned}[t]
      &w_{a_i} \cdot p_{a_i}(\mathscr{P}_i) \cdot \underset{\mathscr{P}_j \in \{\mathscr{P}\}_{a_i} \setminus \mathscr{P}_i }{\prod} [1 - p_{a_i}(\mathscr{P}_j)]
      \end{aligned}
\end{align}
}

\begin{prop}\label{theorem:potentialgames}
The game $G$ with potential function $\phi$ of Eq. (\ref{eq:phi}) and the utility function $\mathscr{U}_i$ of Eq. (\ref{eq:utility}) is a potential game.
\end{prop}
\begin{proof}
Given a joint action $a_{-i}$, the difference in potential $\phi$ when player $\mathscr{P}_i$ deviates its action from $a_i'$ to $a_i''$ is:
\begin{align*}
       & \phi(a_i', a_{-i}) - \phi(a_i'', a_{-i}) \\
       &= (\phi(a_i', a_{-i}) - \phi(\emptyset, a_{-i})) - (\phi(a_i'', a_{-i}) - \phi(\emptyset, a_{-i})) \\
       &=  \mathscr{U}_i(a_i', a_{-i}) - \mathscr{U}_i(a_i'', a_{-i})
\end{align*}
Thus game $G$ satisfies Eq. (\ref{eq:potentialgames}) and it is a potential game.
\end{proof}

In this paper, the optimal equilibrium $a_{\mathscr{P}}^\star$ is acquired using the Max-Logit algorithm~\cite{SWL11}. Before any game starts, each player computes its success probability $p_r(\mathscr{P}_i), \forall r \in \tilde{\mathscr{A}}$ using Eq. (\ref{eq:successprobability}), and updates the estimated task worth $w_r, \forall r \in \tilde{\mathscr{A}}$ using Eq. (\ref{eq:taskworth}). Then, necessary information are communicated and synchronized as mentioned in Remark~\ref{rem:data_synchronization}.

Algorithm~\ref{alg:optimizer} presents details to acquire $a_{\mathscr{P}}^\star$ in a distributed manner using Max-Logit. In particular, the initial joint action $a_\mathscr{P}(1)$ (\textit{line} 1) is initialized as follows: for resilience games, $a_i(1)$ is set as the current task $r_c(\mathscr{P}_i)$ of player $\mathscr{P}_i$, while for no-idling games, $a_i(1)$ is randomly picked from $\tilde{\mathscr{A}}$; then, $a_{\mathscr{P}}(1)$ is determined via synchronization with all other players. 

Once $a_{\mathscr{P}}^\star$ is obtained using Algorithm~\ref{alg:optimizer}, the new task $r$ for player $\mathscr{P}_i$ is set as its action $a_i^\star$ in the equilibrium $a_{\mathscr{P}}^\star$.

\vspace{3pt}
\textbf{Post-game Coordination}: If multiple robots (including both existing robots and incoming players) are assigned to the same task $r$, it becomes imperative to utilize some strategy to ensure their safety and efficiency when searching together. Let $n_{\max} \in \mathbb{N}^+$ be the maximum number of robots allowed to work in the same task at the same time. In this regard, task $r$ is evenly partitioned into $n_{\max}$ sub-regions, where each sub-region is only allowed one robot at a time.

{
\setlength{\textfloatsep}{-5pt}
\begin{algorithm}[t!]
\SetKwInOut{Input}{input}\SetKwInOut{Output}{output}
\SetKwComment{Comment}{}{}
\DontPrintSemicolon
\caption{The Optimizer for $\mathscr{P}_i$ using Max-Logit~\cite{SWL11}}\label{alg:optimizer}
\Input{$w_r$, $p_r(\mathscr{P}_i), \forall r \in \tilde{\mathscr{A}}$, and $\mathscr{P}_i \in \mathscr{P}$} 
\Output{$a_{\mathscr{P}}^\star$}
{Initialize: set initial joint action $a_{\mathscr{P}}(1) \in \mathscr{A}_{\mathscr{P}}$ using a fixed rule, set learning parameter $\tau$}\;
\For{$k \leftarrow 1$ \KwTo $L$}{
    {Determine randomly if $\mathscr{P}_i$ is the single player among others that may alter its action $a_i(k)$;} \;
    \uIf{$\mathscr{P}_i$ is not selected}{
		{Repeat $a_i(k+1) = a_i(k)$;}\;
	    {Continue;}\;
	}
	\Else{
	    {Select an alternative action $\hat{a_i}(k) \in \mathscr{A}_i$ with equal probability;}\;
	    {Compute alternative utility $\mathscr{U}_i(\hat{a_i}(k), a_{-i}(k))$ using Eq.(\ref{eq:utility});}\;
	    {Compute $\psi(\hat{a_i}) = e^{\mathscr{U}_i(\hat{a_i}, a_{-i})/\tau}$;}\;
	    {Compute $\mu = \psi(\hat{a_i})/{\max\{\psi(a_i), \psi(\hat{a_i})\}}$;}\;
	    {Update $a_i(k+1)$ as follows:
	    \vspace{-5pt}
        \[a_i(k+1) = \begin{cases}
                      \hat{a_i}(m),& \text{with probability } \mu \\
                      a_i(m),& \text{with probability } 1 - \mu.
                    \end{cases}\]}\;
        \vspace{-6pt}
        {Inform $a_i(k+1)$ to others $\mathscr{P}_j, j \in \mathscr{P}\setminus \mathscr{P}_i$;}\;
	}
}
    Return $a_{\mathscr{P}}^\star = a(k+1)$.
\end{algorithm}
}

Consider some non-player robot $v_\ell \in \{\bar{\mathscr{P}}\}_r$ that is currently working in task $r$. It continues as usual but its task is restricted to the sub-region determined by its current location. This produces $n_0 \in \mathbb{N}$ incomplete sub-regions that are instantly available to the incoming players. Now consider a player $\mathscr{P}_i \in \{\mathscr{P}\}_r^\star \triangleq \{\mathscr{P}_i \in \mathscr{P}: a_i^\star = r\}$ that is also assigned to task $r$. It selects the sub-region by following the steps below. First, it computes its rank in $\{\mathscr{P}\}_r^\star$ based on its success probability. If it ranks in the top $n_0$ and all other players ranked above it have selected their new sub-region, then it selects the new available sub-region for itself that minimizes its traveling distance. However, if it ranks after $n_0$, it stays temporarily idle but can later be reactivated to replace any robot in task $r$ should it fail. Once $\mathscr{P}_i$ finds a new sub-region, its centroid is set as the movement goal. As described previously, $\mathscr{P}_i$ resumes to search its new sub-region using the $\epsilon^\star$ algorithm upon its arrival, and its supervisor $H$ transitions to the state $WK$ accordingly.

\vspace{-3pt}
\subsection{\textbf{Computational Complexity of the Optimizer}}
As described above, once the \textit{Optimizer} triggers a game involving the player set $\mathscr{P}$ and the action set $\tilde{\mathscr{A}}$, the joint action $a_{\mathscr{P}}$ is first initialized locally and then synchronized with other players. This process takes $O(|\mathscr{P}|)$ complexity. 

Thereafter, the game follows Algorithm~\ref{alg:optimizer} in a distributed manner, which operates in a loop for a user-defined $L \in \mathbb{N}^+$ computation cycles. At each cycle, one player is randomly selected and is allowed to probabilistically alter its action, which takes $O(|\mathscr{P}|)$ to find out if $\mathscr{P}_i$ is selected. If not, its action is repeated, which takes $O(1)$ complexity; otherwise, $\mathscr{P}_i$ first randomly chooses an alternative action $\hat{a_i} \in \tilde{\mathscr{A}}$ with equal probability, which is $O(|\tilde{\mathscr{A}|})$. Then the associated utility $\mathscr{U}_i(\hat{a_i}, a_{-i})$ is computed using Eq. (\ref{eq:utility}), which takes $O(|\mathscr{P}|)$ complexity. Thereafter, $\mathscr{P}_i$ uses $\hat{a_i}$ to update its action $a_i$ in a probabilistic manner, which has $O(1)$. At the end of each cycle, the updated action $a_i$ is transmitted to other players, which requires $O(|\mathscr{P}|)$ complexity.

Therefore, in the worst case where $\mathscr{P}_i$ is selected in every cycle, the total complexity becomes $O(|\mathscr{P}| + L \cdot (3|\mathscr{P}| + |\tilde{\mathscr{A}}|))$. In comparison, for a centralized optimization algorithm, it must search over $|\tilde{\mathscr{A}}|^{|\mathscr{P}|}$ possible joint actions, which grows significantly faster as $|\mathscr{A}|$ and $|\mathscr{P}|$ increase.

\vspace{-6pt}
\subsection{\textbf{Connection between Local Games and Team Potential}}\label{sec:connection}
As discussed earlier, in both resilience games and no-idling games, the potential function $\phi$ is optimized for the set of players, which form a subset of the robot team. Now, we show that the increase in $\phi$ will directly improve the performance of the whole team.

To illustrate this, let $\Phi(a)$ denote the total team potential that defines the total expected worth achievable by the team, where $a = (a_{\mathscr{P}}, a_{\bar{\mathscr{P}}})$ is the joint action of the team including players $\mathscr{P}$ and non-players $\bar{\mathscr{P}}$. Note for the non-players, the action $a_{\bar{\mathscr{P}}}$ simply represent their current tasks. Since the players and non-players are mixed and distributed over different tasks, the total team potential $\Phi(a)$ is defined as:

\vspace{-3pt}
\begin{equation}\label{eq:cap_phi}
\Phi(a) = \sum\limits_{r=1}^M \tilde{w_r} \Big( 1 - \prod_{v_\ell \in \{V\}_r} [1-p_r(v_\ell)] \Big),
\end{equation}
where $\{V\}_r = \{\mathscr{P}\}_r \cup \{\bar{\mathscr{P}}\}_r$ is the set of all robots that are assigned to task $r$ in the joint action $a$, and the term within the parentheses on the right hand side computes the joint success probability to complete task $r$ by all of its assigned robots.

As the players reach the optimal equilibrium, the joint action becomes $a^\star = (a_{\mathscr{P}}^\star, a_{\bar{\mathscr{P}}})$ and the team potential becomes $\Phi(a^\star)$.

\vspace{3pt}
\begin{thm}\label{thm:resilience}
The optimal equilibrium $a^\star$ increases the total team potential $\Phi(a)$, i.e., $\Phi(a^\star) \geq \Phi(a)$.
\end{thm}
\begin{proof}
First, let us show that the team potential $\Phi(a)$ is separable by the worth created by the players (i.e., $\mathscr{P}$) and the rest of the robots (i.e., $\bar{\mathscr{P}}$). Then we will investigate the change of $\Phi$ due to task reallocation. Now $\Phi(a)$ can be decomposed as follows.
\vspace{-3pt}
\begin{align}\label{eq:team_phi}
  \Phi(a) &= \begin{aligned}[t]
      & \sum\limits_{r=1}^M \tilde{w_r} \Big( 1 - \prod_{\mathscr{P}_i \in \{\mathscr{P}\}_r} [1-p_r(\mathscr{P}_i)] \cdot \prod_{v_\ell \in \{\bar{\mathscr{P}}\}_r} [1-p_r(v_\ell)] \Big)
      \end{aligned} \nonumber\\
      &= \begin{aligned}[t]
      &\sum\limits_{r=1}^M \tilde{w_r} \Big(1- [1 - p(r)]\cdot [1-q(r)] \Big)
      \end{aligned}
\end{align}
where $p(r) := 1 - \prod_{\mathscr{P}_i \in \{\mathscr{P}\}_r} [1-p_r(\mathscr{P}_i)]$ and $q(r) := 1 - \prod_{v_\ell \in \{\bar{\mathscr{P}}\}_r} [1-p_r(v_\ell)]$ are used to denote the joint success probability of the players and the rest of the robots that are assigned to task $r$ in the joint action $a$, respectively.

Then we can further break down $\Phi(a)$ as follows.
\vspace{-6pt}
\begin{align}
  \Phi(a) &= \begin{aligned}[t]
      &\sum\limits_{r=1}^M \tilde{w_r} \bigg( 1 - [1-q(r)] + p(r)[1-q(r)] \bigg)
      \end{aligned} \nonumber\\
      &= \begin{aligned}[t]
      &\sum\limits_{r=1}^M \tilde{w_r} \cdot [1-q(r)] \cdot p(r) + \sum\limits_{r=1}^M \tilde{w_r} \cdot q(r)
      \end{aligned} \nonumber\\
      &= \begin{aligned}[t]
      &\sum\limits_{r=1}^M w_r \cdot p(r) + \sum\limits_{r=1}^M \tilde{w_r} \cdot q(r)
      \end{aligned} \nonumber\\
      &= \begin{aligned}[t]
      &\bigg( \sum_{r \in \tilde{\mathscr{A}}} w_r \cdot p(r) + \sum_{r \notin \tilde{\mathscr{A}}} w_r \cdot p(r) \bigg) + \sum\limits_{r=1}^M \tilde{w_r} \cdot q(r)
      \end{aligned} \nonumber\\
      &= \begin{aligned}[t]
      &\phi(a_{\mathscr{P}}) + \sum_{r \notin \tilde{\mathscr{A}}} w_r \cdot p(r) + \sum\limits_{r=1}^M \tilde{w_r} \cdot q(r)
      \end{aligned}
\end{align}
where the second term in the last step is the worth generated by the players (if any) that would like to finish the small unfinished part in their current tasks before being reallocated to new tasks, while the third term indicates the worth generated by the non-player robots $\bar{\mathscr{P}}$. The values of both these terms do not change by games. Since $\phi(a_{\mathscr{P}}^\star) \geq \phi(a_{\mathscr{P}})$, $\forall a_{\mathscr{P}} \in \mathscr{A}_{\mathscr{P}}$, we have $\Phi(a^\star) \geq \Phi(a)$.
\end{proof}

\textbf{Game Performance Metrics:} The quality of the task reallocation decision (i.e., $a_{\mathscr{P}}^\star$ for the players and $a^\star$ for the team) can be evaluated by the worth gain. Note that in any task reallocation,  there is a tradeoff  between  whether  the  robot  should continue  with  its  current  task  or  reallocate  to  a  new task. Thus, where a higher gain implies early detection of targets. Specifically, at player-level, the \textit{Gain of Players} ($G_P$) is defined as

\begin{equation}\label{eq:GP}
G_P = \frac{\phi(a_{\mathscr{P}}^\star)-\phi(a_{\mathscr{P}})}{\sum_{r \in \tilde{\mathscr{A}}} w_r} \in [0, 1].
\end{equation}

Similarly, at team-level, the \textit{Gain of Team} ($G_T$) is

\begin{equation}\label{eq:GT}
G_T = \frac{\Phi(a^\star)-\Phi(a)}{\sum_{r = 1}^M \tilde{w_r}} \in [0, 1].
\end{equation}

Note that since $\phi(a_{\mathscr{P}}^\star) \geq \phi(a_{\mathscr{P}})$ and $\Phi(a^\star) \geq \Phi(a)$, both $G_p$ and $G_T$ are non-negative, which implies that the outcome of a game results in the gain of worth not only for the players but also for the whole team. Both $G_P$ and $G_T$ will be quantitatively examined in Section~\ref{sec:results}.

\vspace{-3pt}
\subsection{\textbf{Complete Coverage under Failures}}\label{sec:complete_coverage}
The success of finding all hidden targets relies on the complete coverage of the whole area $\mathcal{R}$. Due to the completeness of the underlying single-robot coverage algorithm~\cite{SG17}, each task can be fully covered by the assigned robot in finite time if it stays alive. Now, let us examine coverage under failures.
\begin{thm}\label{thm:coverage}
The CARE algorithm guarantees complete coverage in finite time as long as one robot is alive.
\end{thm}
\begin{proof}
Consider a robot $v_\ell$ that is alive during the whole operation, whose supervisor $H$ starts with the state $WK$ upon robot being turned on. We show below that $v_\ell$ must reach the terminal state $SP$ in finite time, which happens if and only if $\sum_{r=1}^M n_U(r) = 0$, i.e., complete coverage.

First, as shown in Fig.~\ref{fig:TE}, any cycle between states in $H$ involves either state $NG$ or $RG$. Also, a robot can reach the states $NG$ or $RG$ due to completion of some task or failure of some robot, respectively. Now, since there are only a finite number of robots (i.e., $N$) and a finite number of tasks (i.e., $M$), each robot can visit these states only a finite number of times. Thus, $H$ cannot have any live lock. Moreover, in states $NG$ or $RG$, it takes a finite amount of time to reach an equilibrium solution using the Max-Logit algorithm. Thus, $H$ will always switch to either state $WK$ or $ID$ after games. In state $WK$, the underlying $\epsilon^\star$ algorithm is used to explore in the current assigned task $r_c(v_\ell)$ of robot $v_\ell$. As shown in~\cite{SG17}, $\epsilon^\star$ constantly reduces $n_U(r_c(v_\ell))$ until task $r_c(v_\ell)$ is completed in finite time, so $H$ can only stay in state $WK$ for finite time.

Further, in state $ID$, $H$ can either be invoked to play new games and hence move to states $NG$ or $RG$, or it can move to the state $SP$ upon complete coverage, i.e., $\sum_{r=1}^M n_U(r) = 0$. Since the former case can only happen for a finite number of times, $H$ will come back to the state $ID$ when no incomplete task is available to $v_\ell$ anymore. The same logic applies to all other active robots. Thus, all active robots including $v_\ell$ will reach state $ID$ in finite time, which implies that no incomplete tasks exist, i.e., $\sum_{r=1}^M n_U(r) = 0$. Then, they all transition to the terminal state $SP$ and the complete coverage is achieved.
\end{proof}

\section{Results and Discussion}\label{sec:results}
The CARE algorithm was validated on the high-fidelity robotic simulation platform called Player/Stage~\cite{GVH03} using a computer with $3.40$ GHZ CPU and $16$ GB RAM. The Player provides a software base whose libraries contain models of different types of robots, sensors and actuators. On the other hand, Stage is a highly configurable robot simulator.

In this section, we present the performance of the CARE algorithm in three complex obstacle-rich scenarios. The search area $\mathcal{R}$ of size $50m \times 50m$, was partitioned into a $50 \times 50$ tiling consisting of $\epsilon$-cells of size $1m \times 1m$. The $\mathcal{R}$ was partitioned into $M = 10$ tasks $\{\mathcal{R}_r: r=1, \ldots 10\}$, each of size $10m \times 25m$. Each task $r\in\{1,\ldots 10\}$ was initially assigned with one robot, and a maximum number of $n_{\max} = 4$ robots were allowed to search together in one task. Each task $r$ contained an unknown number of targets distributed randomly according to the Poisson distribution with mean $\lambda_r \in \{1,\ldots 32\}$.

A team of $N = 10$ Pioneer 2AT robots was simulated, where each robot has dimensions of $0.44m \times 0.38m \times 0.22m$, and was equipped with $16$-beam laser scanners with a detection range of $r_s = 5m$. The kinematic constraints of the robot, such as the top speed of $0.4m/s$ and the minimum turn radius of $0.04m$, were included in the simulation model. The tasking speed was set as $\omega = 0.32$cells/s. The parameters $\rho_0$ and $\rho_1$ in the battery reliability model are chosen such that each robot can finish one and two tasks with more than $0.9$ and $0.4$ remaining reliability, respectively. Specifically, based on the size of each task ($250$ cells) and the robot tasking speed, it takes $\sim 780s$ to finish an obstacle-free task. Then, using Eq. (\ref{eq:battery}):

\vspace{-3pt}
\begin{equation}
\centering
\begin{cases} 
\frac{1}{1+e^{\rho_0(780-\rho_1)}} &= 0.9 \\ 
\frac{1}{1+e^{\rho_0(2\times780-\rho_1)}} &= 0.4  
\end{cases}, \nonumber
\end{equation}
which lead to $\rho_0 \sim 3.0\times10^{-3}$ and $\rho_1 \sim 1400$s. Then, considering stochastic uncertainties in the initial battery charging conditions, these parameters are generated on different robots using Gaussian distributions, s.t., $\rho_0\sim N(3\times 10^{-3}, 7.5\times 10^{-5})$ and $\rho_1\sim N(1400, 35)$, where the standard deviation is chosen as $2.5\%$ of the corresponding mean value.

Initially, due to lack of \textit{a priori} knowledge of the environment, all $\epsilon$-cells are initialized as unexplored, and as the robot explores the environment, the obstacle and forbidden cells are discovered and updated accordingly. The game parameters are chosen as: $\kappa_1 = 6$ and $\kappa_2 = 3$, and in Max-Logit, the number of computation cycles is set as $L = 50$ and learning parameter is $\tau = 0.05$. The other parameters $\eta$ and $\gamma$ are chosen as follows. We set $\eta$ such that it corresponds to less than $4\%$ of the time to finish one task, i.e., $780 \times 4\% = 31.2s$. Thus, robots which have only $4\%$ of the task left will participate as players for no-idling games. Similarly, we set $\gamma$ such that it corresponds to over $25\%$ of the time to finish one task, i.e., $780 \times 25\% = 195s$. Thus, tasks which have still more than $25\%$ unexplored area become contested tasks. Hence, further rounding up we used $\eta = 30s$ and $\gamma = 200s$.

\subsection{\textbf{Scenario 1: No Failures but Some Robots Idle}}\label{sec:scenario1}
Fig.~\ref{fig:res1} presents the cooperative coverage of a complex islands scenario. A total number of $107$ targets were distributed randomly in the field. No failure appeared throughout the whole search, while two no-idling games were triggered to reallocate early completed robots to reduce the coverage time. Each subfigure in Fig.~\ref{fig:res1}, i.e., Fig.~\ref{fig:res1}(1)$\sim$\ref{fig:res1}(8), is comprised of a top figure showing the trajectories of robots by different colors, and a bottom figure showing the corresponding overall symbolic map of the entire search area $\mathcal{R}$, which is periodically synchronized and merged by all live robots. The different colors in the symbolic map represent the following regions: i) light green for obstacles, ii) medium green for unexplored, iii) dark green for explored with no obstacles, and iv) yellow for the forbidden region around the obstacles.

Fig.~\ref{fig:res1}$(1)$ shows that the robots started exploration and used their on-board sensing systems to explore the \textit{a priori} unknown environment. Fig.~\ref{fig:res1}(2) shows that the robots continue searching within their assigned tasks. Fig.~\ref{fig:res1}$(3)$ shows the instance when robot $v_{10}$ finished task $10$ and triggered a no-idling game $G_1$. The player set was formed as $\mathscr{P} = \{v_5, v_{10}\}$, where $v_5$ was near finishing its task. At that moment, tasks $3$, $4$, $7$ and $9$ still had a lot of area unexplored and required significant time to finish by their currently assigned robots, thus they formed the action set $\tilde{\mathscr{A}} = \{3, 4, 7, 9\}$. The optimal equilibrium of $G_1$ reassigned $v_5$ and $v_{10}$ to task $4$ and task $9$, respectively. Because task $10$ was already completed, $v_{10}$ immediately traveled to its new task $9$, while $v_5$ had to first finish the remainder of its current task $5$ before moving to task $4$. Since there was a robot $v_9$ currently working in task $9$, the post-game coordination strategy was used to further partition task $9$ into $n_{\max} = 4$ sub-regions. As observed in Fig.~\ref{fig:res1}$(4)$, $v_{10}$ selected the closest sub-region in the upper right corner and searched in parallel with $v_9$. Similar post-game coordination was performed when $v_5$ joined to search with $v_4$ in task $4$.

\begin{figure*}[!ht]
     \centering
     \vspace{-5pt}
     \subfloat[Coverage trajectories and corresponding symbolic map of the environment map using CARE]{
     \includegraphics[width=0.95\textwidth]{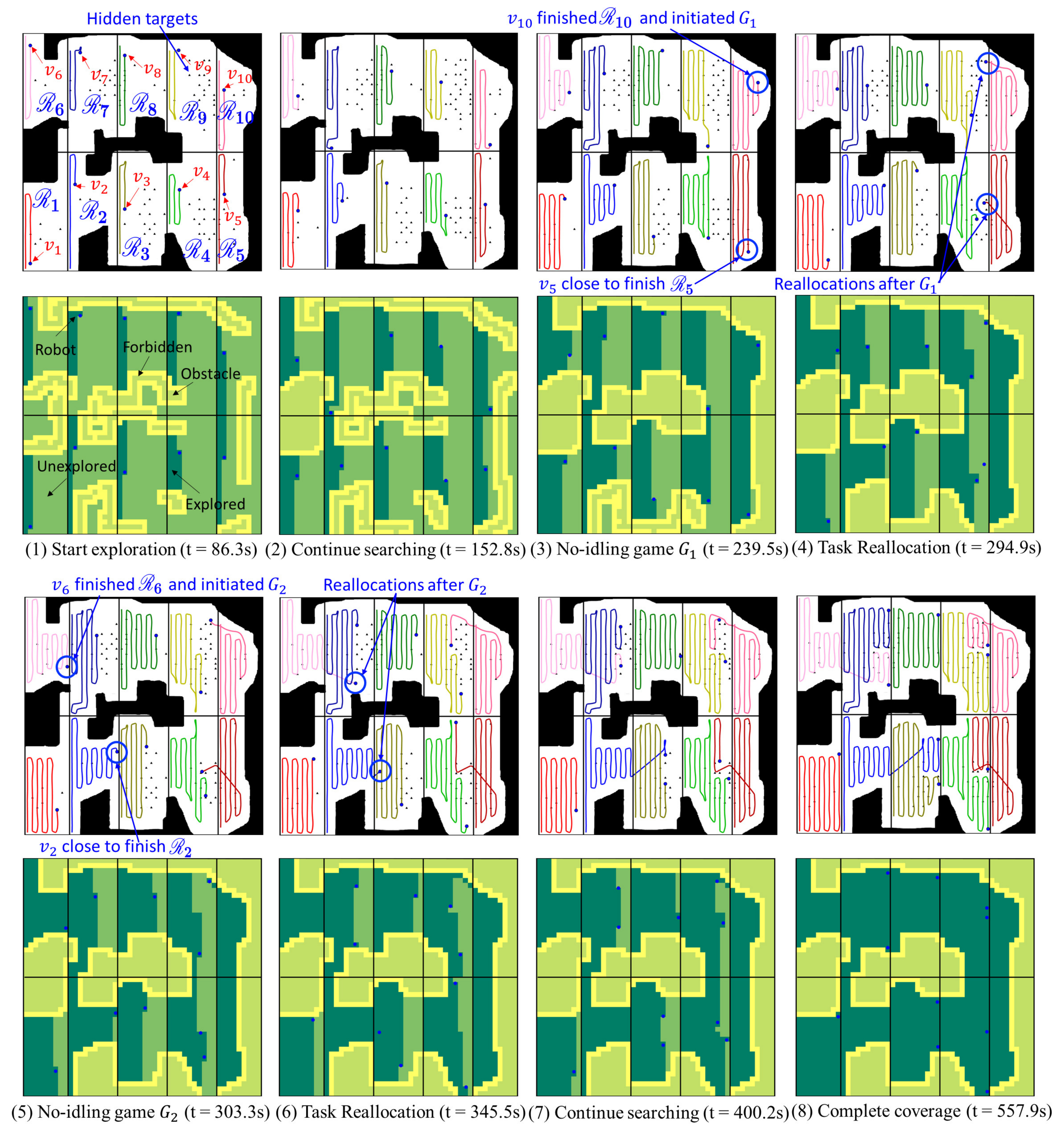}\label{fig:res1}}
     \\ \vspace{-6pt}
     \subfloat[Summary of game specifics and performances]{
     \includegraphics[width=0.9\textwidth]{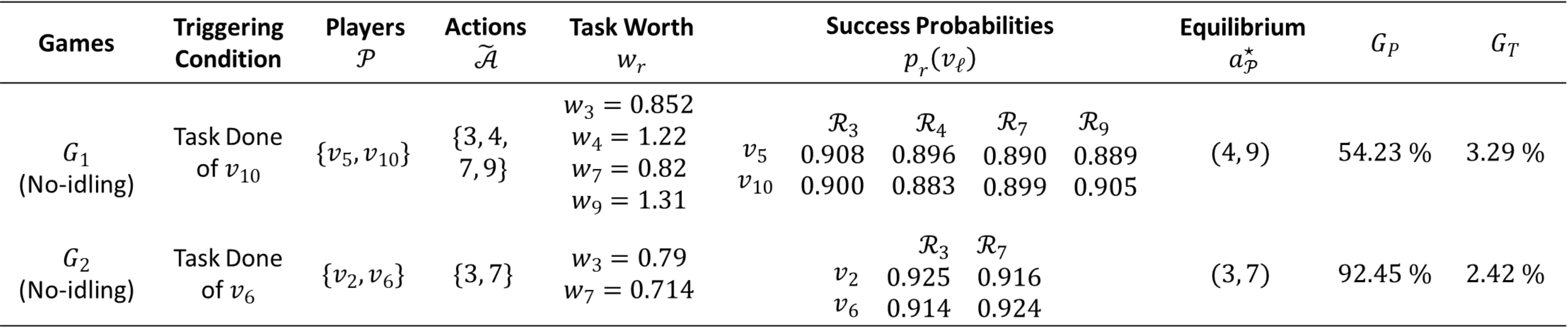}\label{fig:res1_table}}
     \caption{\textbf{Scenario $1$}: Incremental discovery and efficient coverage using CARE.}\label{fig:scenario1}
\end{figure*}

Later, another no-idling game $G_2$ was triggered when $v_6$ finished task $6$, as shown in Fig.~\ref{fig:res1}$(5)$. The player set was formed as $\mathscr{P}=\{v_2, v_6\}$, where $v_2$ was near finishing its task. Since tasks $4$ and $9$ have been assigned with extra robots after game $G_1$, the estimated time to finish these tasks dropped significantly, hence they were excluded from game $G_2$. The tasks $3$ and $7$, however, still required significant time to finish, thus they formed the action set $\tilde{\mathscr{A}} = \{3, 7\}$. The optimal equilibrium of $G_2$ in turn reassigned $v_2$ and $v_6$ to task $3$ and task $7$, respectively, as shown in Fig.~\ref{fig:res1}$(6)$. It is observed in Fig.~\ref{fig:res1}(7) that, robot $v_2$ selected the upper right sub-region of task $3$ and continued searching in parallel with robot $v_3$, while robot $v_6$ joined robot $v_7$ to search task $7$ in a similar fashion. Finally, complete coverage was achieved with all targets discovered, as shown in Fig.~\ref{fig:res1}(8).

\begin{figure*}[!ht]
     \centering
     \vspace{-25pt}
     \subfloat[Coverage trajectories and corresponding symbolic map of the environment map using CARE]{
     \includegraphics[width=0.95\textwidth]{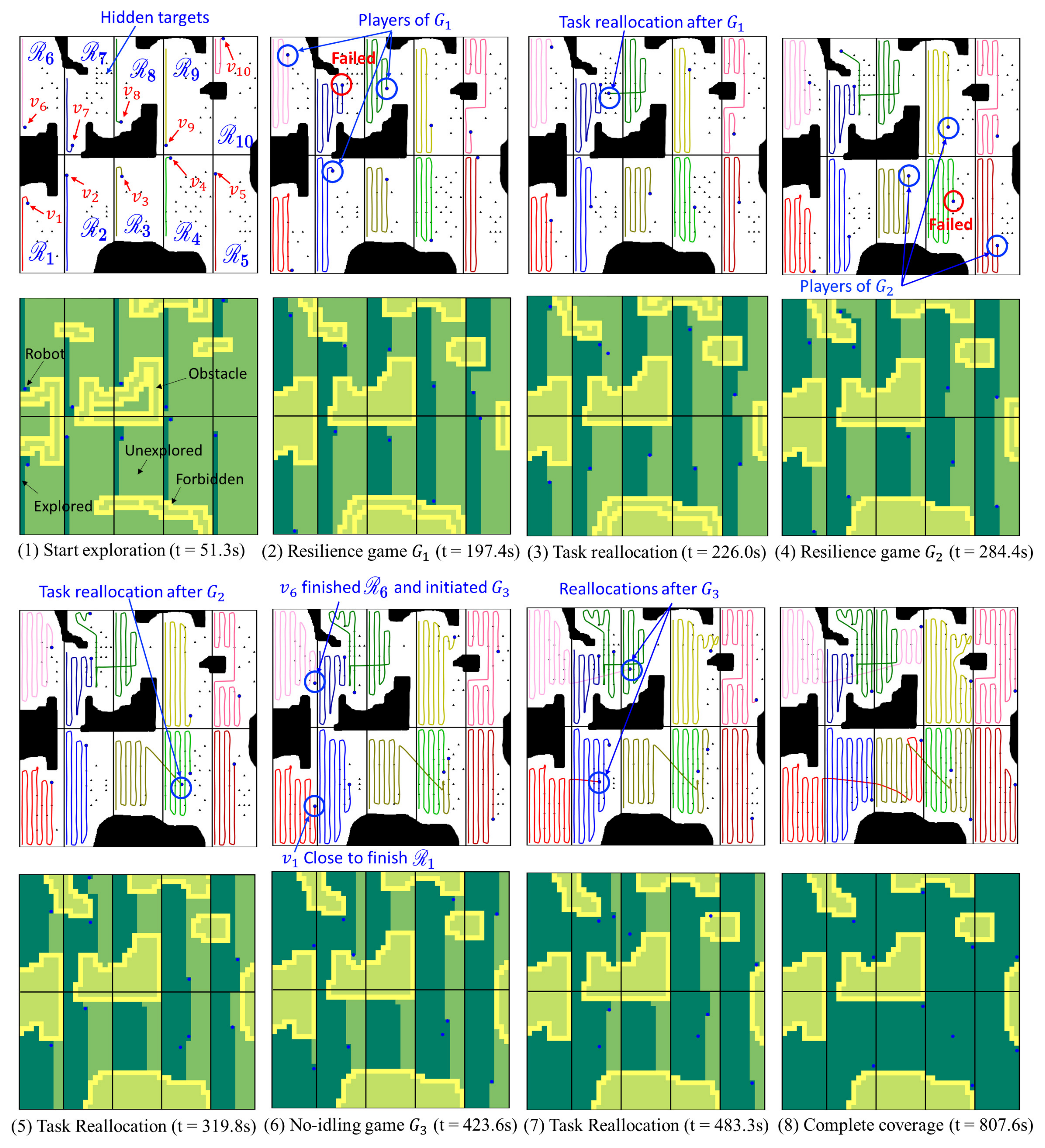}\label{fig:res2}}
     \\ \vspace{-6pt}
     \subfloat[Summary of game specifics and performances]{
     \includegraphics[width=0.9\textwidth]{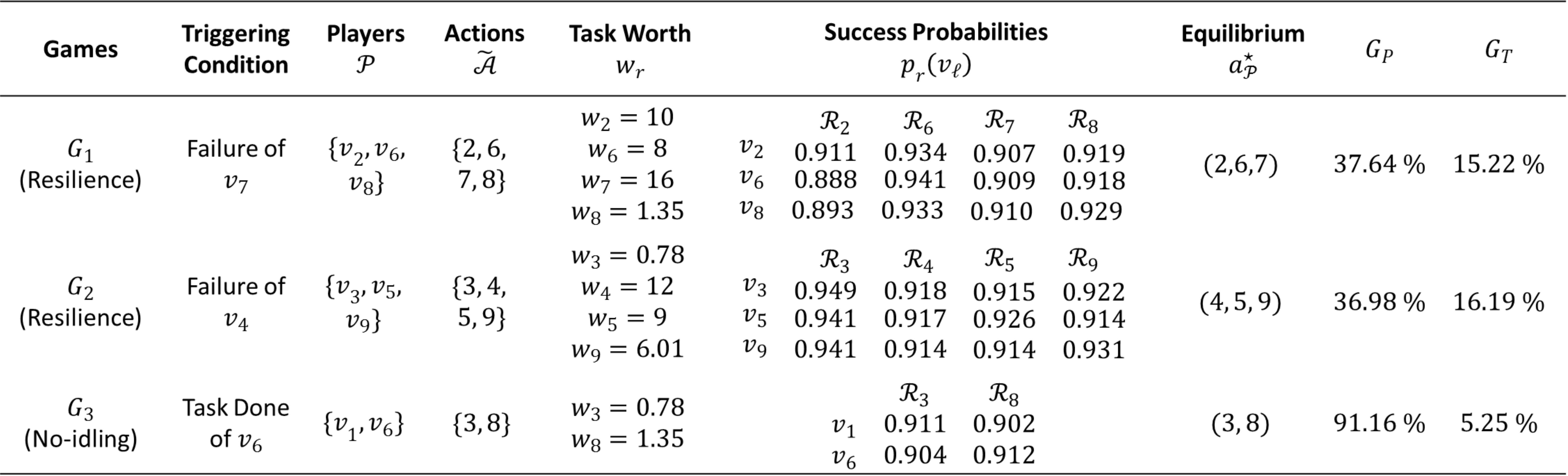}\label{fig:res2_table}}
     \caption{\textbf{Scenario $2$}: Incremental discovery and efficient coverage using CARE.}\label{fig:scenario2}\vspace{-10pt}
\end{figure*}

Fig.~\ref{fig:res1_table} summarizes the specifics and performance of the two games. As observed, the player-level worth gain $G_P$ reached $54.23\%$ and $92.45\%$ in games $G_1$ and $G_2$, respectively, which means that after task reallocations, the idling robots can expect a higher number of targets from the remaining tasks. At the team-level, $G_T$ is $3.29\%$ for $G_1$ and $2.42\%$ for $G_2$, thus the whole team also benefits from the task reallocations.

\begin{figure*}[t!]
     \centering
     \vspace{-15pt}
     \subfloat[CARE]{
     \includegraphics[width=0.24\textwidth]{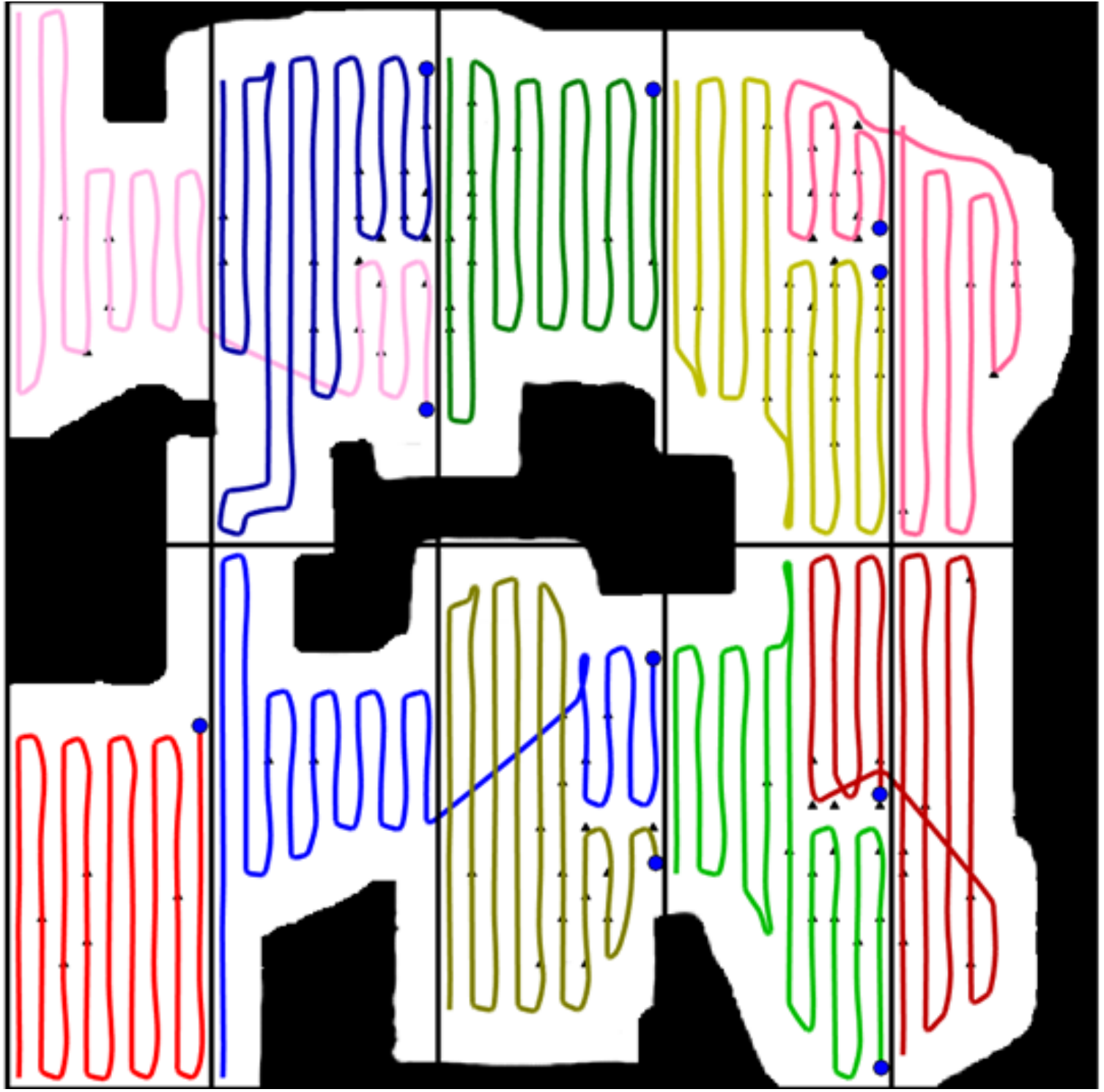}\label{fig:res1_others_care}}
     \subfloat[Non-cooperative Coverage]{
     \includegraphics[width=0.24\textwidth]{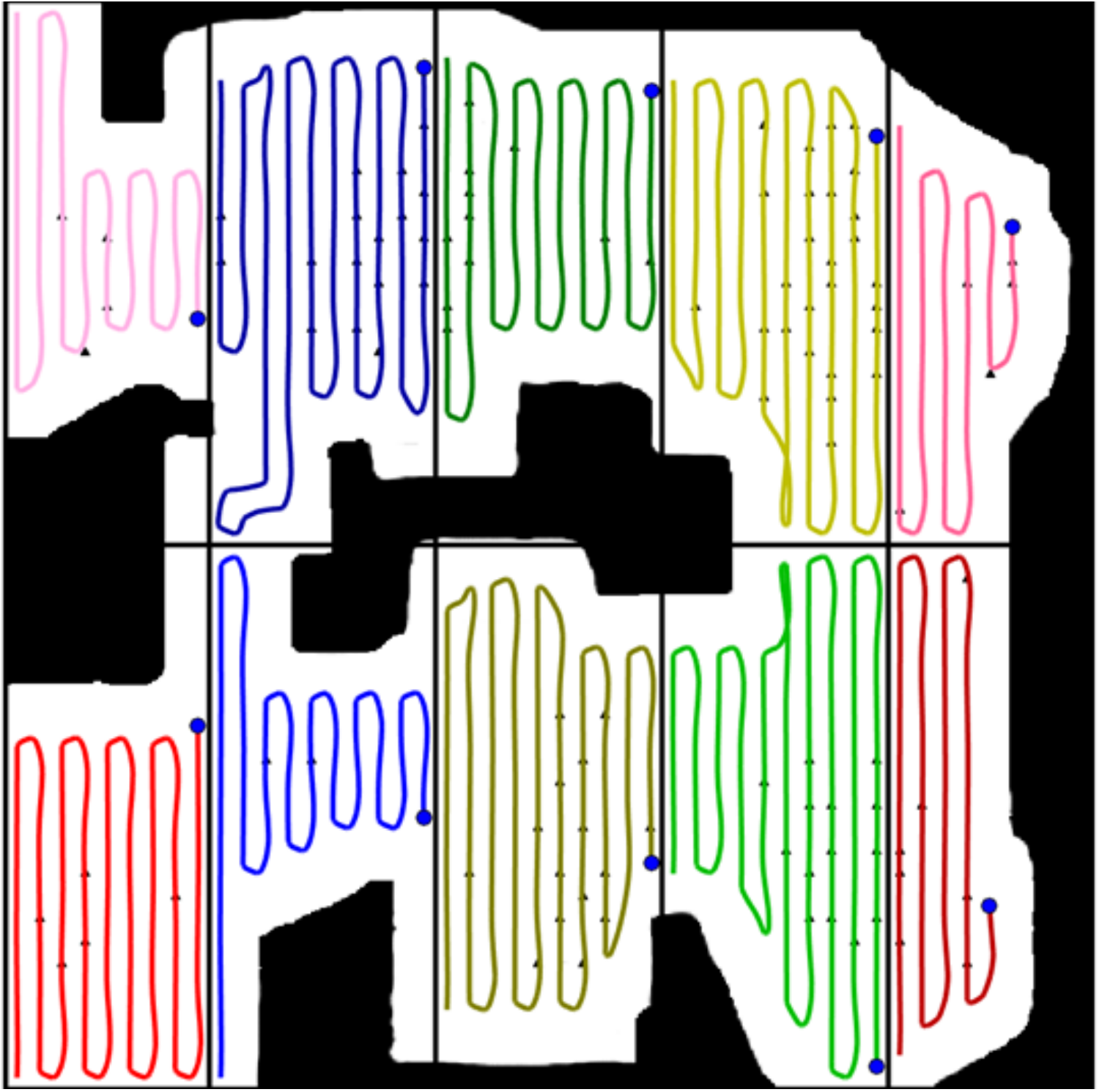}\label{fig:res1_others_nocoop}}
     \subfloat[First-responder Coverage]{
     \includegraphics[width=0.24\textwidth]{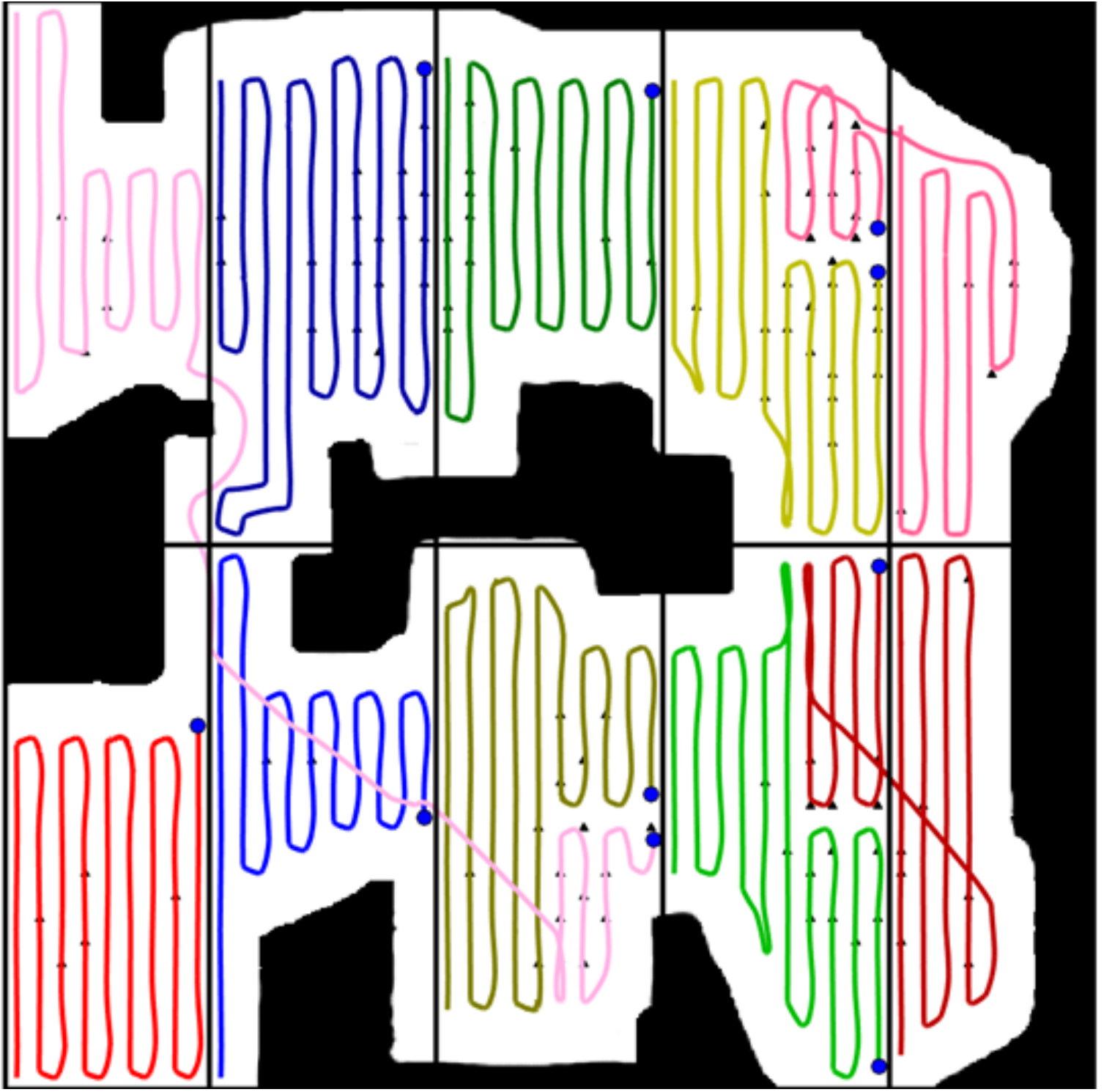}\label{fig:res1_others_first}}
     \subfloat[Brick and Mortar]{
     \includegraphics[width=0.24\textwidth]{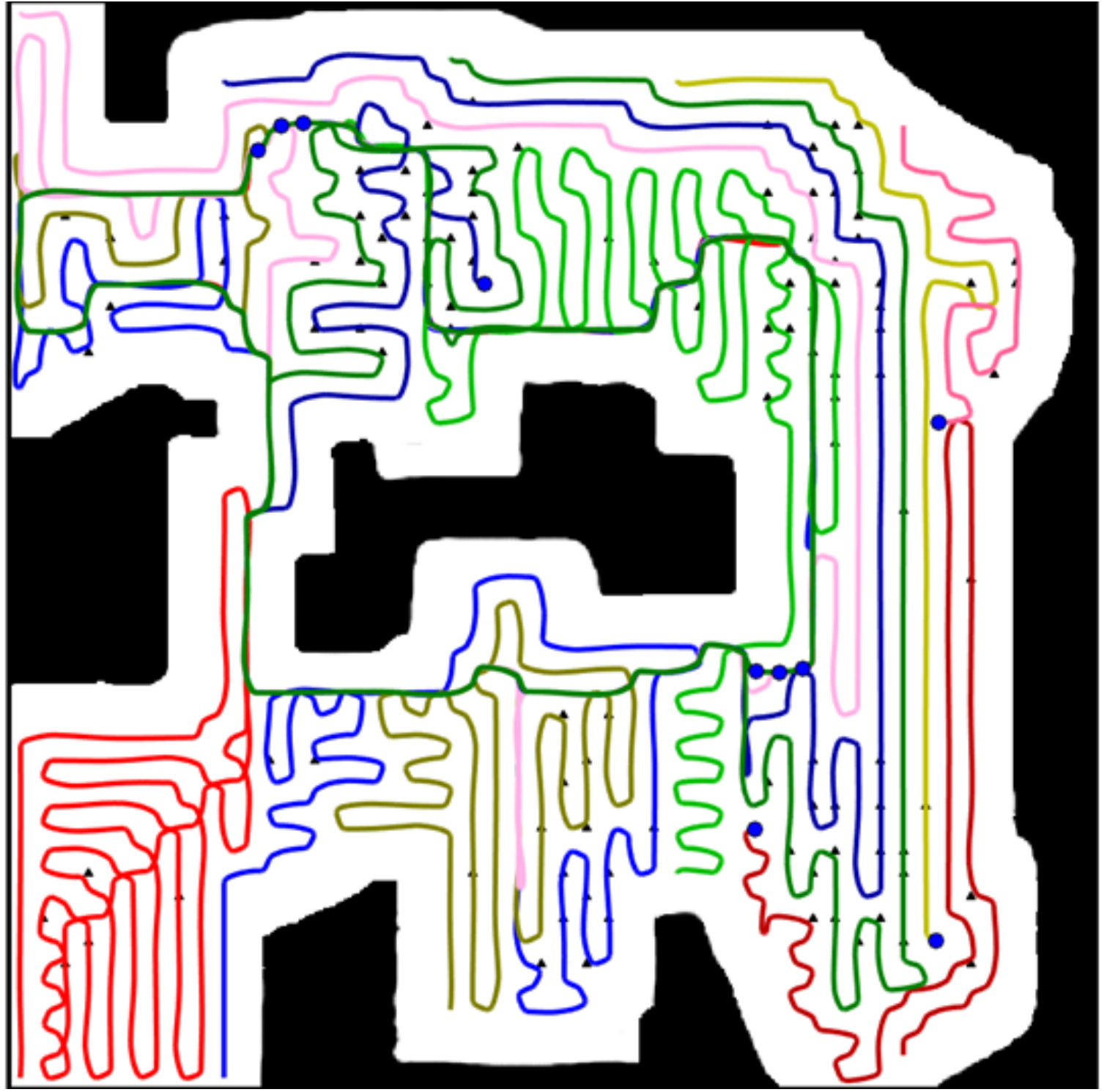}\label{fig:res1_others_bm}}
     \caption{\textbf{Scenario $1$}: Comparison of coverage trajectories using different online multi-robot coverage methods}
      \label{fig:res1_others} \vspace{-8pt}
\end{figure*}

\begin{figure*}[t!]
     \centering
     \subfloat[Coverage Time ($CT$)]{
     \includegraphics[width=0.28\textwidth]{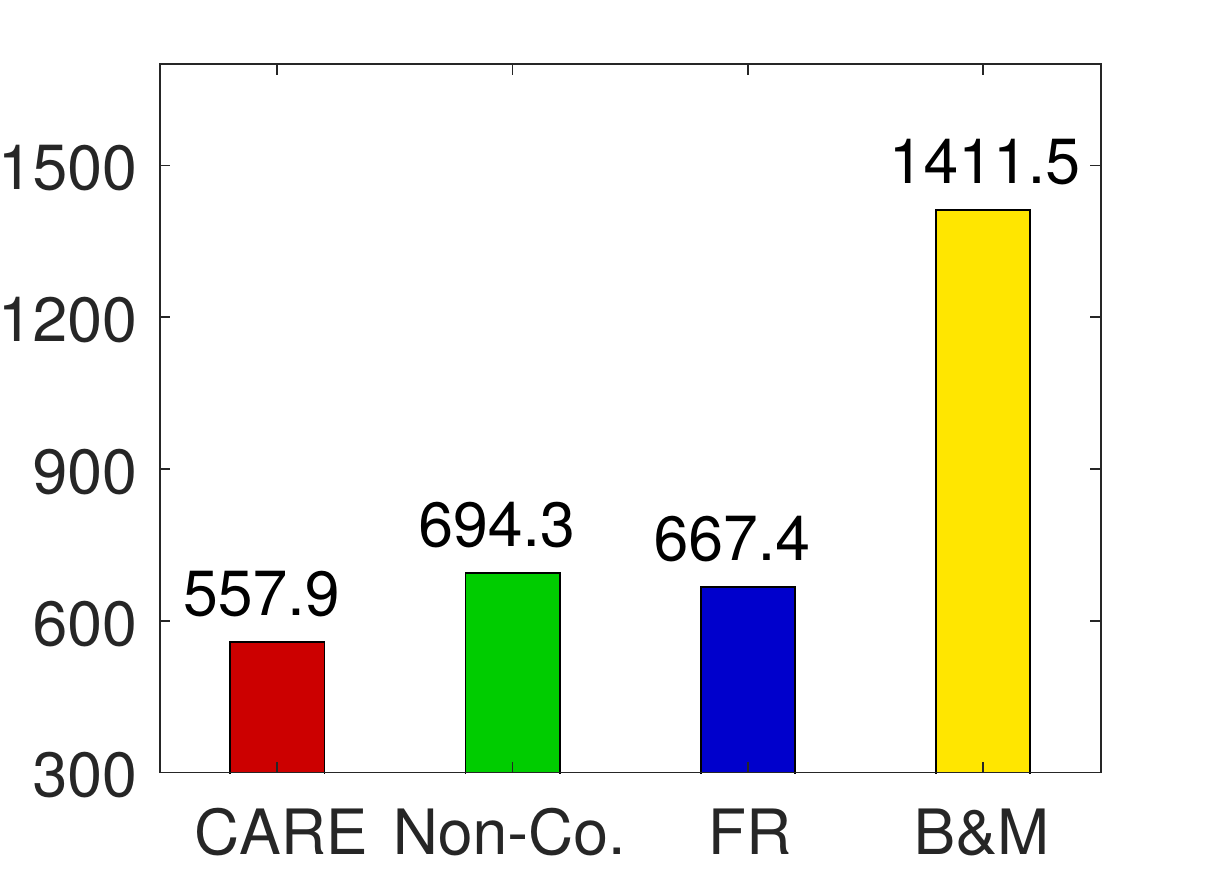}\label{fig:CT_1}}
     \subfloat[Coverage Ratio ($CR$)]{
     \includegraphics[width=0.28\textwidth]{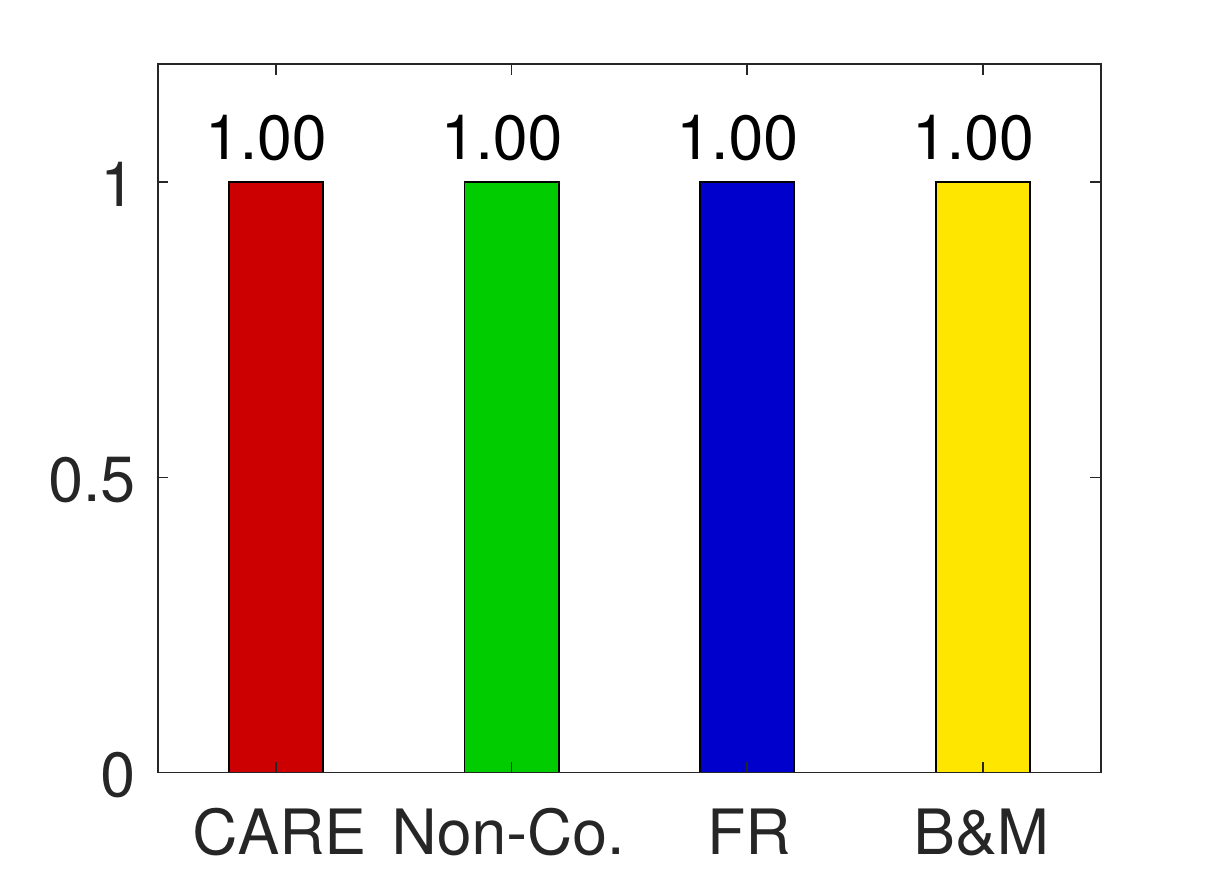}\label{fig:CR_1}}
     \subfloat[Remaining Reliability ($RR$)]{
     \includegraphics[width=0.28\textwidth]{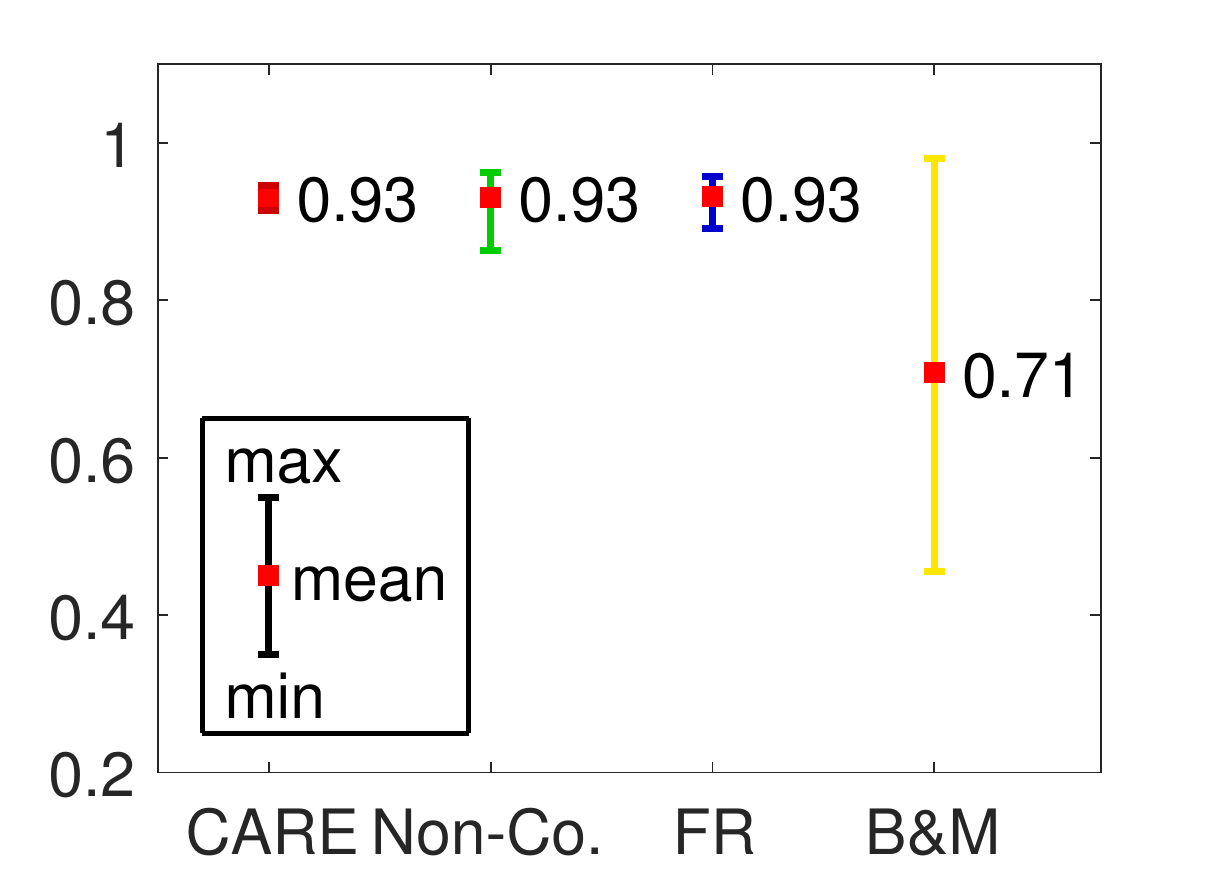}\label{fig:RR_1}} \\
     \hspace{-15pt}
     \subfloat[Number of Targets Found ($NoTF$)]{
     \includegraphics[width=0.28\textwidth]{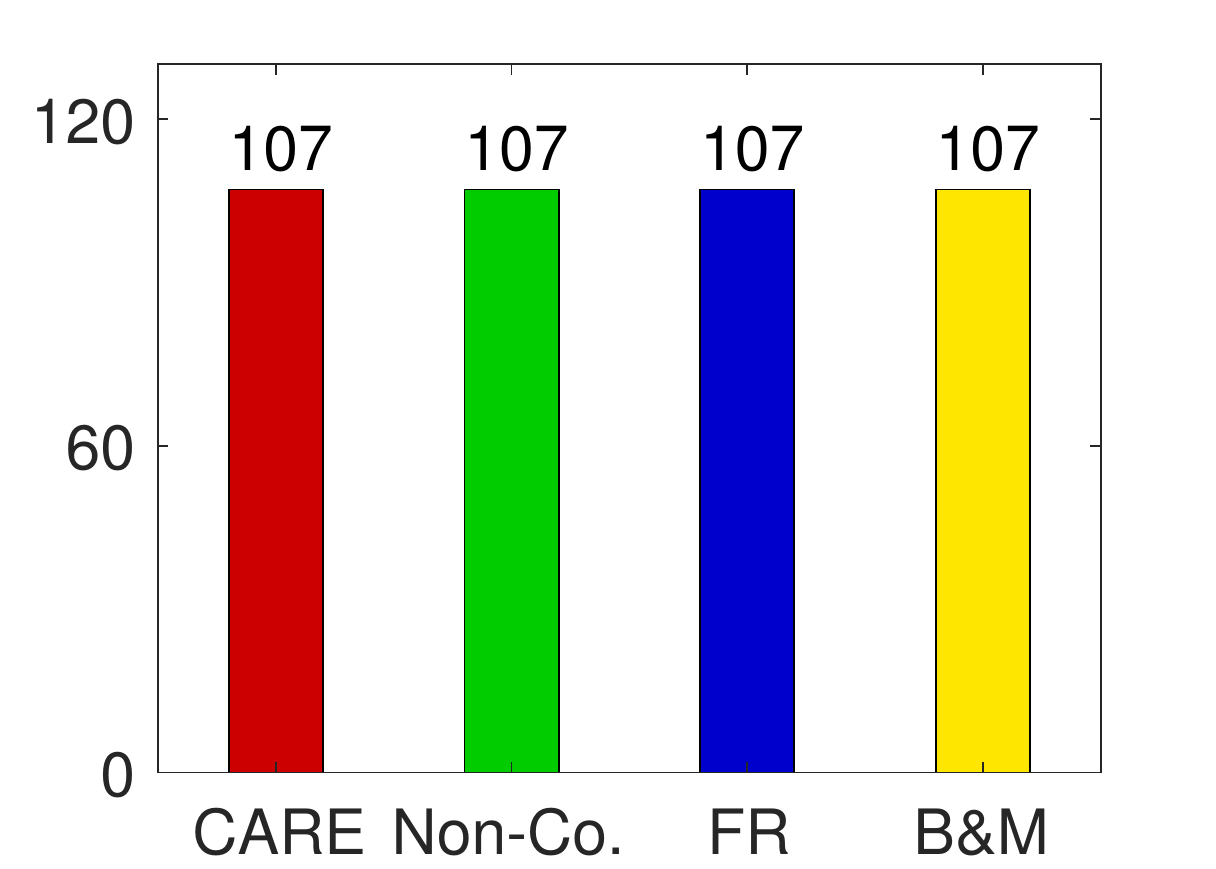}\label{fig:NoFT_1}} \hspace{-15pt}
     \subfloat[Time of Target Discovery $ToTD$]{
     \includegraphics[width=0.47\textwidth]{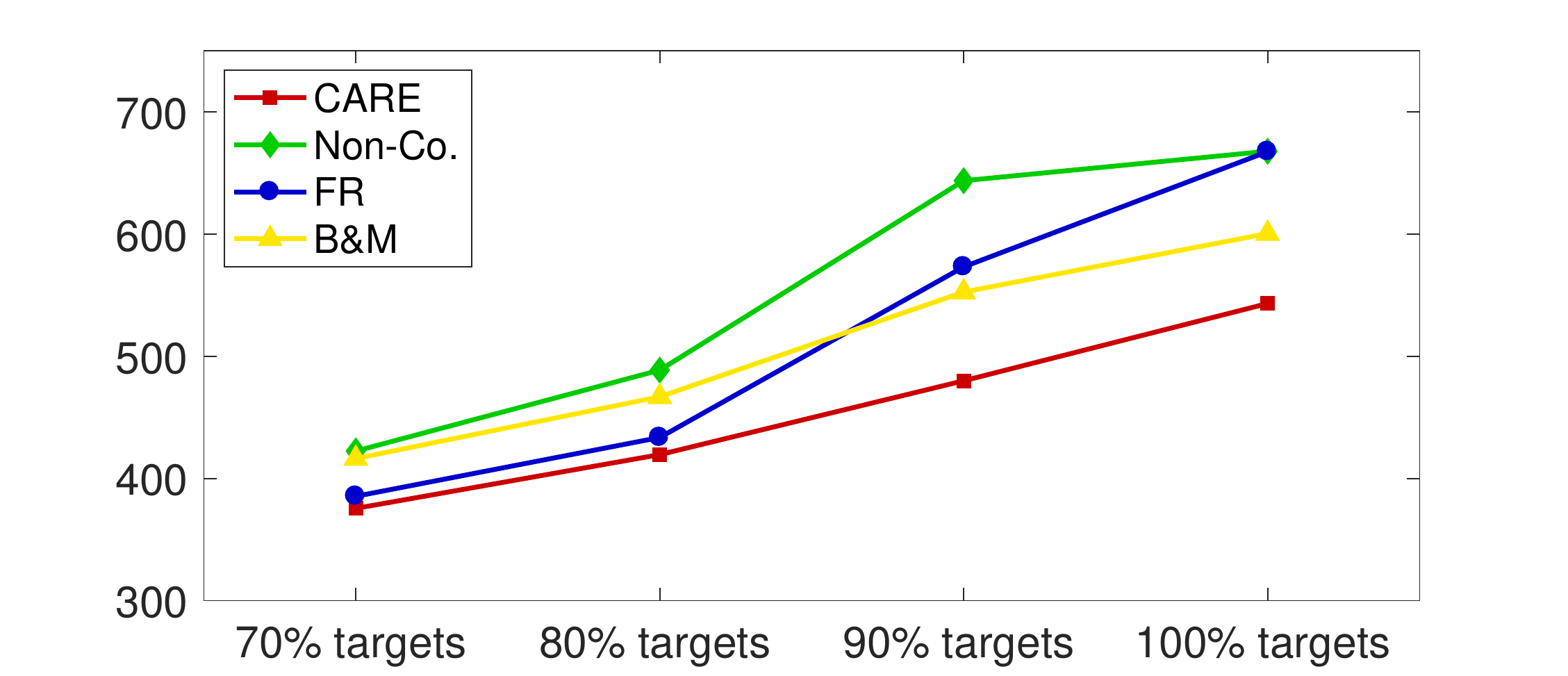}\label{fig:TDT_1}}
     \caption{\textbf{Scenario $1$}: Comparison of coverage performance using different online multi-robot coverage methods}
      \label{fig:metrics_1} \vspace{-8pt}
\end{figure*}

\subsection{\textbf{Scenario 2: Some Robots Fail and Some Idle}}
Fig.~\ref{fig:res2} presents a more complex scenario where two robots failed unexpectedly during operation. A total number of $106$ targets were randomly distributed in the field.

Fig.~\ref{fig:res2}(1) shows that the robots start exploration while using their on-board sensing systems to discover the environment. Fig.~\ref{fig:res2}$(2)$ shows that robot $v_7$ failed unexpectedly and a resilience game $G_1$ was triggered involving $\kappa_2 = 3$ of its closest neighbors. The player set was formed as $\mathscr{P} = \{v_2, v_6, v_8\}$. The action set consisted of the current tasks of all players, as well as the task belonging to the failed robot, i.e., $\tilde{\mathscr{A}} = \{2, 6, 7, 8\}$. As seen in Fig.~\ref{fig:res2}(3), the optimal equilibrium of $G_1$ immediately reallocated $v_8$ to drop its current task and help $v_7$, because task $7$ has a much higher expected worth even at the expense of traveling. Later, as shown in Fig.~\ref{fig:res2}$(4)$, robot $v_4$ failed too, which initiated the second resilience game $G_2$. Similarly, robots $v_3$, $v_5$ and $v_9$ were the closest neighbors, hence they formed the player set $\mathscr{P} = \{v_3,v_5,v_9\}$. The action set was $\tilde{\mathscr{A}} = \{3,4,5,9\}$. As observed in Fig.~\ref{fig:res2}$(5)$, the optimal equilibrium of $G_2$ lead $v_3$ to drop its task $3$ and immediately transition to help $v_4$, in pursuit of a higher worth task. Since tasks $3$ and $8$ were dropped by their initially assigned robots after games $G_1$ and $G_2$, thus they are available to any future no-idling games. Fig.~\ref{fig:res2}$(6)$ shows that robot $v_6$ has just completed its task and triggered the third no-idling game $G_3$. It called the other robot $v_1$ to join $G_3$, which was close to finish task $1$. Thus the player set was $\mathscr{P} = \{v_1,v_6\}$. The action set $\tilde{\mathscr{A}} = \{3, 8\}$ included tasks $3$ and $8$ that were assigned with no robot. No other region had sufficient task left. The optimal equilibrium of $G_3$ reallocated $v_1$ and $v_6$ to task $3$ and task $8$, respectively, as seen in Fig.~\ref{fig:res2}$(7)$. Finally, complete coverage was achieved with all targets found, as shown in Fig.~\ref{fig:res2}(8).

Fig.~\ref{fig:res2_table} presents the details of all games. It is observed that $G_P$ is $37.64\%$ for $G_1$ and $36.98\%$ for $G_2$, thus via event-driven task reallocations, the neighbors of the failed robot can re-organize to compensate for the loss of expected worth due to robot failures. Also, $G_P$ is $91.16\%$ for game $G_3$, hence the idling robots can expect to discover more targets from the remaining tasks after task reallocation. Accordingly, $G_T$ is $15.22\%$, $16.19\%$ and $5.25\%$ for games $G_1$, $G_2$ and $G_3$, respectively. Thus, the whole team also benefits from each task reallocation.

\vspace{-3pt}
\subsection{\textbf{Performance Comparison with Alternative Methods}}
Now, we examine the performance of CARE as compared to three other online multi-robot coverage methods, including: (1) Non-cooperative (Non-Co.) strategy, where each robot covers its own task using the $\epsilon^\star$ algorithm without cooperation upon task completion or robot failures; (2) modified First-responder (FR) strategy~\cite{SGH14}, where robots that finish early would selfishly seek for new tasks that can maximize their own utility using Eq. (\ref{eq:utility}); and (3) Brick and Mortar (B\&M) algorithm~\cite{FTL07}. The performance metrics of the alternative methods have been examined in both scenarios using the same initial conditions. The time measured in performance metrics is in seconds.

\begin{figure*}[ht!]
     \centering
     \vspace{-15pt}
     \subfloat[CARE]{
     \includegraphics[width=0.24\textwidth]{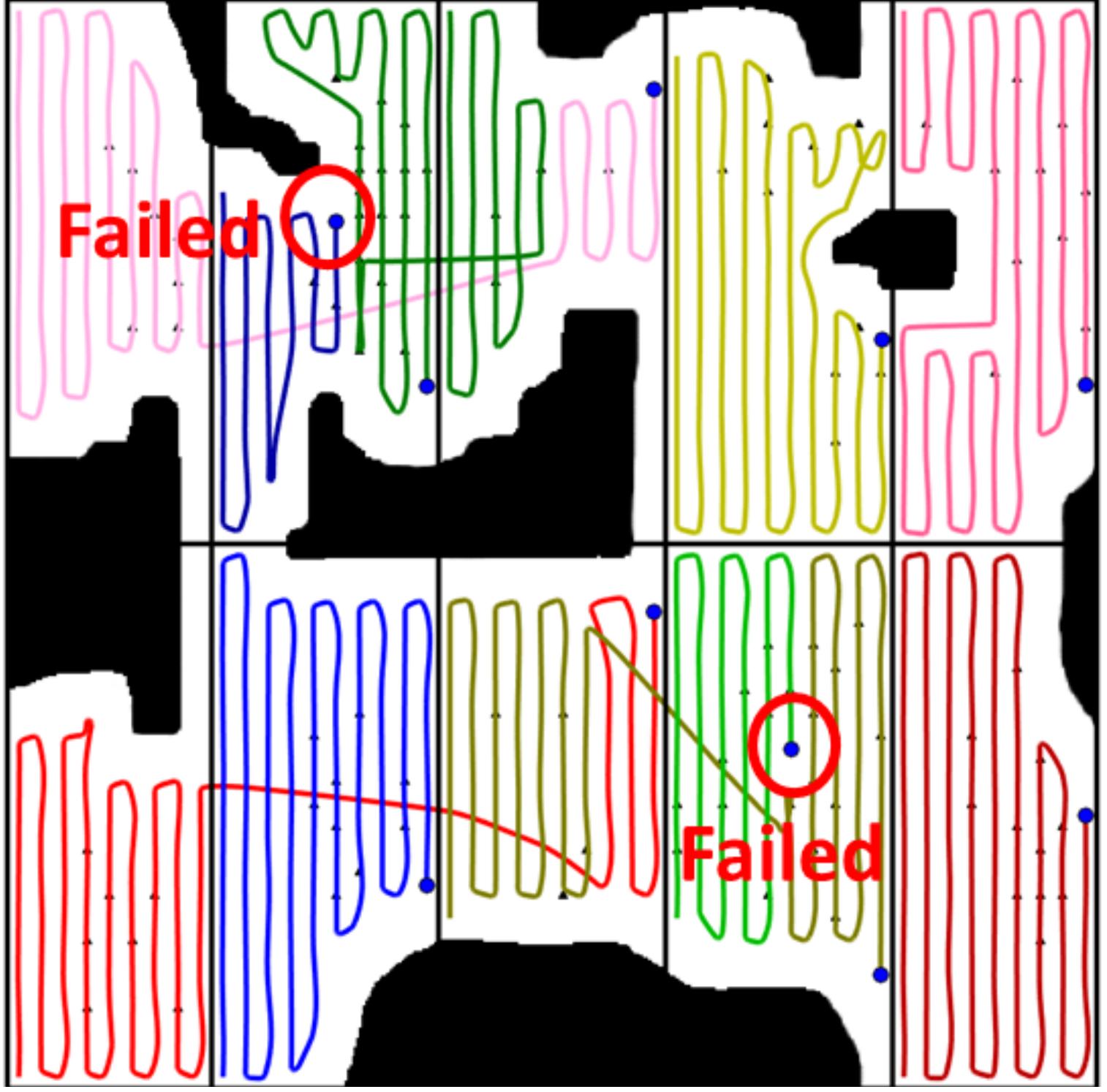}\label{fig:res2_others_care}}
     \subfloat[Non-cooperative Coverage]{
     \includegraphics[width=0.24\textwidth]{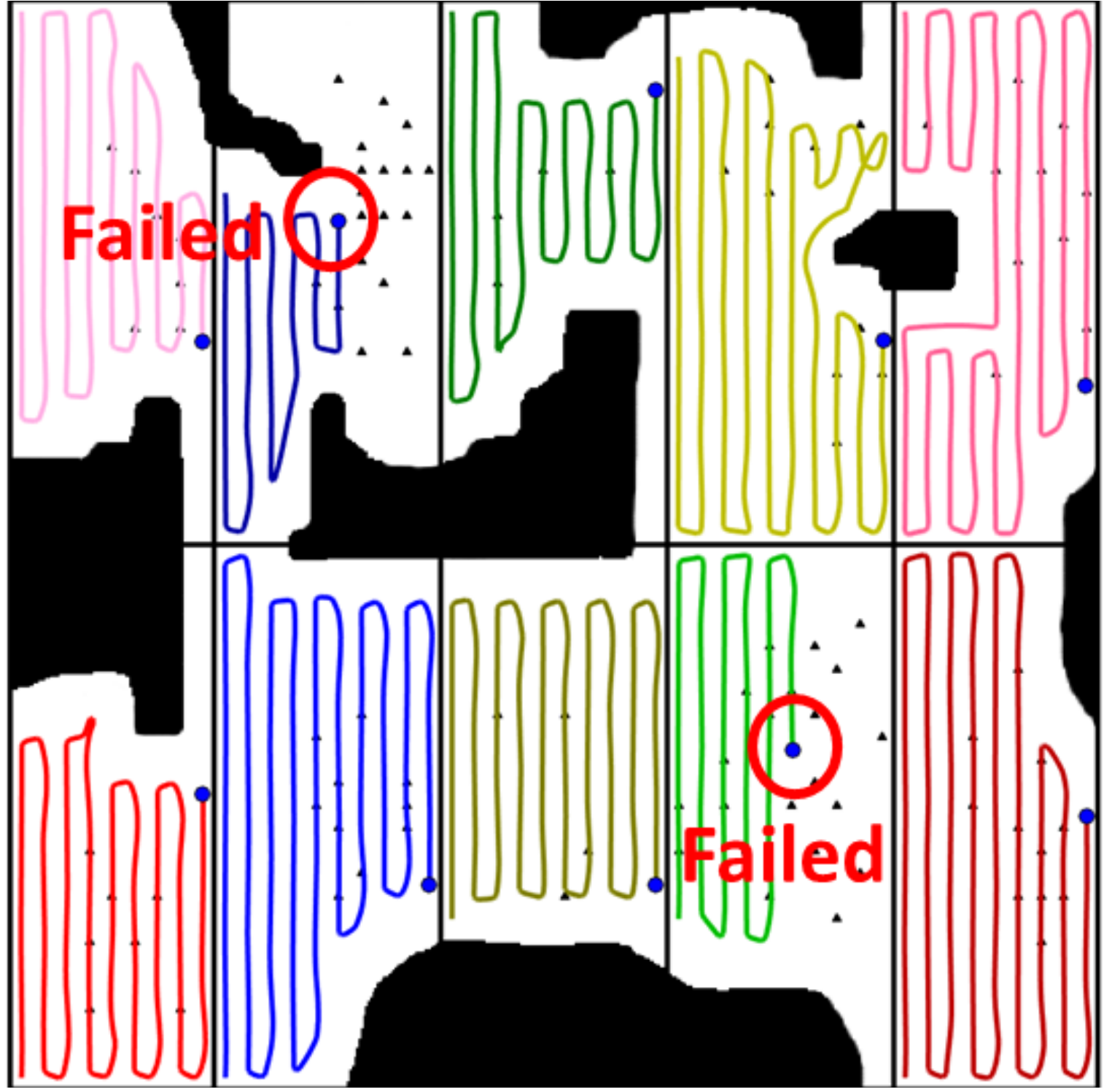}\label{fig:res2_others_nocoop}}
     \subfloat[First-responder Coverage]{
     \includegraphics[width=0.24\textwidth]{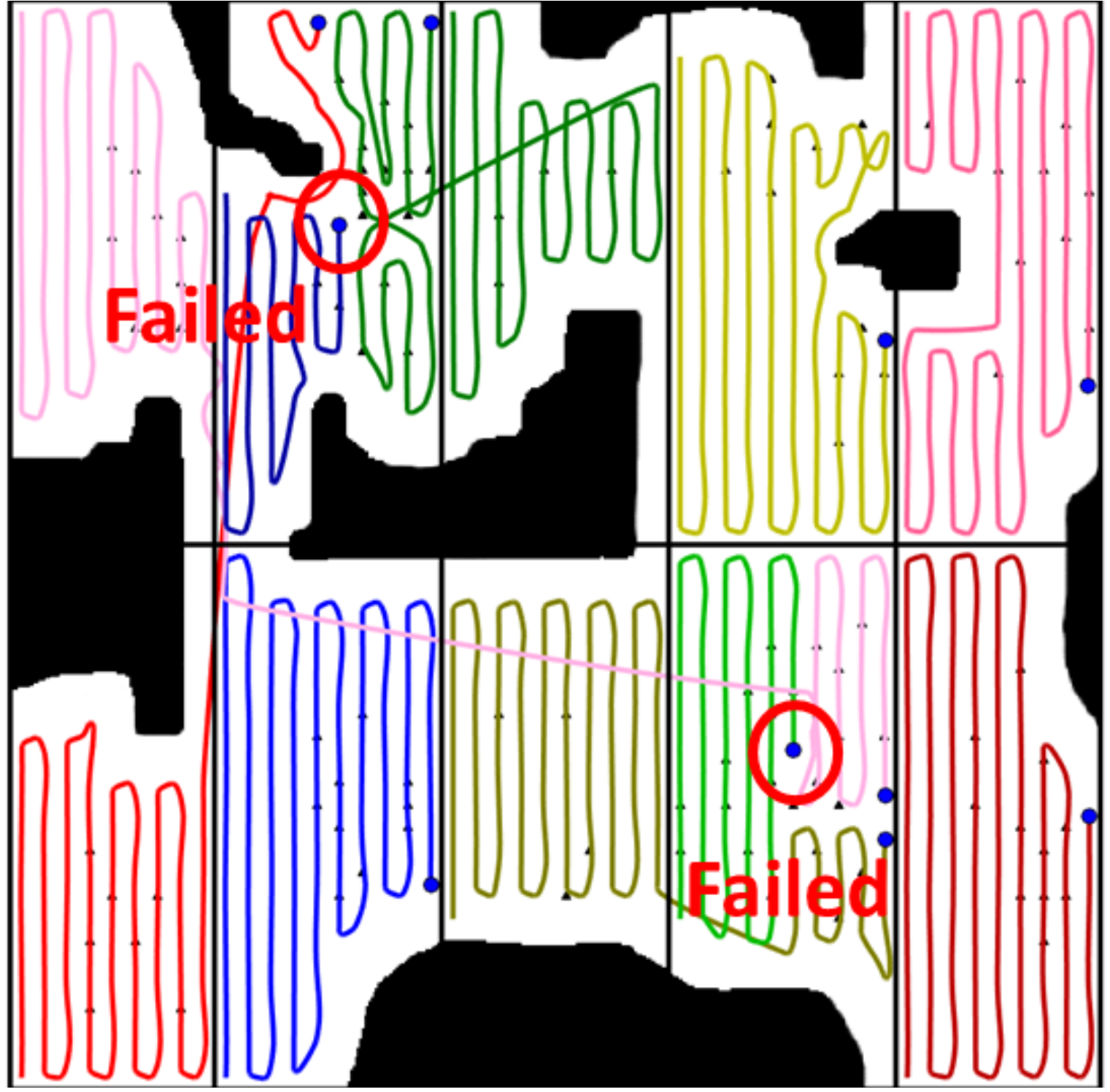}\label{fig:res2_others_first}}
     \subfloat[Brick and Mortar]{
     \includegraphics[width=0.24\textwidth]{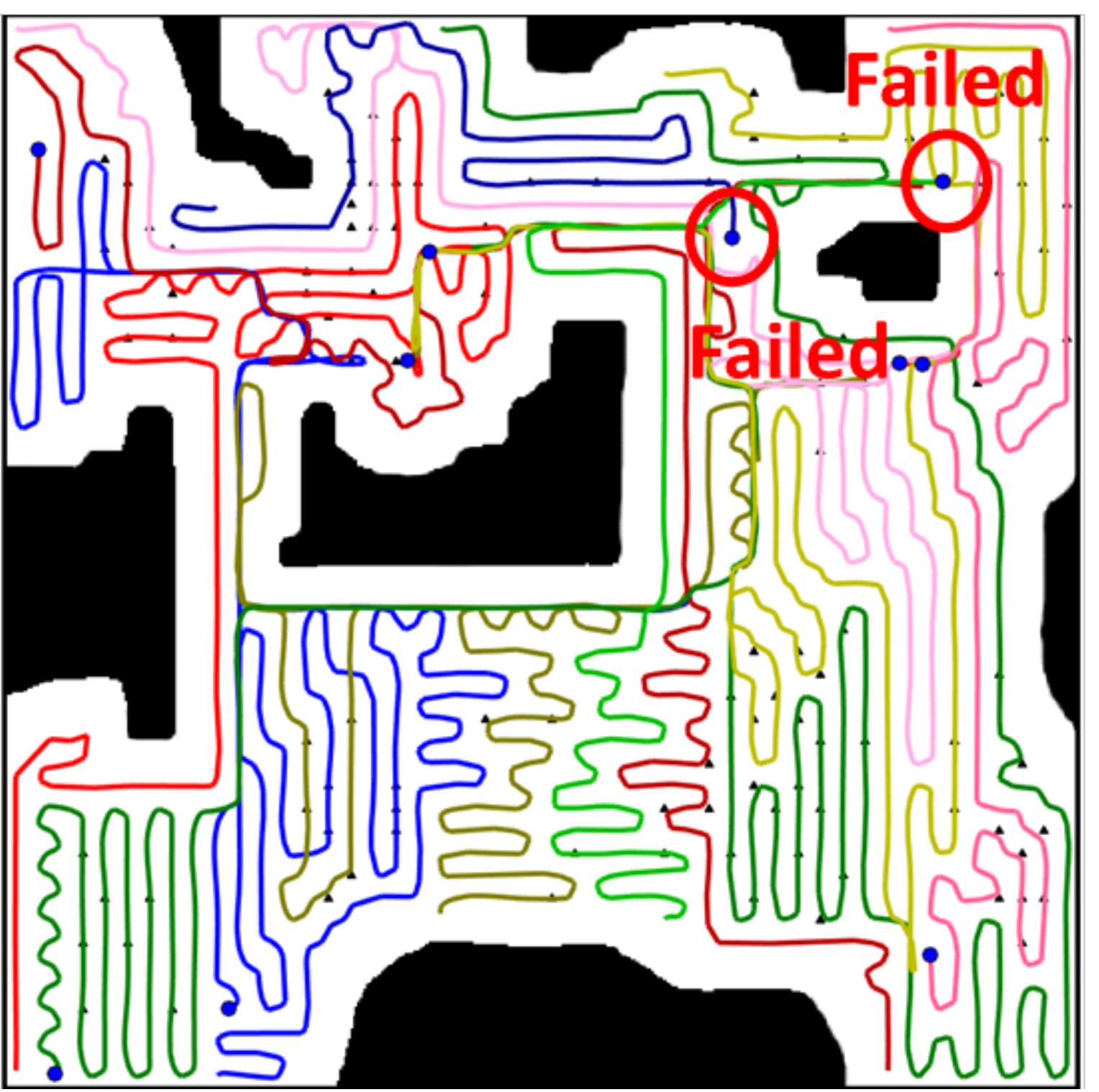}\label{fig:res2_others_bm}}
     \caption{\textbf{Scenario $2$}: Comparison of coverage trajectories using different online multi-robot coverage methods}
      \label{fig:res2_others} \vspace{-8pt}
\end{figure*}

\begin{figure*}[ht!]
     \centering
     \subfloat[Coverage Time ($CT$)]{
     \includegraphics[width=0.28\textwidth]{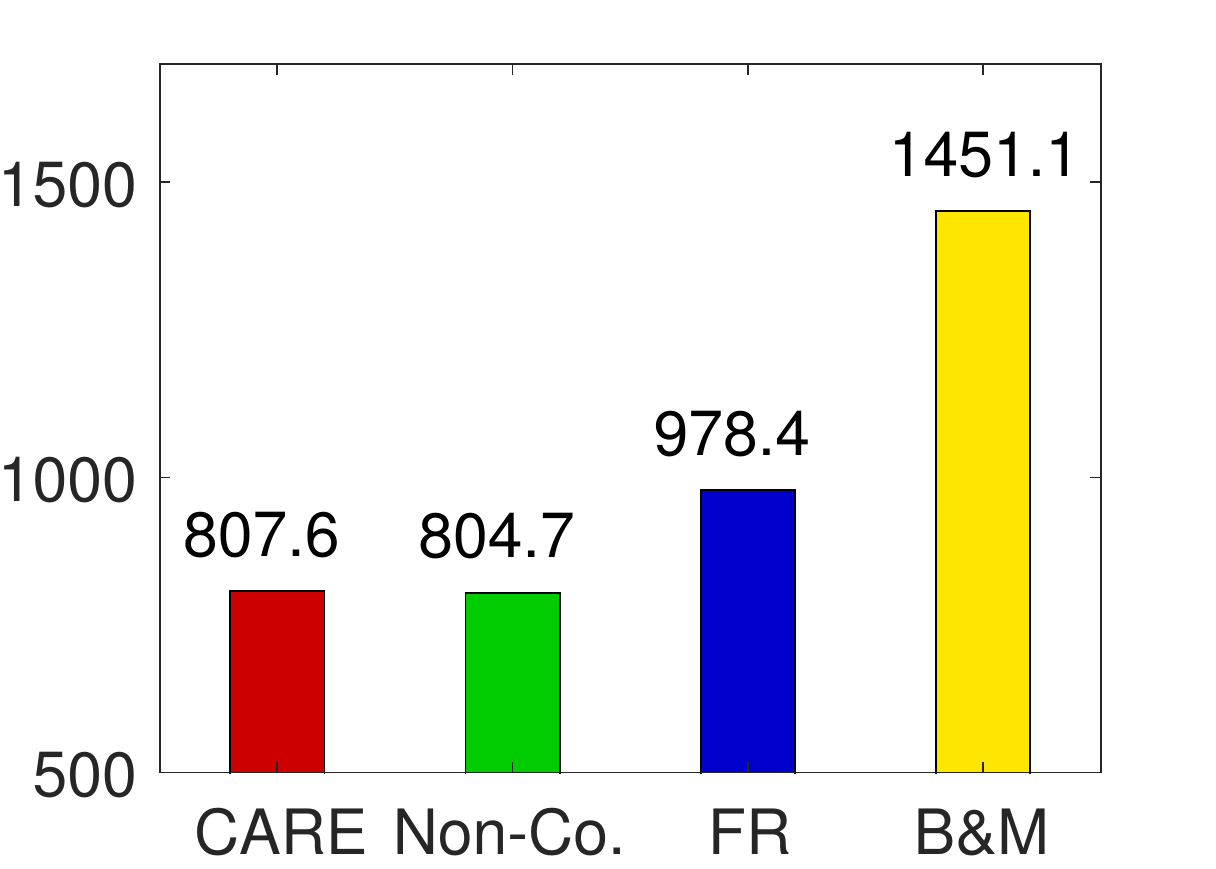}\label{fig:CT_2}}
     \subfloat[Coverage Ratio ($CR$)]{
     \includegraphics[width=0.28\textwidth]{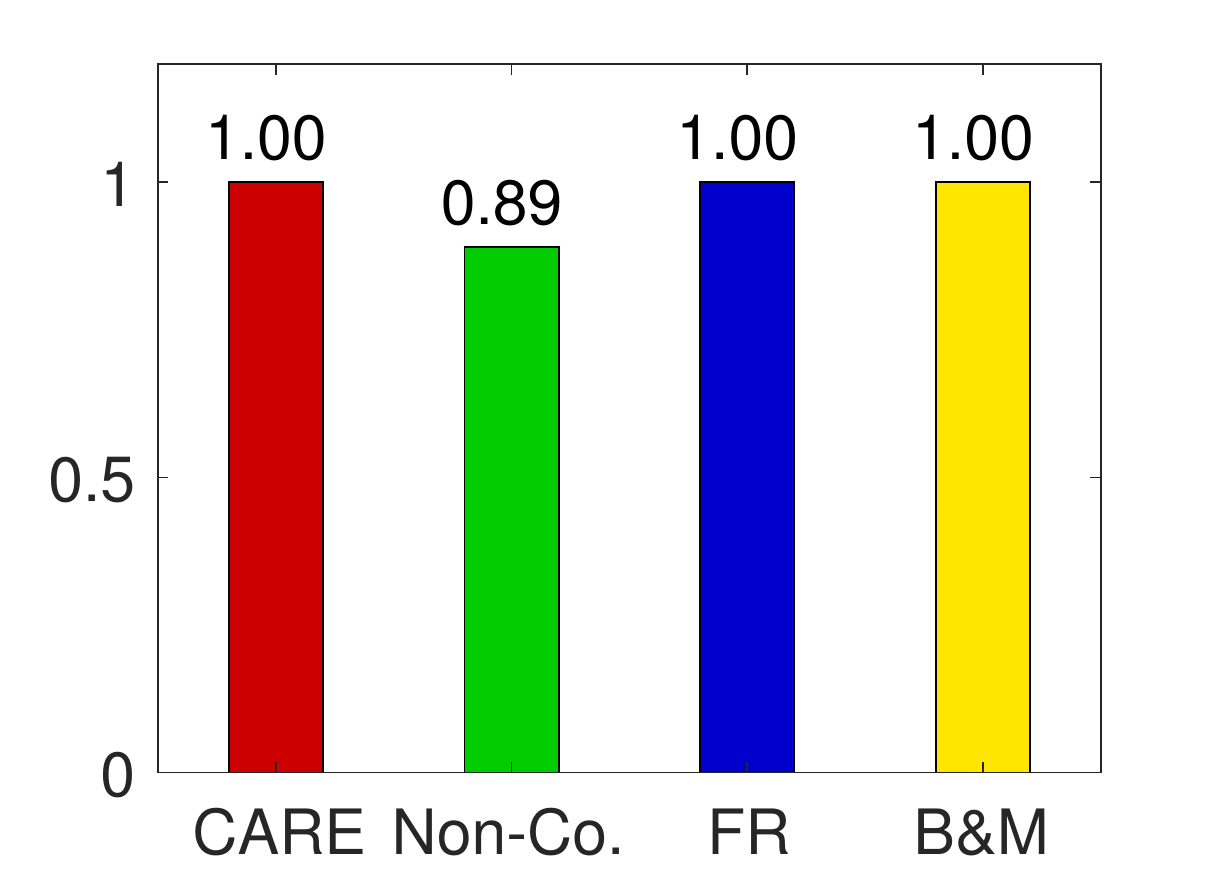}\label{fig:CR_2}}
     \subfloat[Remaining Reliability ($RR$)]{
     \includegraphics[width=0.28\textwidth]{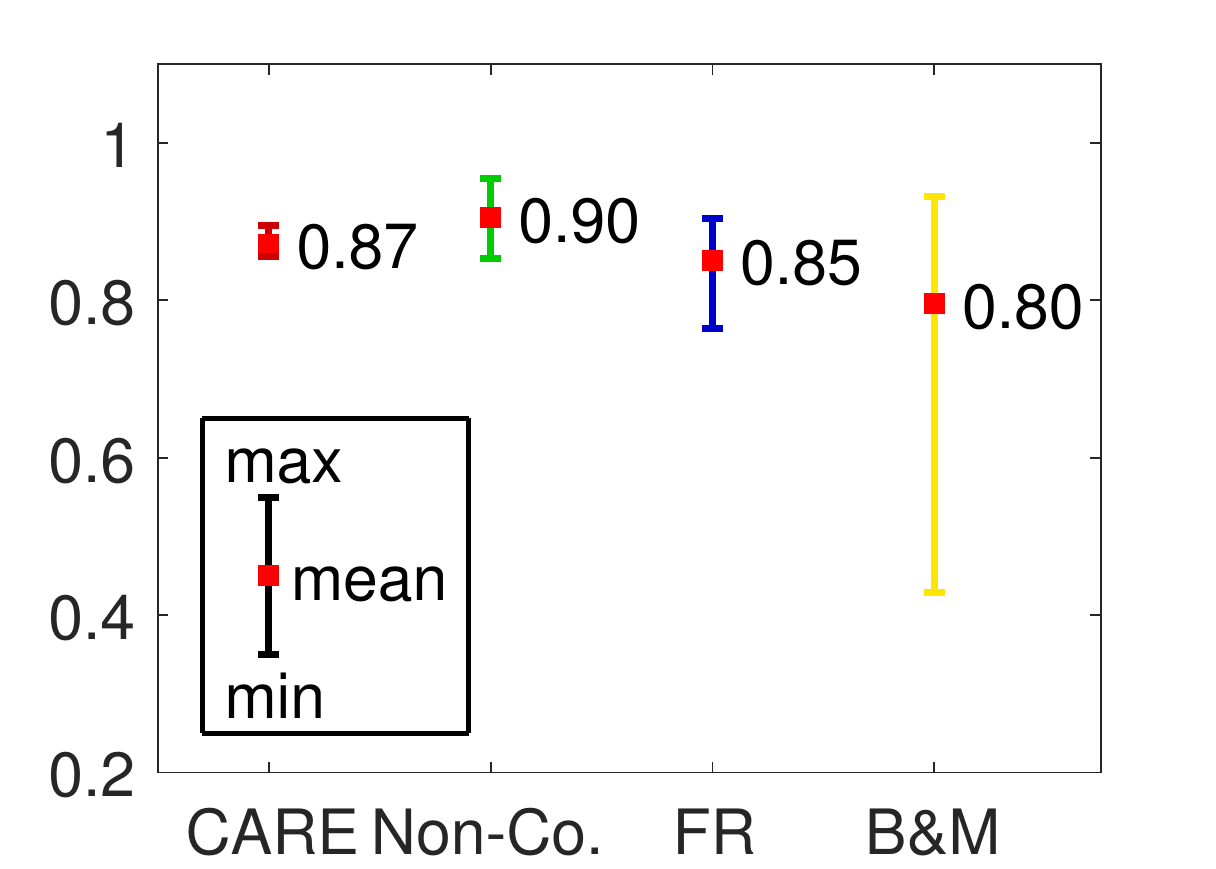}\label{fig:RR_2}}\\
     \subfloat[Number of Targets Found ($NoTF$)]{
     \includegraphics[width=0.28\textwidth]{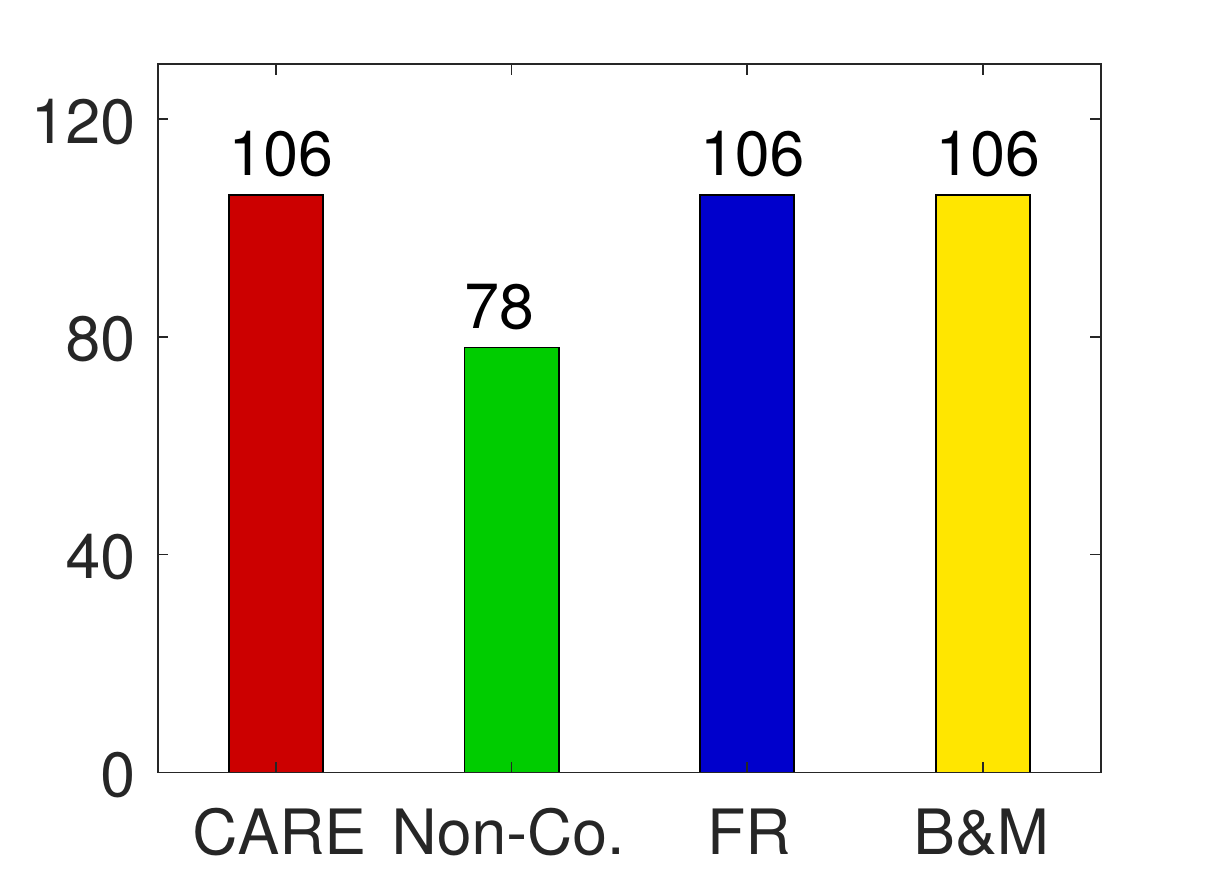}\label{fig:NoFT_2}} \hspace{-15pt}
     \subfloat[Time of Target Discovery ($ToTD$)]{
     \includegraphics[width=0.47\textwidth]{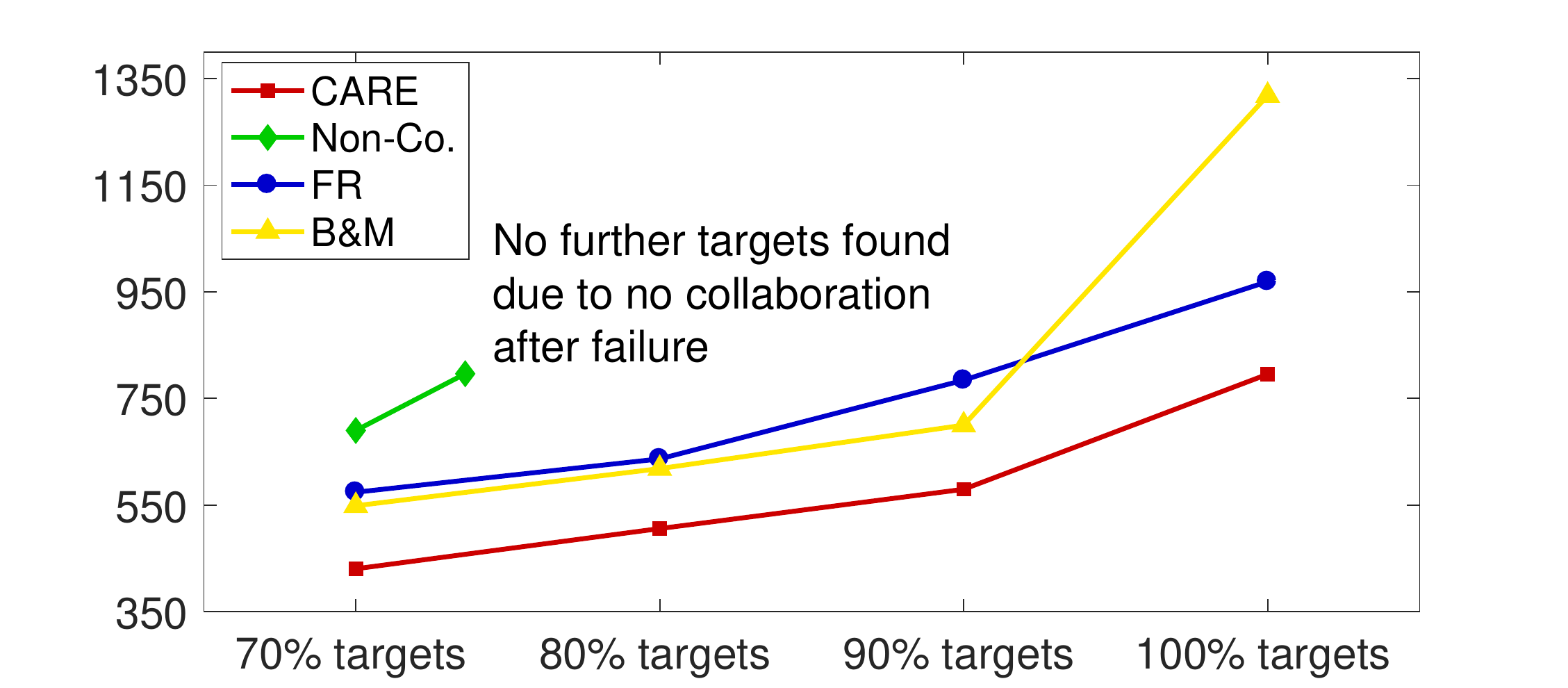}\label{fig:TDT_2}}
     \caption{\textbf{Scenario $2$}: Comparison of coverage performance using different online multi-robot coverage methods}
      \label{fig:metrics_2} \vspace{-8pt}
\end{figure*}

\vspace{3pt}
\subsubsection{\textbf{Scenario 1}}
Figs.~\ref{fig:res1_others_care}$\sim$\ref{fig:res1_others_bm} show the robot trajectories for Scenario $1$ using CARE and the three alternative methods. Since there was no failure in this scenario, complete coverage was achieved using all methods. The corresponding performance metrics are shown in Fig.~\ref{fig:metrics_1}.

As seen in Fig.~\ref{fig:CT_1}, CARE requires the least coverage time ($CT$), which saves $\frac{694.3-557.9}{694.3}\approx 19.65\%$ in time as compared to using the Non-Co method. Similarly, CARE saves about $16.41\%$ and $60.47\%$ in time as compared to using the FR and B\&M methods, respectively. The significant savings in $CT$ are due to the no-idling games that reallocated early completed robots $v_2$, $v_5$, $v_6$, and $v_{10}$ in an optimized way.

Moreover, due to lack of cooperation, the Non-Co. method requires a much higher $CT$ than CARE and the FR method. In the FR strategy, since early completed robots selfishly selects their new tasks that can maximize their own utility, robot $v_6$ ended up picking task $3$ upon finishing task $6$, which contains higher worth even at the expense of long traveling time (see Fig.~\ref{fig:res1_others_first}). In this regard, the FR method presents higher $CT$ than CARE. Further, due to lack of task partitioning as well as the looping problem, the B\&M method generated strongly overlapped trajectories that leads to the highest $CT$.

Fig.~\ref{fig:RR_1} shows the minimum, mean and maximum remaining reliability ($RR$) among all live robots as the operation ended. It is seen that, although CARE shares the same mean $RR$ with the Non-Co. and FR methods, it has the smallest difference between minimum and maximum $RR$ of all robots, which implies a more balanced battery depletion for different robots. In contrast, the B\&M method presents the smallest mean $RR$ due to the highest $CT$.

As for the number of targets found ($NoTF$), since $CR=1$ in this scenario, all $107$ hidden targets were found using all methods, as shown in Fig.~\ref{fig:CR_1}$\sim$\ref{fig:NoFT_1}.

Fig.~\ref{fig:TDT_1} shows the time of target discovery ($ToTD$). It is seen that at each percentage of targets found, CARE always requires the least time, thus leading to the fastest target discovery progress as compared to other methods. This is due to the optimized task reallocations of early completed robots $v_2$, $v_5$, $v_6$ and $v_{10}$ after playing no-idling games $G_1$ and $G_2$. Note that the time in $ToTD$ when $100\%$ targets are discovered, is different from $CT$, because robots should continue searching in unexplored regions towards complete coverage.

\vspace{3pt}
\subsubsection{\textbf{Scenario 2}}\label{sec:scenario2}
Figs.~\ref{fig:res2_others_care}$\sim$\ref{fig:res2_others_bm} show the robot trajectories using CARE and the three alternative methods for Scenario $2$. The corresponding performance metrics are presented in Fig.~\ref{fig:metrics_2}. In this scenario, two robots of $v_4$ and $v_7$ failed during operation. The alternative methods were evaluated using the same failing condition, i.e., the same robots failed after traveling for the same amount of time.

As shown in Fig.~\ref{fig:CT_2}, CARE saves about  $17.46\%$ and $44.35\%$ in $CT$ as compared to using the FR and B\&M methods, respectively. This is due to the no-idling game $G_3$ that reallocated early completed robots $v_1$ and $v_6$ in an optimized manner; while in the FR method, the initially assigned tasks of the failed robots were left unattended until some other robot completes its task. Again, the B\&M method has the highest $CT$ due to strongly overlapped trajectories.

Fig.~\ref{fig:RR_2} shows the $RR$ of live robots at the end of the team operation. It is seen that, CARE has a higher mean $RR$, as well as the smallest difference between minimum and maximum $RR$ of all live robots, as compared to the FR and B\&M methods. Also, since tasks $4$ and $7$ were left unattended after robots $v_4$ and $v_7$ failed, the Non-Co. method has the highest mean $RR$.

As shown in Fig.~\ref{fig:CR_2}, coverage was incomplete ($CR=0.89$) using the Non-Co. method, while all other methods achieved $CR=1$. Accordingly, a total number of $28$ hidden targets were missed using the Non-Co. method, while all other methods successfully discovered all $106$ targets, as shown in Fig.~\ref{fig:NoFT_2}.

Fig.~\ref{fig:TDT_2} shows the performance of $ToTD$, where CARE again shows the fastest target discovery progress as compared to other methods. This is mainly because after playing the resilience games $G_1$ and $G_2$, robots $v_3$ and $v_6$ immediately determined to drop their current tasks and search tasks $4$ and $7$ when robots $v_4$ and $v_7$ failed, respectively, in pursuit of a much higher expected worth than their current tasks. Also, the Non-Co. method only collected around $73.6\%$ of all targets at the end of the team operation due to incomplete coverage.

\subsection{\textbf{Effects of Parameters on Coverage Performance}}\label{sec:parameter_effect}
This section evaluates the effects of different parameters on the coverage performance. Specifically, we vary the values of $N$, $\lambda_r$, $\kappa_1$ and $\kappa_2$, while keeping the values of all other parameters constant. The performance metrics presented in Section~\ref{Sec:performancemetrics} are used for evaluations.

{\setlength{\belowcaptionskip}{-6pt}
\begin{figure*}[!ht]
     \centering
     \subfloat[Coverage trajectories using a team of $4$ robots ($v_2, v_3, v_4$ and $v_5$). Robot $v_4$ was failed during exploration]{
     \includegraphics[width=.98\textwidth]{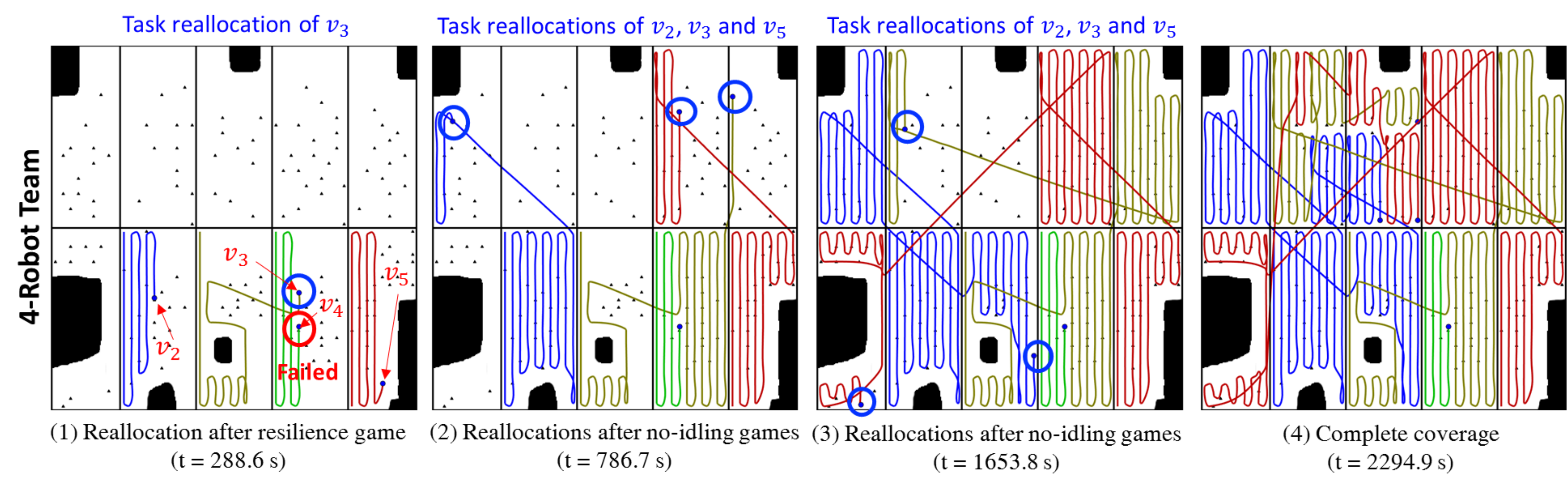}\label{fig:res3_4robots}} \\ \vspace{-10pt}
     \subfloat[Coverage trajectories using a team of $6$ robots ($v_1, v_2, v_3, v_4$, $v_5$ and $v_9$). Robot $v_4$ was failed during exploration]{
     \includegraphics[width=.98\textwidth]{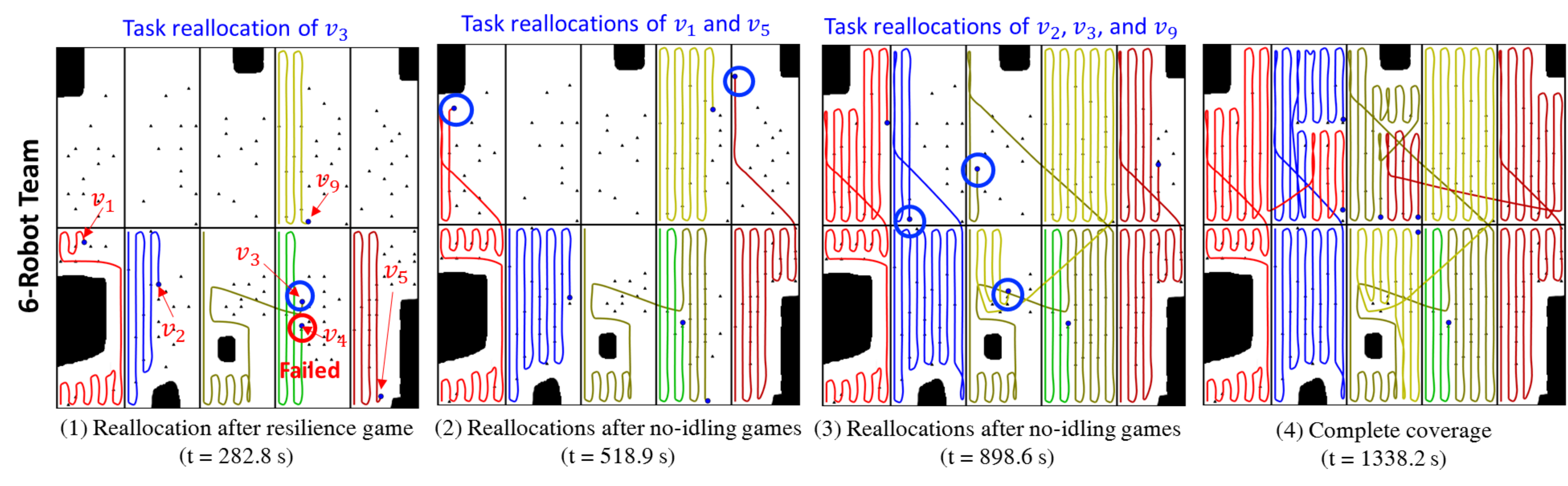}\label{fig:res3_6robots}} \\ \vspace{-10pt}
     \subfloat[Coverage trajectories using a team of $8$ robots ($v_1, v_2, v_3, v_4$, $v_5$, $v_7$, $v_8$ and $v_9$). Robot $v_4$ was failed during exploration]{
     \includegraphics[width=.98\textwidth]{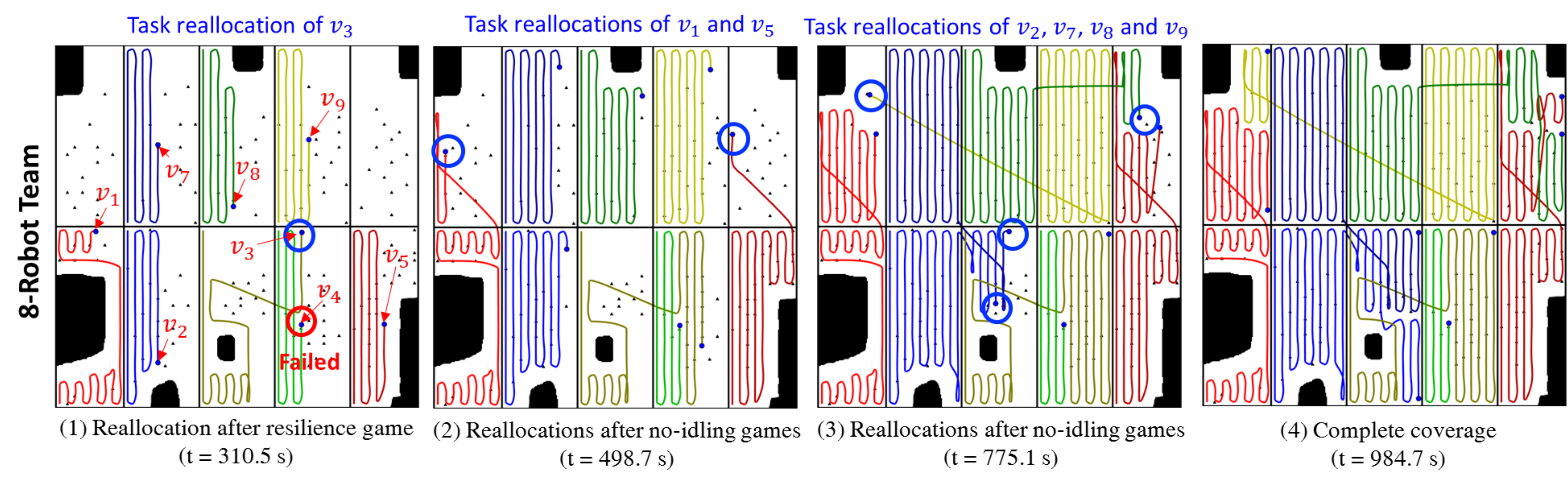}\label{fig:res3_8robots}} \\ \vspace{-10pt}
     \subfloat[Coverage trajectories using a team of $10$ robots ($v_1, v_2, v_3, v_4$, $v_5$, $v_6$, $v_7$, $v_8$, $v_9$ and $v_{10}$). Robot $v_4$ was failed during exploration]{
     \includegraphics[width=.98\textwidth]{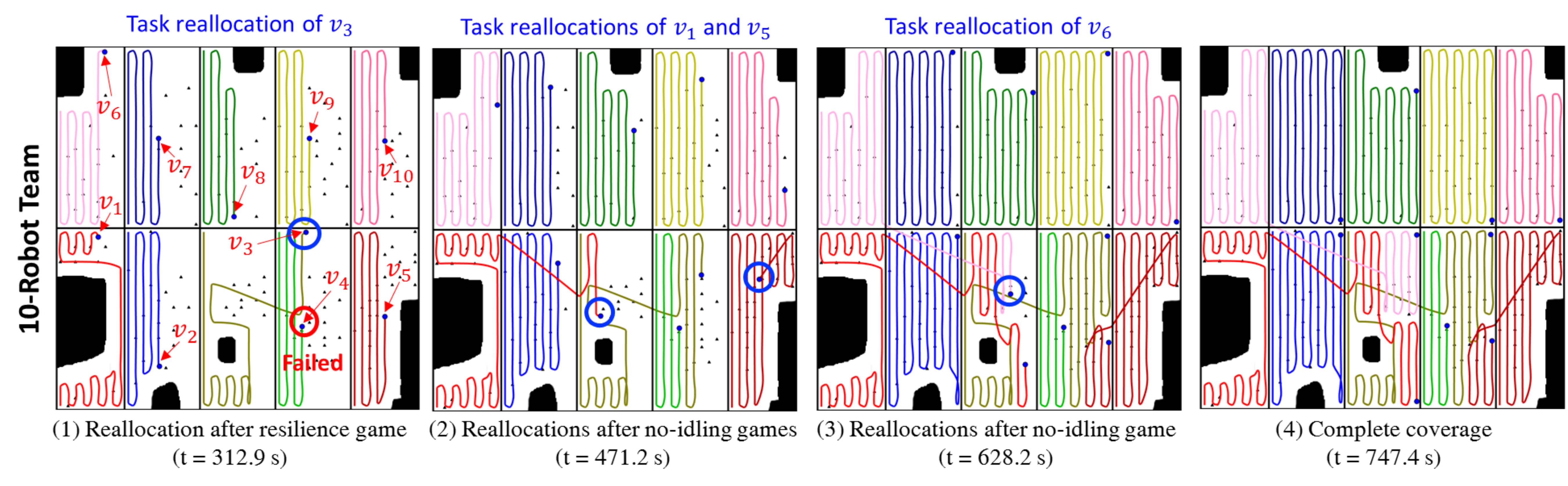}\label{fig:res3_10robots}}
     \caption{\textbf{Scenario $3$}: Coverage trajectories of CARE using a team of $4$, $6$, $8$ and $8$ robots}
      \label{fig:res3_robotnumbers} \vspace{-8pt}
\end{figure*}

\vspace{6pt}
\subsubsection{Team Size ($N$)}\label{sec:team_size}
We examine the effectiveness of CARE when different number of robots are deployed to search the same area $\mathcal{R}$. For this purpose, we present Scenario 3, where teams of $N = 4$, $6$, $8$ and $10$ robots were deployed to cover the same $10$ tasks, as shown in Fig.~\ref{fig:res3_robotnumbers}. Fig.~\ref{fig:res3_4robots}, Fig.~\ref{fig:res3_6robots}, Fig.~\ref{fig:res3_8robots} and Fig.~\ref{fig:res3_10robots} present the coverage trajectories at different time instants for $N = 4$, $6$, $8$, and $10$, respectively. The scenario setup is kept the same across all simulations, where robot $v_4$ fails after it travels for the same amount of time. As seen in Fig.~\ref{fig:res3_4robots}(1), Fig.~\ref{fig:res3_6robots}(1), Fig.~\ref{fig:res3_8robots}(1) and Fig.~\ref{fig:res3_10robots}(1), a resilience game was initiated upon failure of $v_4$, and its neighbor $v_3$ immediately dropped its task $3$ to help $v_4$, due to a higher expected worth. Later several no-idling reallocations occurred and in all cases complete coverage was achieved. Moreover, it is seen that with a smaller $N$, task reallocation appears more often. As shown in Fig.~\ref{fig:res3_4robots}(4), Fig.~\ref{fig:res3_6robots}(4), Fig.~\ref{fig:res3_8robots}(4) and Fig.~\ref{fig:res3_10robots}(4), the total coverage time clearly decreases when $N$ increases.

Table~\ref{table:different_team_size} presents the corresponding coverage performances. It is seen that with smaller $N$, since each robot must cover more tasks, the average $RR$ of all live robots become smaller. Also, the minimum and maximum $RR$ are close to the mean, which implies the live robots have been operating for similar amounts of time and robot idling was successfully prevented.

{\setlength{\belowcaptionskip}{-5pt}
\begin{figure*}[!ht]
     \centering
     \subfloat[\textbf{Target distribution example 1}: sparse number of targets in the task of failed robot $v_4$ in a $8$-robot team. Tasking sequence of each robot: $v_1: \mathcal{R}_1 \rightarrow \mathcal{R}_6, v_2: \mathcal{R}_2 \rightarrow \mathcal{R}_4, v_3: \mathcal{R}_3 \rightarrow \mathcal{R}_4, v_4: \mathcal{R}_4, v_5: \mathcal{R}_5 \rightarrow \mathcal{R}_{10}, v_7: \mathcal{R}_7 \rightarrow \mathcal{R}_6, v_8: \mathcal{R}_8 \rightarrow \mathcal{R}_{10}, v_9: \mathcal{R}_9 \rightarrow \mathcal{R}_{10}$]{
     \includegraphics[width=.95\textwidth]{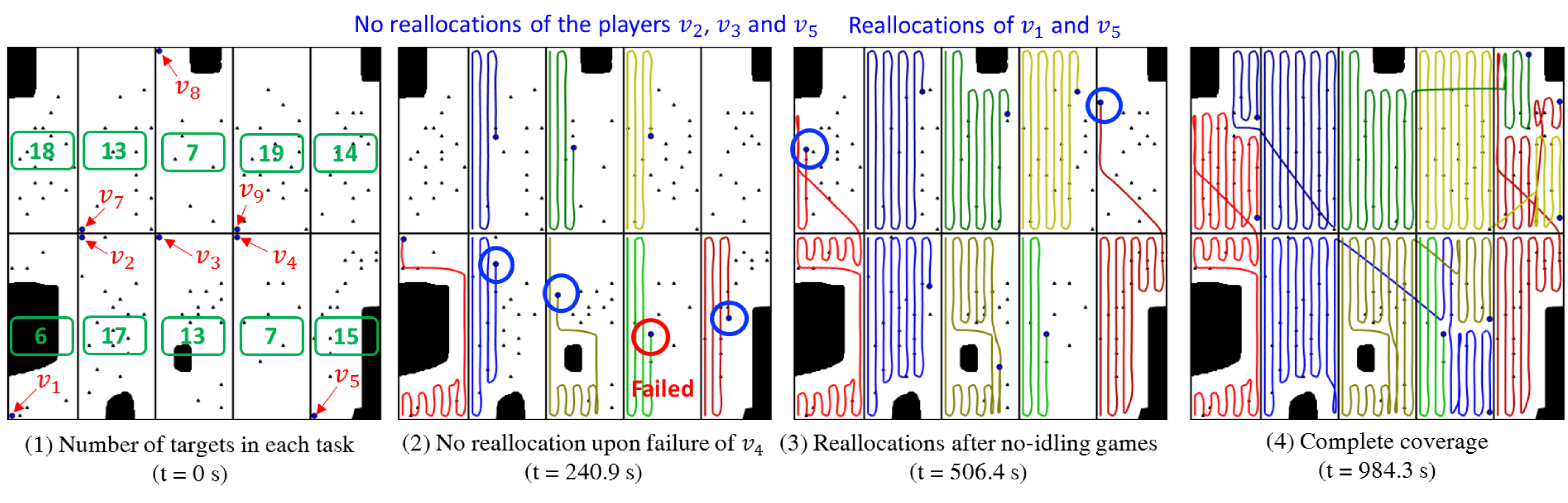}\label{fig:res3_8robots_case1}} \\ \vspace{-10pt}
     \subfloat[\textbf{Target distribution example 2}: medium number of targets in the task of failed robot $v_4$ in a $8$-robot team. Tasking sequence of each robot: $v_1: \mathcal{R}_1 \rightarrow \mathcal{R}_6, v_2: \mathcal{R}_2 \rightarrow \mathcal{R}_3, v_3: \mathcal{R}_3 \rightarrow \mathcal{R}_4, v_4: \mathcal{R}_4, v_5: \mathcal{R}_5 \rightarrow \mathcal{R}_{10}, v_7: \mathcal{R}_7 \rightarrow \mathcal{R}_3, v_8: \mathcal{R}_8 \rightarrow \mathcal{R}_{10}, v_9: \mathcal{R}_9 \rightarrow \mathcal{R}_{10}$]{
     \includegraphics[width=.95\textwidth]{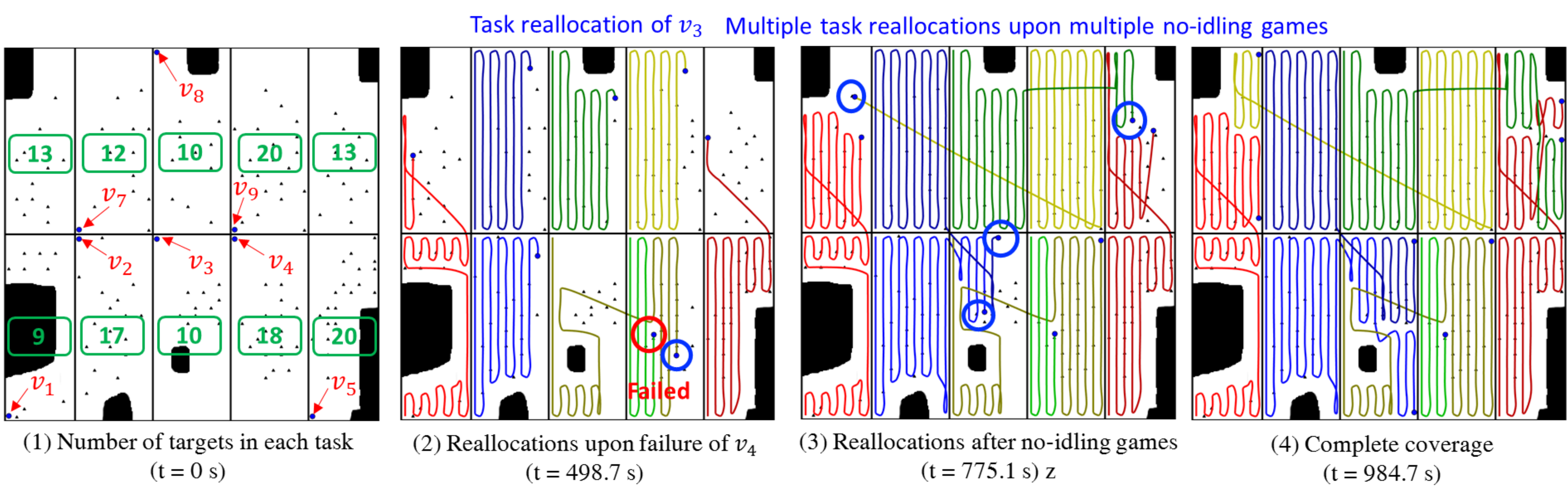}\label{fig:res3_8robots_case2}} \\ \vspace{-10pt}
     \subfloat[\textbf{Target distribution example 3}: dense number of targets in the task of failed robot $v_4$ in a $8$-robot team. Tasking sequence of each robot: $v_1: \mathcal{R}_1 \rightarrow \mathcal{R}_6, v_2: \mathcal{R}_2 \rightarrow \mathcal{R}_3, v_3: \mathcal{R}_3 \rightarrow \mathcal{R}_4 \rightarrow \mathcal{R}_{10}, v_4: \mathcal{R}_4, v_5: \mathcal{R}_5 \rightarrow \mathcal{R}_{4} \rightarrow \mathcal{R}_5 \rightarrow \mathcal{R}_{10}, v_7: \mathcal{R}_7 \rightarrow \mathcal{R}_6, v_8: \mathcal{R}_8 \rightarrow \mathcal{R}_{10}, v_9: \mathcal{R}_9 \rightarrow \mathcal{R}_3$]{
     \includegraphics[width=.95\textwidth]{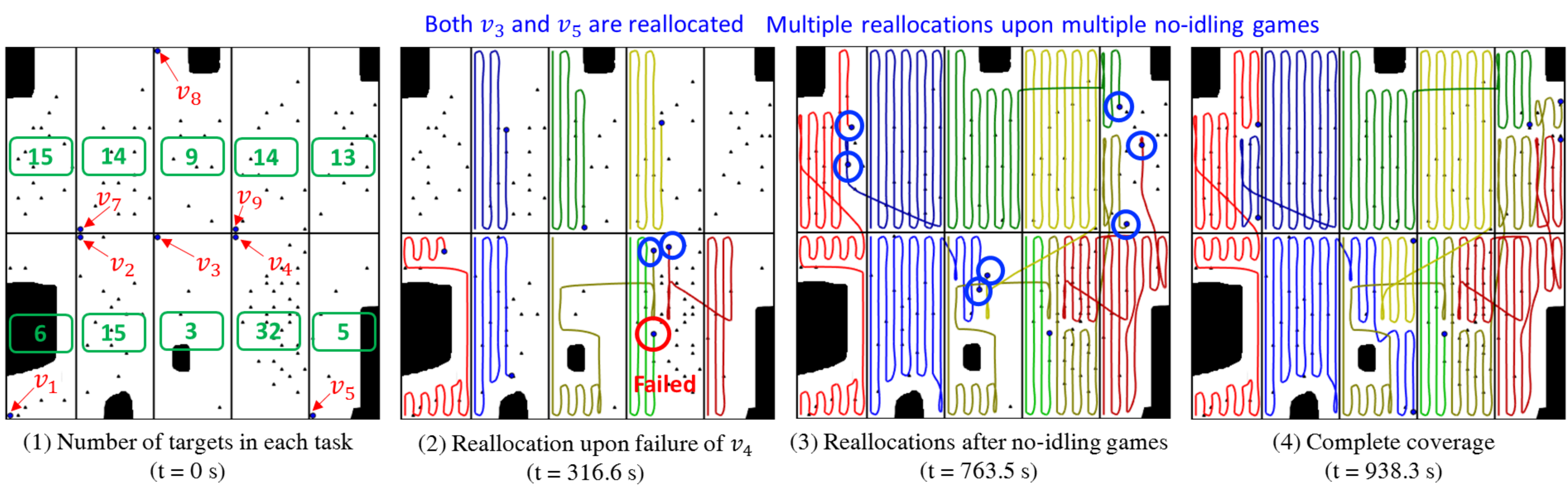}\label{fig:res3_8robots_case3}} \\ \vspace{-3pt}
     \caption{Coverage trajectories under different target distributions using a team of $8$ robots}
      \label{fig:res3_lambda} 
\end{figure*}
}

\vspace{3pt}
\subsubsection{Distribution of Targets ($\lambda_r$)}\label{sec:target_distribution}
Here, we utilize a team of $8$ robots to examine the performance of CARE under different target distributions. Fig.~\ref{fig:res3_lambda} shows the coverage trajectories at different time instants under three different target distribution examples. The number of targets in each task $r$ is labeled in Fig.~\ref{fig:res3_8robots_case1}(1), Fig.~\ref{fig:res3_8robots_case2}(1), and Fig.~\ref{fig:res3_8robots_case3}(1), and $\lambda_r$ is set as the actual number in each task. While the target distributions are randomly generated, in particular, the task of the failed robot $v_4$ has significantly different $\lambda_r$ in the three examples.

In target distribution example $1$, as seen in Fig.~\ref{fig:res3_8robots_case1}, task $4$ has sparse targets. When $v_4$ failed, its neighbors $v_2$, $v_3$ and $v_5$ played a resilience game, but none of them was reallocated to help $v_4$. This is because they can expect higher utility from their current tasks at the moment. Later, as shown in Fig.~\ref{fig:res3_8robots_case1}(3), when $v_1$ and $v_5$ finished, they were reallocated to tasks $6$ and $10$ after playing no-idling games, respectively. At that moment, due to much higher estimated worths in tasks $6$ and $10$, again none of them was reallocated to task $4$. At last, as seen in Fig.~\ref{fig:res3_8robots_case1}(4), when $v_2$ and $v_3$ finished, they moved to task $4$ and eventually complete coverage was achieved and all targets were found.

{\setlength{\belowcaptionskip}{0pt}
\begin{table}[!t]
\centering
\caption{Effects of varying team size}
\label{table:different_team_size}
\begin{tabular}{ccccc}
\toprule
\multirow{2}{*}{$N$} & \multirow{2}{*}{$CT$} & \multicolumn{3}{c}{$RR$} \\ \cline{3-5} 
 &  & min & mean & max \\ \midrule \vspace{5pt}
$4$ & $2258.6s$ & $0.071$ & $0.086$ & $0.109$ \\\vspace{5pt}
$6$ & $1338.2s$ & $0.546$ & $0.610$ & $0.672$ \\\vspace{5pt}
$8$ & $984.71s$ & $0.777$ & $0.826$ & $0.883$ \\\vspace{5pt}
$10$ & $747.4s$ & $0.876$ & $0.904$ & $0.926$ \\ \bottomrule
\end{tabular}
\end{table}
}

Fig.~\ref{fig:res3_8robots_case2} shows the coverage trajectories under target distribution example $2$. As compared to the previous example, now tasks $4$ and $5$ have slightly more targets, but less targets are present in tasks $3$. Thus, upon failure of $v_4$, $v_3$ was reallocated to task $4$ to pursue a higher utility, as shown in Fig.~\ref{fig:res3_8robots_case2}(2). Later, multiple no-idling games were played and the idling robots $v_2$, $v_7$, $v_8$ and $v_9$ were reallocated, as shown in Fig.~\ref{fig:res3_8robots_case2}(3). Finally, complete coverage was achieved as shown in Fig.~\ref{fig:res3_8robots_case2}(4).

In target distribution example $3$, as shown in Fig.~\ref{fig:res3_8robots_case3}(1), task $4$ has significantly more targets, which makes it prioritized for coverage. In contrast, tasks $3$ and $5$ have much less targets. Hence, as seen in Fig.~\ref{fig:res3_8robots_case3}(2), when $v_4$ failed, a resilience game was initiated involving players $v_2$, $v_3$ and $v_5$; then both $v_3$ and $v_5$ moved to task $4$. Upon reaching task $4$, each of them picked a sub-area and worked in parallel. Thereafter, as shown in Fig.~\ref{fig:res3_8robots_case3}(3), multiple no-idling games appeared and all the live robots were reallocated all around to fill the incomplete tasks. At the end, as shown in Fig.~\ref{fig:res3_8robots_case3}(4), complete coverage was achieved with all targets found.

\begin{figure*}[!ht]
     \flushleft
     \subfloat[$ToTD$ using different $\kappa_1$ in Scenario 1]{
     \includegraphics[width=.46\textwidth]{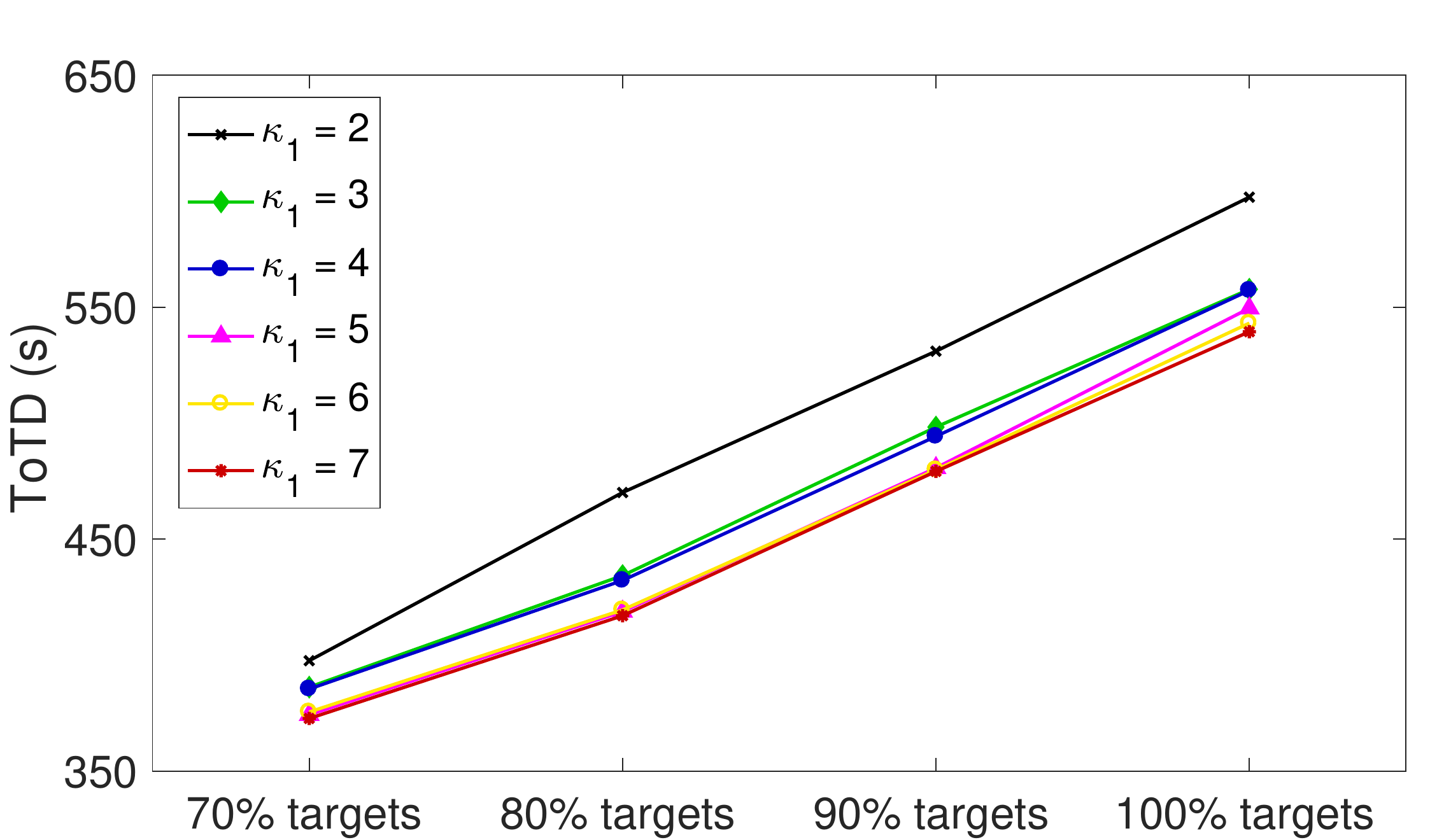}\label{fig:vary_kappa_1}}
     \hspace{1.5em}
     \subfloat[$ToTD$ using different $\kappa_2$ in Scenario 3]{
     \includegraphics[width=.46\textwidth]{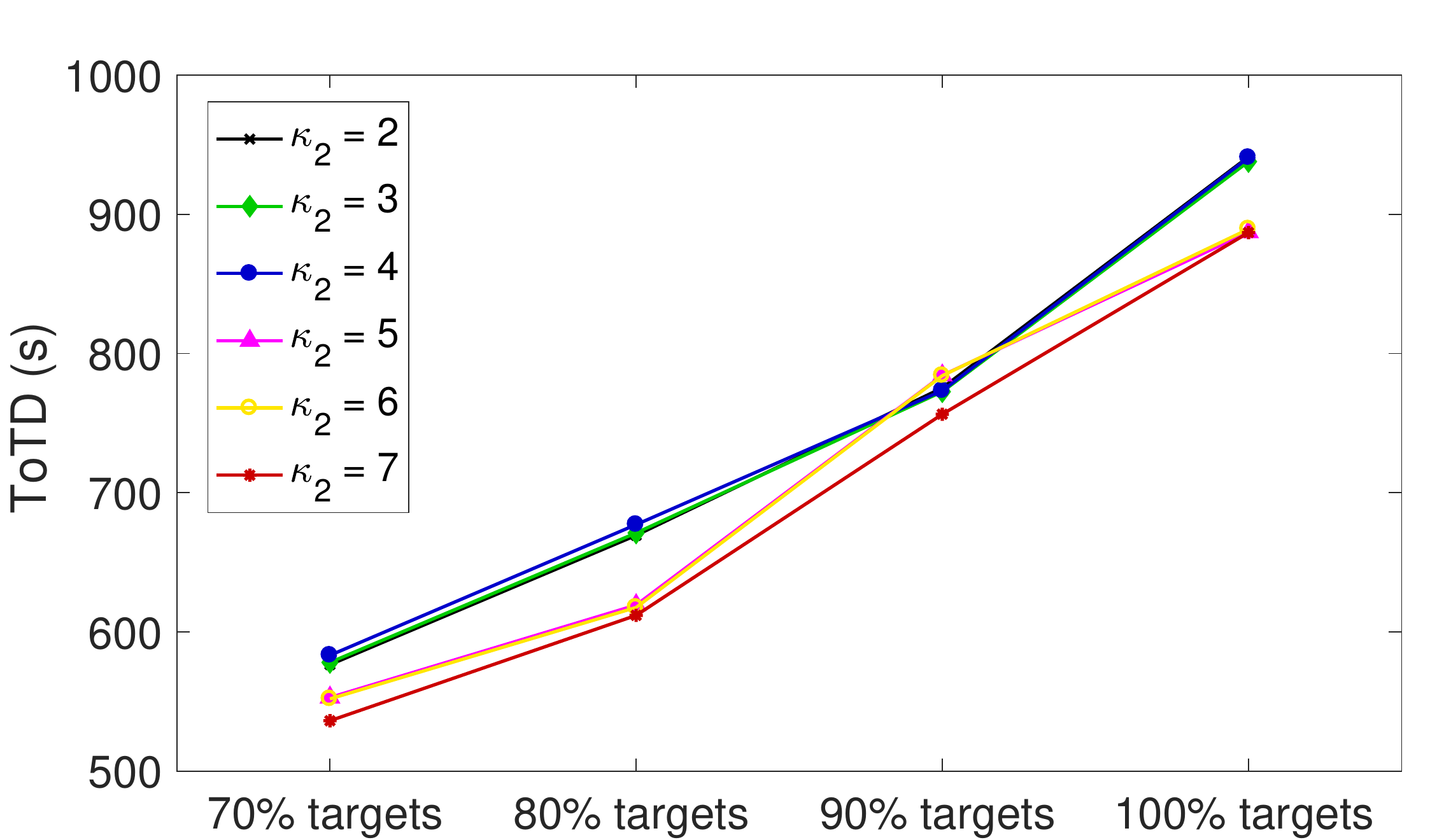}\label{fig:vary_kappa_2}} \\
     \caption{Time of target discovery ($ToTD$) using different neighborhood sizes $\kappa_1$ and $\kappa_2$}
      \label{fig:vary_kappa} 
\end{figure*}

Based on the above analysis, it is seen that the target distribution has a direct impact on the game decisions, which tends to drive the players to pursue prioritized coverage in tasks with higher estimated worth in general.

\vspace{3pt}
\subsubsection{Player Set Size Parameters ($\kappa_1,\kappa_2$)}\label{sec:neighborhood_size}
Now, we examine the effects of neighborhood sizes $\kappa_1$ and $\kappa_2$ on the coverage  performance for a team of $8$ robots in total. Specifically, we focus on two aspects: (1) using a varying $\kappa_1$ with a fixed $\kappa_2$; and (2) using a varying $\kappa_2$ with a fixed $\kappa_1$. The team-level performance metric $ToTD$ is used for evaluation.

First, we fix $\kappa_2 = 3$ and vary $\kappa_1$ in Scenario $1$. Note that $\kappa_1$ describes the neighborhood size in no-idling games, which is used to define player set $\mathscr{P}$ in Section~\ref{sec:taskreallocation}. However, the players within the $\kappa_1$ neighborhood of the idling robot $v_{id}$, must also satisfy another condition of being close to finish their current tasks at that moment. Thus, the actual number of players (i.e., $|\mathscr{P}|$) could be smaller than $\kappa_1$.

\begin{figure}[t]
  \centering
  \vspace{-6pt}
  \includegraphics[width=.95\columnwidth]{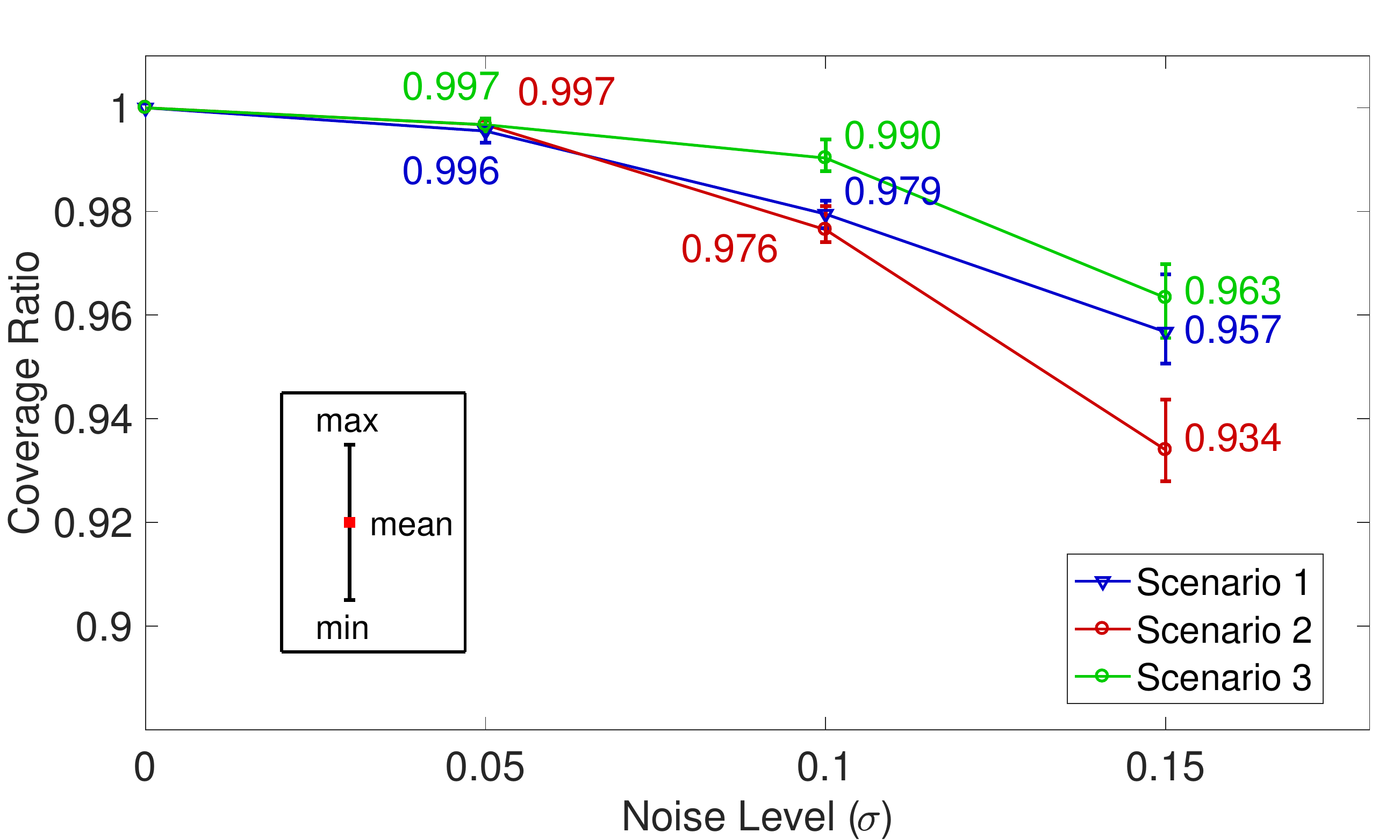} \vspace{-3pt}
  \caption{Coverage ratios at various noise levels over 5 runs/scenario}\label{fig:noiseratio}\vspace{0pt}
\end{figure}

Fig.~\ref{fig:vary_kappa_1} shows the $ToTD$ when $\kappa_1$ gradually increases from $2$ to $7$ in Scenario $1$. As described in Section~\ref{sec:scenario1}, in this scenario, the no-idling games involved $v_2$, $v_5$, $v_6$ and $v_{10}$ that finished earlier than the rest; hence, multiple no-idling games were initiated containing different subset of these players depending on the size of $\kappa_1$. Clearly, it shows that with a larger $\kappa_1$, $ToTD$ is reduced at different target discovery percentages. Moreover, when $\kappa_1 = 3, 4$, and when $\kappa_1 = 6, 7$, $ToTD$ are overlapping. This is due to the same task reallocation decisions in the corresponding no-idling games.

Next, we fix $\kappa_1 = 6$, as was used in the previous scenarios, and measure $ToTD$ at different target discovery percentages when $\kappa_2$ varies from $2$ to $7$. Fig.~\ref{fig:vary_kappa_2} shows the results in Scenario $3$, where $v_4$ failed during exploration. 

As defined in Section~\ref{sec:taskreallocation}, the neighborhood size $\kappa_2$ equals the number of players for resilience games, i.e., $|\mathscr{P}| = \kappa_2$. Thus, as $v_4$ failed, a larger $\kappa_2$ could benefit the team via involving more players in the resilience game for optimization. Note that for $\kappa_2 = 7$, all live robots in the team participated in the resilience game. It is seen in Fig.~\ref{fig:vary_kappa_2} that $ToTD$ is reduced when $\kappa_2$ increases. Moreover, when $\kappa_2 = 2,3,4$ and when $\kappa_2 = 5,6$, the $ToTD$ are almost the same. This is because the same task reallocation decisions were made in the resilience games.

\vspace{-12pt}
\subsection{\textbf{Performance in the Presence of Uncertainties}}
In practice, uncertainties in the robot sensing systems could affect the coverage performance. Thus, for uncertainty quantification, noise was injected into the measurements of laser, compass and localization system for each robot. Typically, the uncertainty in laser measurement is $1\%$ of its sensing range, while a modestly priced compass can be as accurate as $1^o$~\cite{PSSL14}. These errors were simulated with Additive White Gaussian Noise (AWGN) with standard deviations of $\sigma_{laser} = 1.6cm$ and $\sigma_{compass} = 0.5^o$, respectively. On the other hand, indoor localization systems~\cite{FWLM13} can achieve an accuracy of $0.02m$, while Real-Time Kinematic (RTK) based GPS system can be as precise as $0.05m$~\cite{PSSL14}. Therefore, the uncertainty due to localization system is investigated using AWGN at various levels of $\sigma = 0.05m$, $0.10m$ and $0.15m$. Fig.~\ref{fig:noiseratio} shows the minimum, mean and maximum coverage ratios over five runs under different $\sigma$ for the three scenarios using $10$ robots.

\begin{figure}[t]
  \centering
  \vspace{-6pt}
  \includegraphics[width=.88\columnwidth]{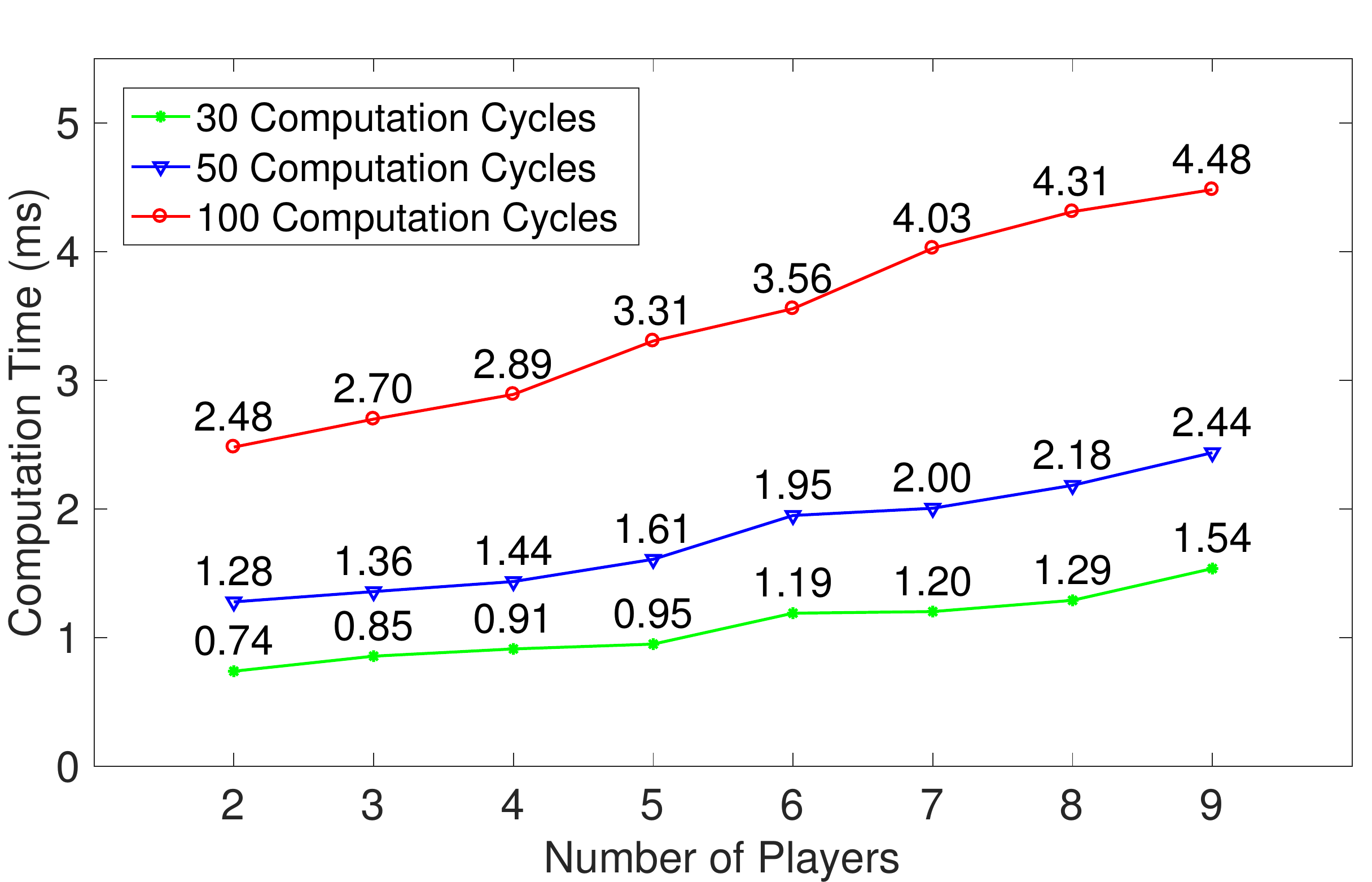} \vspace{-3pt}
  \caption{The computation time for task reallocation}\label{fig:computation_time}\vspace{0pt}
\end{figure}

\vspace{-3pt}
\subsection{\textbf{Computation Time for Task Reallocation}}
As explained earlier, once a resilience game or no-idling game is triggered, the Max-Logit algorithm was used to rapidly converge to the optimal equilibrium. This section evaluates the computation time using Algorithm~\ref{alg:optimizer} for different numbers of players (i.e., $|\mathscr{P}|$) and computation cycles (i.e., $L$).

As an example, Fig.~\ref{fig:computation_time} shows the average computation time of game $G_1$ in Scenario 2 under five runs. It is observed that the computation time monotonically increases as more players are involved; however, due to a distributed computation framework, the slope of growth is gentle. On the other hand, for a fixed number of players, the computation time is proportional to the number of computation cycles.

Note that if less players are involved in a resilience game, a larger number of non-player robots will be able to continue exploration during the task reallocation computation, which facilitates a smooth operation under failures; however, a game with less number of players and tasks may result in a sub-optimal reallocation decision for the whole team. Therefore, the selection of game size must consider these tradeoffs.

\subsection{\textbf{Practical Applications of CARE}}
Some practical applications of CARE are listed below.

\begin{itemize}
\item Cleaning tasks: The floor cleaning task~\cite{PSVC04} in a manufacturing factory environment is one example where a team of robots could be assigned to clean up a large factory floor containing unmapped obstacles. The spread of dirt on the floor can be treated as target distribution with appropriate modifications in the formulation, and based on the day to day experience, the planner could get a good estimate of the heavy or light dirty regions daily. Under these conditions CARE can be implemented to a team of cleaning robots. Also, it is very much possible that some robot fails, thus the nearby robots can be reallocated to help it immediately if needed based on the task priorities. Other similar cleaning application examples using multiple robots include shopping malls, train stations, airports, and commercial buildings. 

\item De-mining in a hazardous environment (e.g., underwater mine countermeasures using UUVs~\cite{MGRP11}): This is an example of a non-cleaning time-critical application, where a robot team is expected to efficiently find all hidden mines even under possible failures of a few robots. In this case, the mines are the targets, and the environment is usually unknown and dangerous, thus CARE can be practically useful for efficient and resilient operation.

\item Agriculture: There could be coverage applications in agriculture~\cite{H14} for seeding and crop-cutting tasks.

\end{itemize}

\section{Conclusions and Future Work}\label{sec:conclusion}
This paper presents a multi-robot coverage algorithm for resilient and efficient coverage of \textit{a priori} unknown environments. The resilience and efficiency of the system are addressed via event-driven task reallocations, using game theoretic solutions. The reallocation decisions are determined by the optimal equilibrium, which is analytically shown to increase the team potential gain. Further, the efficacy of this algorithm has been validated in complex obstacle-rich scenarios on a high-fidelity robotic simulator. The results show that CARE guarantees complete coverage even in presence of failures of some robots. Also, it shows superior coverage performances as compared to three alternative methods in terms of less coverage time and faster target discovery progress.

Future research areas include: i) opportunistic scheduling~\cite{HGW17} to enhance the speed of target discovery, ii) extension of the CARE algorithm to account for restricted communication, iii) integration of SLAM~\cite{SG15} with multi-robot control in the absence of localization devices, iv) consideration of threat levels in different tasks to compute the probability of success, and v) consideration of motion constraints~\cite{SGW19}\cite{SGW17} for the mobile robots.

\vspace{-3pt}
\bibliographystyle{IEEEtran}
\bibliography{./CARE_Song_Gupta}

\end{document}